\documentclass[11pt]{article}

\usepackage{fullpage}
\usepackage{url}
\usepackage{smile}
\usepackage{pdfpages}
\usepackage{kpfonts}
\usepackage{dsfont}
\usepackage{tgpagella}
\usepackage{natbib}

\usepackage[left]{lineno}
\usepackage{blindtext}
\usepackage[ruled, vlined, algo2e]{algorithm2e}
\usepackage{physics}

\newcommand\figcaption{\def\@captype{figure}\caption}
\newcommand\tabcaption{\def\@captype{table}\caption}

\DeclareMathAlphabet{\mathsf}{OT1}{cmss}{m}{n}
\SetMathAlphabet{\mathsf}{bold}{OT1}{cmss}{bx}{n}

\makeatletter

\newcommand{\Rmnum}[1]{\expandafter\@slowromancap\romannumeral #1@}
\makeatother

\begin{document}

%\linenumbers

\title{\Huge Towards Understanding the Importance of Noise in Training Neural Networks\footnote{Working in Progress.}}

\date{}

\author{Mo Zhou, Tianyi Liu, Yan Li, Dachao Lin, Enlu Zhou and Tuo Zhao\thanks{M.Zhou, D.Lin are affiliated with Peking University; T. Liu, Y.Li, E. Zhou, and T. Zhao are affiliated with School of Industrial and Systems Engineering at Georgia Tech; M. Zhou and T. Liu contribute equally; Tianyi Liu and Tuo Zhao is the corresponding author; Email: \{tianyiliu,tourzhao\}@gatech.edu.}}
\maketitle
\begin{abstract}
Numerous empirical evidence has corroborated that the noise plays a crucial rule in effective and efficient training of neural networks. The theory behind, however, is still largely unknown. This paper studies this fundamental problem through training a simple two-layer convolutional neural network model. Although training such a network requires solving a non-convex optimization problem with a spurious local optimum and a global optimum, we prove that perturbed gradient descent and perturbed mini-batch stochastic gradient algorithms in conjunction with noise annealing is guaranteed to converge to a global optimum in polynomial time with arbitrary initialization. This implies that the noise enables the algorithm to efficiently escape from the spurious local optimum. Numerical experiments are provided to support our theory.
\end{abstract}
%\vspace{-0.15in}
%\vskip -0.3in

\section{Introduction}
% DNN important
Deep neural networks (DNNs) have achieved great successes in a wide variety of domains such as speech and image recognition  \citep{hinton2012deep, krizhevsky2012imagenet}, nature language processing \citep{rumelhart1986learning} and recommendation systems \citep{salakhutdinov2007restricted}. Training DNNs requires solving  non-convex optimization problems. Specifically, given $n$ samples denoted by $\{(x_i, y_i)\}_{i=1}^{n}$, where $x_i$ is the $i$-th input feature and $y_i$ is the response, we solve the following optimization problem,
\begin{align*}\label{obj-general}
	\min_{\theta} \cF(\theta)=\frac{1}{n}\sum_{i=1}^n\ell(y_i,f(x_i,\theta)),
\end{align*}
where $\ell$ is a loss function, $f$ denotes the decision function based on the neural network, and $\theta$ denotes the  parameter associated with $f$.

%DNN optima: sharp and flat
Simple first order algorithms such as Stochastic Gradient Descent (SGD) and its variants have been very successful for training deep neural networks, despite the highly complex non-convex landscape. For instance, recent results show that there are a combinatorially large number of saddle points and local optima in training DNNs \citep{choromanska2015loss}. Though it has been shown that SGD and its variants can escape saddle points efficiently and converge to local optima \citep{dauphin2014identifying,kawaguchi2016deep,hardt2016identity, jin2017escape},  the reason why the neural network learnt by SGD generalizes well  cannot yet be well explained, since local optima do not necessarily guarantee generalization. For example, \citet{zhang2016understanding} empirically show the proliferation of global optima (when minimizing the empirical risk), most of which cannot generalize; \citet{keskar2016large} also provide empirical evidence of the existence of  sharp local optima, which do not generalize. They further observe that gradient descent (GD) can often converge to the sharp optima, while SGD tends to converge to the flat ones. This phenomenon implies that the noise in SGD is very crucial and enables SGD to select good optima. Besides, \citet{bottou1991stochastic, neelakantan2015adding} also show that adding noise to gradient can potentially improve training of deep neural networks. These empirical observations motivate us to theoretically investigate the role of the noise in training DNNs.
% Some other works \citep{jin2018minimizing,zhang2017hitting} focus on the ``shallow" local minima  whose basins of attraction have relatively low barriers. These ``shallow" local minima also yield bad performance.  \\
%% shallow or deep
%\begin{figure}
%\center
%	\includegraphics[scale=0.5]{flat_sharp}
%	\caption{Flat Optima v.s. Sharp Optima \citep{keskar2016large}}\label{flat_sharp}
%\end{figure}
% Though the bad local optima exist, some works (e.g.  \cite{choromanska2015loss}) argue that the major challenge of training DNNs is escaping saddle points since the number of bad local optima has been proved to diminish exponentially fast with the size of the network. Based on this, \cite{ge2015escaping, jin2017escape, liu2018toward, liu2018towards} explain the good convergence property by studying how SGD and its variants escape saddle points successfully.
% This “No spurious/bad local optima" argument, however, is often considered to be overoptimistic. People has  observed and theoretically proved that  the spurious/bad local optima are not completely negligible \citep{hochreiter1997flat,keskar2016large,zhang2017theory,neyshabur2017exploring,safran2017spurious}.   Thus, the theoretical justification on the global convergence to a good optima is still required.
% People observe noise escape
%\cite{zhang2017hitting} prove the polynomial hitting time of a local optima of SGD with extra noise under certain assumptions.   These results, though insightful, do not directly explain why noise helps escape bad optima. The theoretical justification is still in need.

%Our obervation

This paper aims to provide more theoretical insights on the following fundamental question:
\begin{center}
	
	\textbf{\emph{How does noise help train neural networks in the presence of bad  local optima?}}
	
\end{center}
Specifically, we  study a two-layer non-overlapping convolutional neural network (CNN) with one hidden layer, which takes the following form:
\begin{equation*}
	f(Z,w,a)=a^T\sigma(Z^Tw),
\end{equation*}
where $w\in \mathbb{R}^p$, $a\in \mathbb{R}^k$ and $Z\in \mathbb{R}^{p\times k}$ are the convolutional weight, the output weight and the input, respectively, and $\sigma(\cdot)$ is the element-wise ReLU activation operator. Since the ReLU activation is positive homogeneous, the weights $a$ and $w$ can arbitrarily scale with each other. Thus, we impose an additional constraint $\norm{w}_2=1$ to make the neural network identifiable. We consider the realizable case, where the training data is generated from a {\bf teacher} network with true parameters  $w^*$, $a^*$ and $\norm{w^*}_2=1$. Then we aim to recover the {\bf teacher} neural network by solving the following optimization problem,
\begin{align}\label{optimization}
	\min_{w,~a} &~~\frac{1}{2}\mathbb{E}_Z(f(Z,w,a) - f(Z,w^*,a^*))^2\quad \mathrm{subject to}\quad w^\top w=1,
\end{align}
where $Z$ is independent  Gaussian input\footnote{This is a common assumption in previous works \citep{tian2017analytical, brutzkus2017globally, zhong2017recovery}}. One can verify that $({w^*},a^*)$ is a global optimum of \eqref{optimization}.

Though over-simplified compared with complex deep neural networks in practice, the above model turns out to have some intriguing property, which helps us get insight into understanding the optimization landscape of training neural  networks. Specifically, \citet{du2017gradient} show that the optimization problem  \eqref{optimization} has a non-trivial spurious  optimum, which does not generalize well. They further prove that with random initialization, Gradient Descent (GD) can be trapped in this spurious optimum with constant probability\footnote{\citet{du2017gradient} prove that this probability is bounded between $1/4$ and $3/4$. Their numerical experiments show that this probability can be as worse as $1/2$.}.

Inspired by \citet{du2017gradient}, we propose to investigate whether adding noise to gradient descent helps avoid the spurious optimum using the same model. Specifically, we consider a perturbed GD algorithm\footnote{Our algorithm actually updates $w$ using the manifold gradient over the sphere. See more details in Section 2} in conjunction with noise annealing to solve the optimization problem \eqref{optimization}. To be more concrete, we run the algorithm with multiple epochs and decrease the magnitude of the noise  as the number of epochs increases. Note that our algorithm is different from SGD in terms  of the noise.  In our algorithm, we inject independent noise to the gradient update at every iteration, while the noise of SGD comes from the training sample. As a consequence, the noise of SGD has very complex dependence on the iterate, which is very difficult to analyze. See more detailed discussions in Sections \ref{sec_model} and \ref{sec_discuss}.

We further analyze the convergence properties of our perturbed GD algorithm: At early stages, large noise essentially convolutes with the loss surface and makes the optimization landscape smoother, which tames non-convexity and rules out the spurious local optimum. Hence, perturbed GD is capable of escaping from the spurious local optimum.  Though large noise leads to large optimization errors, this can be further compensated by noise annealing.  In another word, the injected noise with decreasing magnitude essentially guides GD to gradually approach and eventually fall in the basin of attraction of the global optimum. Given that the noise has been annealed to a sufficiently small level at later stages, the algorithm finally converges to the global optimum and stays in its neighborhood. Overall, we prove that with random initialization and noise  annealing, perturbed GD is guaranteed to converge to the global optimum with high probability in polynomial time. Moreover, we further extend our proposed theory to the perturbed mini-batch stochastic gradient algorithm, and establish similar theoretical guarantees. To the best of our knowledge, this is the first theoretical result towards justifying the effect of noise in training NNs by first order algorithms in the presence of the spurious local optima.

Our work is related to \citet{zhou2017stochastic,li2017convergence, kleinberg2018alternative,jin2018local}, which also study the effect of noise in non-convex optimization. We give detailed discussions in Section \ref{sec_discuss}.

The rest of the paper is organized as follows: Section \ref{sec_model} describes the two-layer non-overlapping convolutional network and introduces our perturbed GD algorithm; Section \ref{sec_theory} and \ref{extension} present the convergence analysis; Section \ref{sec_num} provides the numerical experiments; Section \ref{sec_discuss} discusses related works.

\noindent{\textbf{Notations}}: Given a vector $v = (v_1, \ldots, v_d)^\top\in\RR^{d}$, we define $\norm{v}_2^2 = \sum_jv_j^2$, $\norm{v}_1 = \sum_j|v_j|$. For vectors $v,u\in\RR^d$, we define $\langle u,v\rangle=\sum_{j=1}^du_jv_j$. $\mathbb{B}_0(r)$ denotes a ball with radius $r$ centered at zero in  $\RR^{d}$, i.e., $\mathbb{B}_0(r)=\{v\in\RR^{d}\,|\,\|v\|_2\leq r\}$ and $\mathbb{S}_0(r)$ denotes the boundary of  $\mathbb{B}_0(r)$. For two vectors $v, w$,  $\angle(v,w)$ represents the angle between them, i.e., $\angle(v,w)=\arccos{\frac{v^\top w}{\norm{v}_2\norm{w}_2}}.$ We denote the uniform distribution on $\cM\subseteq\RR^{d}$ by $\mathrm{unif}(\cM)$ and the projection of vector $v$ on set $\cM$ by $\mathrm{Proj}_{\cM}(v)$. For two sets $A$ and $B\in\RR^{d},$ $A/B=\{x\in\RR^d\big|x\in A, ~x\notin B\}$.
%Denote $\mathcal{F}_t=\sigma\{\epsilon_0,\xi_0,...,\epsilon_t,\xi_t\}$ as the sigma algebra generated by first t iterations. Denote $\tilde{O}(\cdot)$ hides $poly(d,p,\norm{\tilde{a}^*}_2,|\mathds{1}^\top\tilde{a}^*|, \rho_a, \rho_w)$ and $polylog(d,p,\norm{\tilde{a}^*}_2,|\mathds{1}^\top\tilde{a}^*|, \rho_a, \rho_w)$, and only focus on $\eta$, $\delta$, $\gamma$....
\section{Model and Algorithm}\label{sec_model}

We first introduce the neural network models of our interests, and then present the nonconvex optimization algorithm.

\subsection{Neural Network Models}

Recall that we study a two-layer non-overlapping convolutional neural network (CNN) given by:
\begin{align}\label{NN-model}
	f(Z,w,a)=a^\top\sigma(Z^\top w),
\end{align}
where $a\in \mathbb{R}^k$,  $w\in \mathbb{R}^p$ and $Z\in \mathbb{R}^{p\times k}$ are the output weight, the convolutional weight and input, respectively. $\sigma(\cdot)$ denotes the element-wise ReLU activation operator. Since the ReLU activation is homogeneous, $w$ and $a$ can arbitrarily scale with each other without changing the output of the network, i.e., $ f(Z, w, a) = f(Z, cw, \frac{a}{c})$ for any $c>0$.
Thus, we impose an additional constraint $\norm{w}_2=1$ to make the model identifiable.
We assume independent Gaussian input $Z=[Z_1,...,Z_p]$, where $Z_i$'s are independently sampled from $N(0,I)$, and focus on the noiseless realizable setting -- i.e., the response is generated by a noiseless {\bf teacher} network
\begin{align*}
	y=f(Z, w^*,a^*)=(a^*)^\top\sigma(Z^\top w^*)
\end{align*}
with some true parameters $\norm{w^*}_2 = 1$ and $a^*$. We aim to learn a {\bf student} network, i.e., recover the true parameter $(w^*, a^*)$ by solving the following regression problem using mean square loss:
\begin{align}\label{opt-recall}
	\min_{w,a} \cL(w,a)\quad\text{subject to}\quad w^\top w=1,
\end{align}
where $\cL(w,a)=\frac{1}{2}\mathbb{E}_Z(f(Z,w,a) - f(Z,w^*,a^*))^2.$ The optimization landscape has been partially studied by \citet{du2017gradient}. Specifically, one can easily verify that  $(w^*, a^*)$ is a global optimum of \eqref{opt-recall}. Moreover, they prove that there exists a spurious local optimum, and gradient descent with random initialization can be trapped in this spurious optimum with constant probability.  %A formal description is given as follows.
\begin{proposition}[Informal, \citet{du2017gradient}]\label{bad_optima}
	Given $$w_0 \sim \mathrm{unif}\big(\mathbb{S}_0\left(1 \right)\big)\quad\textrm{and}\quad a_0 \sim
	\mathrm{unif}\big(\mathbb{B}_0\big(\frac{|\one^\top a^*|}{\sqrt{k}}\big)\big)$$ as the initialization and the learning rate is sufficiently small, then with at least probability  $1/4,$ GD converges to the spurious local minimum $(v^*,\tilde{a})$ satisfying $$\angle(v^*, w^*)=\pi,~\tilde{a}=(\one\one^\top+(\pi-1)I)^{-1}(\one\one^\top-I)a^*.$$
\end{proposition}
Please refer to \citet{du2017gradient} for more details.

\subsection{Optimization Algorithm}

We then present the perturbed gradient descent algorithm for solving \eqref{opt-recall}.
%\cite{du2017gradient} first show that we can calculate the closed-form expression of the gradient of $w$ and $a$.
% \footnote{In fact, $g_w(w,a)$ differs from $\nabla_w L(w,a) $. However, simple calculation shows $(I-w_tw_t^\top)g_w(w_t,a_t)=(I-w_tw_t^\top)\nabla_w L(w,a)$ -- i.e., both yield the same manifold gradient. }
Specifically, at the $t-$th iteration, we
% samples independent uniformly distributed noise $\xi_t \sim  \text{unif} \left(\mathbb{B}_0 \left(\rho_w \right) \right), \epsilon_t \sim \text{unif} (\mathbb{B}_0(\rho_a))$, where the radius $\rho_w$, $\rho_a$ controls the noise level.
perturb the iterate $(w_t, a_t)$ with independent noise $\xi_t \sim \mathrm{unif}\big(\mathbb{B}_0(\rho_w)\big)$ and $\epsilon_t \sim \mathrm{unif}\big(\mathbb{B}_0(\rho_a)\big)$
%  and use the manifold gradient at perturbed iterate $(I-w_tw_t^\top)g_w(w_t + \xi_t ,a_t + \epsilon_t)$ to
%  perform the update:
and take:
\begin{align*}
	a_{t+1}& = a_{t} - \eta \nabla_a \cL(w_t + \xi_t,a_t + \epsilon_t),\\
	w_{t+1} &= \mathrm{Proj}_{\mathbb{S}_0(1)} \big(w_t-\eta (I-w_tw_t^\top)\nabla_w \cL(w_t + \xi_t,a_t + \epsilon_t)\big),
\end{align*}
where $\eta$ is the learning rate. We remark that the update for $w$ in our algorithm is essentially based on the manifold gradient, where $(I-w_tw_t^\top)$ is the projection operator to the tangent space of the unit sphere at $w_t$. For simplicity, we still refer to our algorithms as Perturbed Gradient Descent.
% where
% \begin{equation*}
%   g_a(w,a)\triangleq \frac{1}{2\pi}(\one\one^\top+(\pi-1)I)a
% -\frac{1}{2\pi}(\one\one^\top+(g(\phi)-1)I)a^*\norm{w^*}_2,
% \end{equation*}
% \begin{equation*}
% g_w(w,a)\triangleq -\frac{a^\top a^* (\pi -\phi)}{2\pi}w^*.
% \end{equation*}
% \begin{equation*}
% g(\phi)=(\pi-\phi)\cos\phi+\sin\phi.
% \end{equation*}
% We have $g_a(w,a) = \nabla_a L(w,a) $ but $g_w(w,a)\neq \nabla_wl(w,a)$.
% Simple calculation shows  $(I-w_tw_t^\top)g_w(w_t,a_t)=(I-w_tw_t^\top)\nabla_wl(w,a),$ we will use the former one as it has simpler form.

% Note that  $g_w(w,a)\neq \nabla_wl(w,a)$. Since simple calculation shows  $(I-w_tw_t^\top)g_w(w_t,a_t)=(I-w_tw_t^\top)\nabla_wl(w,a),$ we still use this simplified version.
% It has been shown in \cite{du2017gradient} that when the initialization point $(w_0, a_0)$ satisfies $w_0^\top w^*\geq 0,$ $a_0^\top a^*\geq p, $ $|I^\top a_0|\leq|I^\top a^*|$, gradient descent will find the true parameter $(w^*, a^*)$. But if $(w^*, a^*)$ satisfy conditions in Proposition \ref{bad_optima}, randomly initialized gradient descent can converge to bad local optima with probability at least $\frac{1}{4}$.
% As we mentioned in Proposition \ref{bad_optima}, randomly initialized gradient descent can converge to bad local optima with probability at least $\frac{1}{4}$.

% Besides random initialization considered in  \cite{du2017gradient}, another common technique to avoid converging to local minima is to add noise to the algorithm.
As can be seen, for $\xi_t=0$ and $\epsilon_t=0$, our algorithm is reduced to the (noiseless) gradient descent. Different from SGD, the noise of which is usually from randomly sampling the data, we inject the noise directly to the iterate used for computing gradient. Moreover, stochastic gradient is usually an unbiased estimate of gradient, while our perturbed gradient $\nabla_aL(w_t + \xi_t,a_t + \epsilon_t)$ and $\nabla_wL(w_t + \xi_t ,a_t + \epsilon_t)$ yield biased estimates, i.e.,
\begin{align*}
	\EE_{\xi_t, \epsilon_t}{\nabla_a \cL(w_t + \xi_t,a_t + \epsilon_t)}\neq \nabla_a \cL(w_t ,a_t)\quad\textrm{and}\quad\EE_{\xi_t, \epsilon_t}{\nabla_w \cL(w_t + \xi_t ,a_t + \epsilon_t)}\neq \nabla_w \cL(w_t  ,a_t).
\end{align*}
See detailed discussions in Section \ref{sec_discuss} and Appendix \ref{app_pre}.

Our algorithm also incorporates the noise annealing approach. Specifically, the noise annealing consists of multiple epochs with varying noise levels. Specifically, we use large noise in early epochs and gradually decrease the noise level, as the number of epoch increases. Since we sample the noise $\xi_t$ and $\epsilon_t$ uniformly from $\mathbb{B}_0(\rho_w)$ and $\mathbb{B}_0(\rho_a)$, respectively, we can directly control the noise level by controlling the radius of the ball, i.e., $\rho_w$ and $\rho_a$. One can easily verify
%Hence, when we inject large noise, the noisy gradient maybe too biased to guarantee convergence, and thus a noise annealing scheme need to be incorporated.
%Simple calculation shows that both the magnitude and the variance of the noise are controlled by
%the radius, i.e.,
\begin{align*}
	& \norm{\xi_t}_2\leq\rho_w,~\EE{ \xi_t}=0,~\Cov{ \xi_t}=\frac{\rho_w^2}{p+2}I,~\norm{\epsilon_t}_2\leq\rho_a,~\EE{ \epsilon_t}=0~\text{and}~ \Cov{ \epsilon_t}=\frac{\rho_a^2}{k+2}I.
\end{align*}
We summarize the algorithm in Algorithm \ref{NMGD}. %The intuition of the noise annealing is illustrated in Figure \ref{BOA}.
\begin{remark}
	Note that our arbitrary initialization is different from the random initialization in \citet{du2017gradient}, which requires $w_0 \sim \mathrm{unif}\left(\mathbb{S}_0 (1)\right)$ and $a_0 \sim\mathrm{unif}\left( \mathbb{B}_0 \left(\frac{|\one^\top a^*|}{\sqrt{k}}\right)\right).$ They need the randomness to avoid falling into the basin of attraction of the spurious local optimum. Our perturbed GD, however, can be guaranteed to escape the spurious local optimum. Thus, we initialize the algorithm arbitrarily.
\end{remark}
%\vspace{-0.1in}
\begin{remark}[Convolutional Effects] We remark that the $s$-epoch of the perturbed GD can also be viewed as solving
	\begin{align}\label{smoothed_opt}
		\min_{\norm{w}_2=1,a}\EE_{\xi_s,\epsilon_s} \cL(w+\xi_s,a+\epsilon_s),
	\end{align}
	where $\xi_{s} \sim \mathrm{unif}\big(\mathbb{B}_0(\rho_w^s)\big)$ and $\epsilon_{s} \sim \mathrm{unif}\big(\mathbb{B}_0(\rho_a^s)\big)$. Therefore, the noise injection can be interpreted as convoluting the objective function with uniform kernels. Such a convolution makes the objective much smoother, and leads to a benign optimization landscape with respect to the global optimum of the original problem, as illustrated in Figure \ref{smoothing} (See more details in the next section).
\end{remark}

\begin{figure}[htb!]
	\centering
	\includegraphics[width=0.8\linewidth]{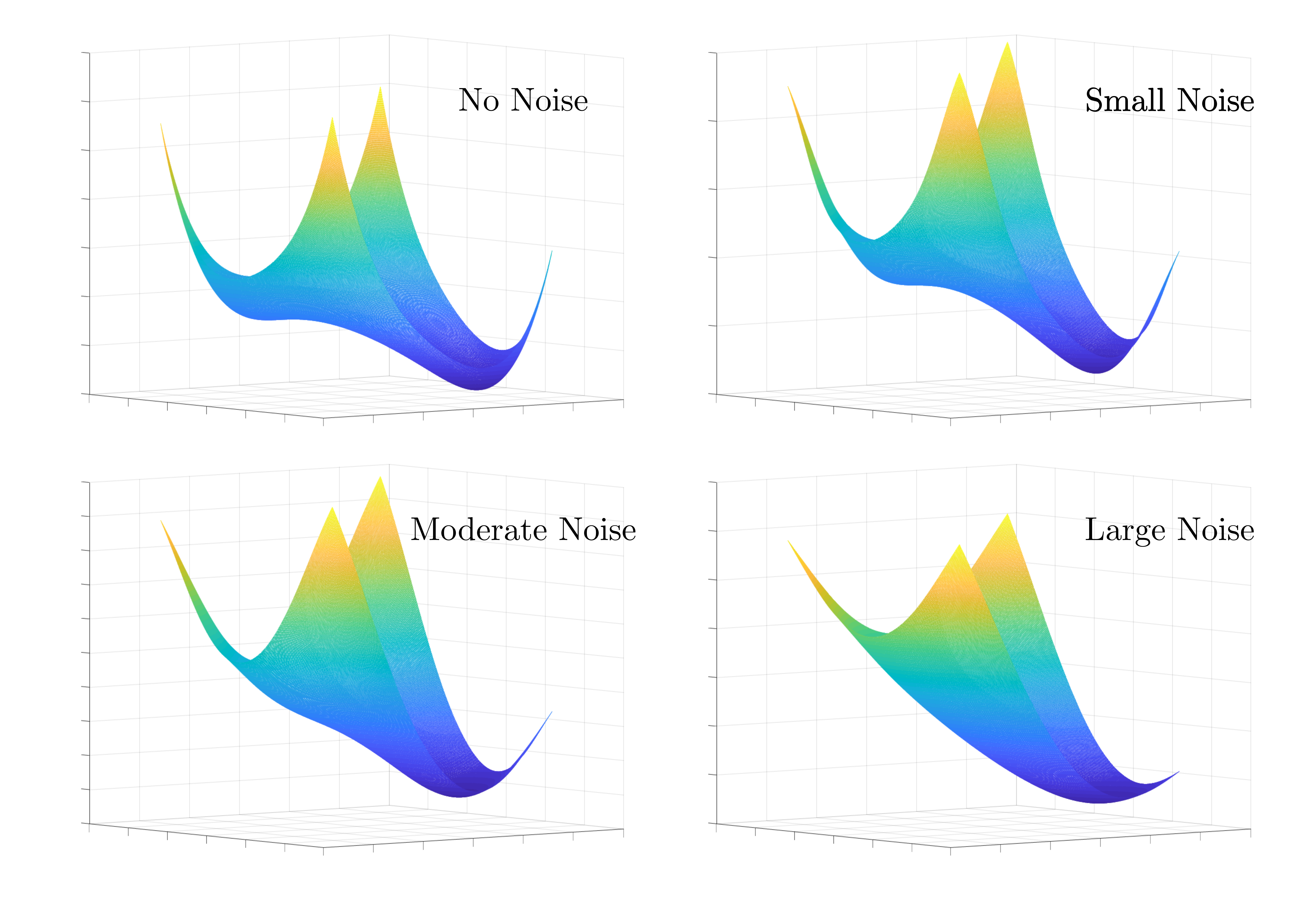}\label{smoothing}
	\caption{\it An illustration of the convolutional effects of the injected noise. Larger noise leads to smoother optimization landscapes, but also yields larger approximation errors to the original problem.}
\end{figure}

Note that the above convolution effect also introduces additional ``bias'' and ``variance'': (I) The global optimum of the smooth approximation \eqref{smoothed_opt} is different from the original problem; (II) The injected noise prevents the algorithm from converging. This is why we need to gradually decreasing the magnitude of the noise, which essentially guides the perturbed GD to gradually approach and eventually fall in the basin of attraction of the global optimum of the original problem (as illustrated in Figure \ref{BOA}).

\begin{algorithm}\label{alg:pgd}
	\setlength{\textfloatsep}{0pt}
	\DontPrintSemicolon
	\caption{\it Perturbed Gradient Descent Algorithm with Noise Annealing}
	\label{NMGD}
	\textbf{input:} number of epochs $S$, length of epochs $\{T_s\}_{s=1}^S$, learning rate schedule $\{\eta_s\}_{s=1}^S$ and noise level schedule $\{\rho_w^s\}_{s=1}^S$, $\{\rho_a^s\}_{s=1}^S$ \\
	\textbf{initialize}: choose any $w_0 \in \mathbb{S}_0 (1)$ and $a_0 \in \mathbb{B}_0 \big(\frac{|\one^\top a^*|}{\sqrt{k}}\big)$\\
	\For{$s = 1, \ldots, S$}{
		$w_{s,1} \leftarrow w_{0}$, $a_{s,1} \leftarrow a_{0}$ \\
		\For{$t = 1 \ldots T_s-1$}{
			$\xi_{s,t} \sim \mathrm{unif}\big(\mathbb{B}_0(\rho_w^s)\big)$ and $\epsilon_{s,t} \sim \mathrm{unif}\big(\mathbb{B}_0(\rho_a^s)\big)$\\
			$a_{s,t+1}   \leftarrow a_{s,t} - \eta_s \nabla_a \cL(w_{s,t}+\xi_{s,t},a_{s,t}+\epsilon_{s,t})$\\
			${w}_{s,t+1} \leftarrow \mathrm{Proj}_{\mathbb{S}_0(1)} \big(w_{s,t}-\eta_s (I-w_{s,t} w_{s,t}^\top)\cdot\nabla_w \cL(w_{s,t}+\xi_{s,t},a_{s,t}+\epsilon_{s,t})\big)$
		}
		$w_{0} \leftarrow w_{s,T_s}$, $a_{0} \leftarrow a_{s,T_s}$
	}
	\textbf{output:} $(w_{s,T_s}, \,\,  a_{s,T_s})$
\end{algorithm}

\begin{figure}[htb!]
	\centering
	\includegraphics[width=0.9\linewidth]{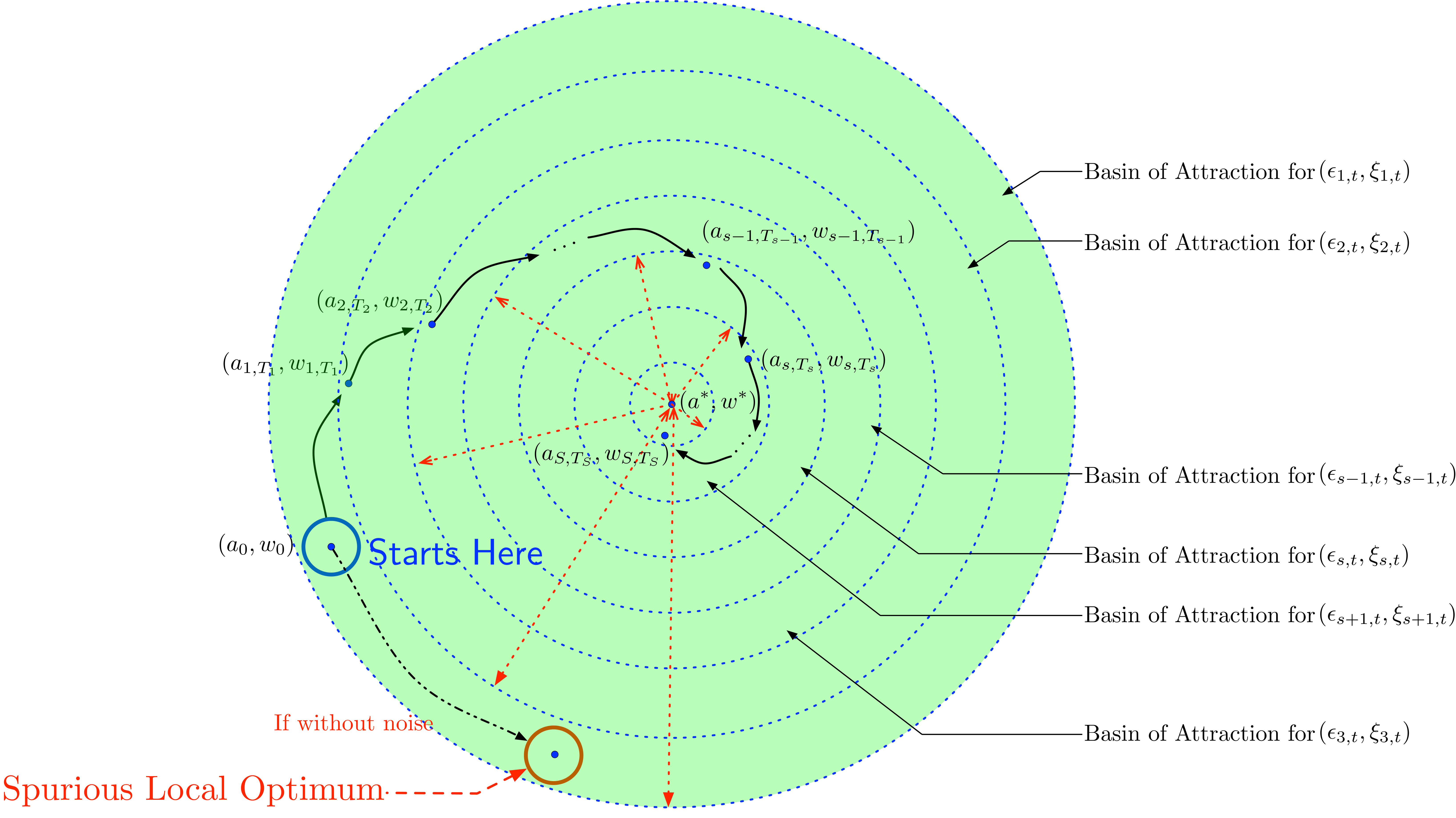}\label{BOA}
	\caption{\it An illustration of the noise injection in the perturbed GD algorithm. The injected noise with decreasing magnitude essentially guides the perturbed GD to gradually approach and eventually fall in the basin of attraction of the global optimum.}
\end{figure}

\section{Convergence Analysis}\label{sec_theory}
%%%convolution effect $$smoother$$$$ one point convexity$$$ not gloabally  holds smoother
We investigate the algorithmic behavior of the proposed perturbed GD algorithm. Our analysis shows that the noise injected to the algorithm has a convolutional effect on the loss surface and makes the optimization landscape smoother, which tames non-convexity by avoiding  being trapped at the bad local optimum. Thus, our proposed algorithm can converge to the global one.

Our theory essentially reveals a phase transition as the magnitude of the injected noise decreases. For simplicity, our analysis only considers a two-epoch version of the proposed perturbed GD algorithm, but can be generalized to the multiple-epoch setting (See more detailed discussions in Section \ref{sec_discuss}). Specifically, the first epoch corresponds to {\bf Phase I}, and the proposed algorithm shows an escaping from the spurious local optimum phenomenon, as the injected noise is sufficiently large; The second epoch corresponds to {\bf  Phase II}, and the proposed algorithm demonstrates  convergence to the global optimum, as the injected noise is reduced.

Before we proceed with our main results, we first define the partial dissipative condition for an operator $\cH$ as follows.
\begin{definition}[Partial dissipativity]\label{SC}
	Let $\cM$ be a subset of $\{1,2,...,d\}$ with $|\cM|=m$, and $x_\cM$ be the subvector of $x\in \RR^d$ with all indices in $\cM$. For any operator  $\cH:~ \RR^d\rightarrow \RR^m,$ we say that $\cH$ is $(c_\cM,\gamma_\cM,\cM)$-partial dissipative with respect to (w.r.t.) the subset $\cX^*\subseteq \RR^d$ over the set $\cX\supseteq\cX^*$, if for every $x\in\cX$, there exist an $x^*\in\cX^*$ and two positive universal constants $c_\cM$ and $\gamma_{\cM}$ such
	that
	\begin{align}\label{con_sc}
		\langle-\cH(x),x_S^*-x_S\rangle\geq c_\cM \norm{x_S-x_S^*}_2^2-\gamma_{\cM}.
	\end{align}
	$\cX$ is called the partial dissipative region of the operator $\cH$ w.r.t. $x_{\cM}$.
\end{definition}

The partial dissipativity in definition \ref{SC} is actually a generalization of  the joint dissipativity from existing literature on studying attractors of dynamical systems \citep{barrera2015thermalisation}. To be specific, when $S=\{1,2,...,d\}$, partial dissipativity is reduced to dissipativity.  Here we are using the partial dissipativity, since our proposed algorithm can be viewed as a complicated dynamical system, and the global optimum is the target attractor.

The variational coherence studied in \cite{zhou2017stochastic} and one point convexity studied in \cite{kleinberg2018alternative} can be viewed as the special example of partial dissipativity. Specifically, they consider $\gamma=0$, the operator $\cH$ as the gradient of the objective function $f$ and $\cX^*$ as the set of all minimizers of $f$. More precisely, their conditions require $$\langle-\nabla f(x),x^*-x\rangle>c\norm{x-x^*}_2^2,$$ i.e., the negative gradient of the objective function to have a positive fraction pointing toward $\cX^*$, and therefore the gradient descent algorithm is guaranteed to make progress towards the optimum at every iteration. The variational coherence/one point convexity, though nice and intuitive, is a very strong assumption. For the optimization problem of our interest in \eqref{opt-recall}, such a condition does not hold even within a small neighborhood around the global optimum. Fortunately, we show that the problem enjoys partial dissipativity which is more general and can characterize more complicated structure of the problem.  Please see more discussion in Section \ref{sec_discuss}.

%Through the rest of this paper, if not clearly specified, we assume that $\norm{{a}^*}_2$ is a constant. The notation $\tilde{O}(\cdot)$ hides $\mathrm{poly}(\norm{{a}^*}_2)$ and $\mathrm{polylog}(\norm{{a}^*}_2)$.

\subsection{Phase I: Escaping from the Local Optimum}
We first characterize the algorithmic behavior of our proposed algorithm in Phase I. Note that our proposed perturbed GD algorithm, different from GD, intentionally injects noise at each iteration, and the update is essentially based on the perturbed gradient.  The following theorem
characterizes the partial dissipativity of the perturbed gradient.
\begin{theorem}\label{thm_SC}
	Choose $\rho_w^0=C_w^0 kp^2\geq 1$ and $\rho^0_a=C_a^0$ for large enough constants $C_w$ and $C_a$. Let $\xi\sim\mathrm{unif}(\BB_0(\rho_w^0))$ and $\epsilon\sim\mathrm{unif}(\BB_0(\rho_a^0))$. There exist some constants $C_1$ and $C_2$ such that the perturbed gradient of $L$ w.r.t. $a$ satisfies
	\begin{align*}
		\langle-\mathbb{E}_{\xi,\epsilon}\nabla_a\cL(w+\xi,a+\epsilon),{a}^*-a\rangle
		\geq \frac{C_{1}}{p} \norm{a-{a}^*}_2^2
	\end{align*}
	for any $(w,a)\in\cA_{C_2,C_3}$, where
	\begin{align*}
		\cA_{C_2,C_3}=\big\{(w,a)~\big| a^\top a^*\leq \frac{C_2}{p}\norm{a^*}_2^2 \text{ or}~
		\norm{a-a^*/2}_2^2\geq\norm{a^*}_2^2,~ &w\in\SSS_0(1),\\
		&~-4(\one^\top{a}^*)^2\leq \one^\top {a}^*\one^\top a - (\one^\top {a}^*)^2
		\leq  \frac{C_{3}}{p}\norm{{a}^*}_2^2\big\}.
	\end{align*}
	Moreover, for any $C_4\in(-1,1]$ and $M>m>0$, there exists some constant $C_5$ such that the perturbed manifold gradient of $L$ w.r.t. $w$ satisfies
	\begin{align*}
		\langle-\mathbb{E}_{\xi,\epsilon}{(I-ww^\top)}\nabla_w\cL(w+\xi,a+\epsilon),{w}^*-w\rangle \geq \frac{m(1+C_4)}{16}\norm{w-{w}^*}_2^2-C_5\frac{k}{\rho_w},
	\end{align*}
	for any $(w,a)\in\mathcal{K}_{C_4,m,M}$, where
	\begin{align*}%\label{one_point_w}
		\mathcal{K}_{C_4,m,M}=\left\{(w,a)~\big|~ a^\top {a}^*\in[m,M],~ w^\top w^*\geq C_4,~w\in\SSS_0(1)\right\}.
	\end{align*}
\end{theorem}
The detailed proof of Theorem \ref{thm_SC} is provided in Appendix \ref{pf_SC}. 
Theorem \ref{thm_SC} shows that the partial dissipativity holds for the perturbed gradient of $L$ with respect to $a$ over $\cA_{C_2,C_3},$ and the partial dissipativity holds for the perturbed manifold gradient of $L$ with respect to $w$ over $\mathcal{K}_{C_4,m,M}$, respectively. Note that the joint dissipativity can hold but only over a smaller set $\cA_{C_2,C_3}\cap\cK_{C_4,m,M}$. 
%Therefore, the optimization landscape (for Phase I) described in Theorem \ref{thm_SC} is actually more complicated than that defined in Definition \ref{SC}, which makes our problem very challenging. 
Fortunately, the partial dissipativity is enough to ensure our proposed algorithm to make progress at every iteration, even though the joint dissipativity does not hold.  As a result, our proposed algorithm can avoid being trapped by the spurious local optimum. For simplicity, we denote $\phi_t$ as the angle between $w_t$ and $w^*,$ i.e., $\phi_t=\angle(w_t,w^*).$ The next theorem analyzes the algorithmic behavior of the perturbed GD algorithm in Phase I.
\begin{theorem}\label{stage1_w}
	Suppose $\rho_w^0=C_w^0 kp^2\geq 1, ~\rho_a^0=C_a^0, ~  a_0\in\BB_0\Big(\frac{|\one^\top{a^*}|}{\sqrt{k}}\Big)$ and $w_0\in \SSS_0(1).$
	For any $\delta\in(0,1)$, we choose step size $$\eta=C_6\Big(k^4p^6\cdot\max\Big\{1,p\log\frac{1}{\delta}\Big\}\Big)^{-1}$$ for some constant $C_6$. Then with at least probability $1-\delta$, we have
	\begin{align}\label{stage1_output}
		m_a\leq a_t^\top{a}^* \leq M_a\quad\text{and}\quad\phi_t\leq \frac{5}{12}\pi
	\end{align}
	for all $T_1\leq t\leq \tilde{O}(\eta^{-2})$, where $m_a=C_4\norm{{a}^*}_2^2/p$, $M_a=4(\one^\top{a}^*)^2  + (3+C_7/p))\norm{{a}^*}_2^2$ for some constants $C_4$ and $C_7$, and $$T_1
	=\tilde{O}\Big(\frac{p}{\eta}\log\frac{1}{\eta}\log\frac{1}{\delta}\Big).$$
\end{theorem}
Theorem \ref{stage1_w} shows that Phase I of our perturbed GD algorithm only needs polynomial time to ensure  the output solution to be sufficiently distant from the spurious local optimum with high probability. Due to the large injected noise, Phase I cannot output a very accurate solution. 

The proof of Theorem \ref{stage1_w} is very technical and highly involved. Here, we only provide a proof sketch. See more details and the proof of all technical lemmas in Appendix \ref{pf_phase I}.

\begin{proof}[\textbf{\textit{Proof Sketch}}]
	The next lemma shows that that our initialization  $(w_0,a_0)$ is guaranteed to fall in a superset of $\cA_{C_2,C_3}.$
	\begin{lemma}\label{lem_initial}
		Given $a_0 \in  \mathbb{B}_0\left(\frac{|\one^\top{a^*}|}{\sqrt{k}}\right)$ and $w_0\in\SSS_0(1)$, we have for any constant $C_3>0,$ $(w_0,a_0)\in\cA_{C_3},$ where \begin{align*}
			\cA_{C_3}=\big\{(w,a)~\big|~-4(\one^\top{a}^*)^2\leq &\one^\top {a}^*\one^\top a - (\one^\top {a}^*)^2\leq  \frac{C_{3}}{p}\norm{{a}^*}_2^2,~ w\in\SSS_0(1)\big\}.
	\end{align*}\end{lemma}
	
	Our subsequent analysis considers two cases: Case (1) $(w_0,a_0)\in\cA_{C_2,C_3}$ and Case (2) $(w_0,a_0)\in\cA_{C_3}\backslash\cA_{C_2,C_3}$. Specifically, we first start with Case (1), and then show the algorithm will be able to escape from $\cA_{C_2,C_3}$ in polynomial time and enter $\cA_{C_3}\backslash\cA_{C_2,C_3}$. Then we only need to proceed with Case (2).
	
	Note that for $\cA_{C_2,C_3}$, the dissipativity holds only for the perturbed gradient with respect to $a.$ Though the dissipativity does not necessarily hold for $w$,  we can show that the noise injection procedure guarantees a sufficiently accurate $w$ for making progress in $a$, as shown in the next lemma.
	\begin{lemma}\label{lem_gphi_1}
		Suppose $w,{w}^*\in\mathds{S}_0(1)$ and $\xi\sim\mathrm{unif}(\mathbb{B}_0(\rho_w))\in\mathbb{R}^p$. Define $\phi_\xi\triangleq\angle(w+\xi,{w}^*)\in [0,\pi]$,  $\phi\triangleq\angle(w,{w}^*)$ and $g(\phi)=(\pi-\phi)cos(\phi)+sin(\phi).$ When $\rho_w\geq C_6p^2$ for some constant $C_6$,  there exists some constant  $C_8$ such that 
		\begin{align*}
			1+\frac{C_8}{p}\leq \mathbb{E}_\xi g(\phi_\xi) \leq \pi\quad\mathrm{and}\quad\mathbb{E}_\xi\phi_\xi\leq \frac{3\pi}{4}	
		\end{align*}
		for all $\phi\in[0,\pi].$
	\end{lemma}
	We remark that Lemma \ref{lem_gphi_1} is actually the key to the convergence analysis for Phase I. It helps prove both Theorems \ref{thm_SC} and \ref{stage1_w}. The proof is highly non-trivial and very involved. See more details in Appendix \ref{pf_gphi}.  Lemma \ref{lem_gphi_1} essentially shows that the noise injection prevents $w$ from being attracted to $v^*$, and further prevents $(w,a)$ from being attracted to the spurious local optimum.
	
	We then analyze Case (1), where $(w_0,a_0)\in\cA_{C_2,C_3}$. 
	\begin{lemma}\label{stage1_a}
		Suppose $\rho_w^0=C_w kp^2\geq1$, $\rho_a^0=C_a$ and $(w_0,a_0)\in\cA_{C_2,C_3}$. For any $\delta\in(0,1),$  we choose step size $$\eta=C_6\Big(k^4p^6\cdot\max\Big\{1,p\log\frac{1}{\delta}\Big\}\Big)^{-1}$$ for some constant $C_6$. Then with at least probability $1-\delta/3$, we have 
		\begin{align}
			m_a\leq a_t^\top{a}^* \leq M_a\quad \mathrm{and}\quad(w_{\tau_{11}},a_{\tau_{11}})\in\cA_{C_3}\backslash\cA_{C_2,C_3}
		\end{align} for all $t$'s such that $\tau_{11}\leq t\leq T=\tilde{O}(\eta^{-2})$, where
		$$\tau_{11}=\tilde{O}\Big(\frac{p}{\eta}\log\frac{1}{\delta}\Big).$$
	\end{lemma}
	As can be seen, after $\tau_{11}$ iterations, the algorithm enters $ \cA_{C_3}\backslash\cA_{C_2,C_3}$. Then our following analysis will only consider Case (2), where $(w_0,a_0)\in\cA_{C_3}\backslash\cA_{C_2,C_3}$. We remark that although Theorem \ref{SC} no longer guarantees the dissipativity of the perturbed gradient with respect to $a$, Lemma \ref{stage1_a} can ensure the optimization error of $a_t$ within Phase I to be nonincreasing as long as $t\geq \tau_{11}$ with high probability.
	
	We then continue to characterize the optimization error of $w_t$. Recall that the noise injection prevents $-w^*$ from being attracted to $-w^*$. Thus, we can guarantee that $w_t$ is sufficiently distant from $-w^*$ after sufficiently many iterations, as shown in the next lemma.
	\begin{lemma}\label{lemma_w_escape}
		Suppose $\rho_w^0=C_w kp^2\geq1$, $\rho_a^0=C_a$, $(w_0,a_0)\in\cA_{C_3}\backslash\cA_{C_2,C_3}$ and $m_a\leq a_t^\top{a}^* \leq M_a$ holds for all $t$'s. For any $\delta\in(0,1), $ we choose step size $$\eta=C_6\Big(k^4p^6\cdot\max\Big\{1,p\log\frac{1}{\delta}\Big\}\Big)^{-1}$$  for some constant $C_6$. Then with at least probability $1-\delta/3,$ there exists $$\tau_{12}=\tilde{O}\Big(\frac{p}{\eta}\log\frac{1}{\eta}\log\frac{1}{\delta}\Big)$$ such that $w_{\tau_{12}}^T w^*\geq C_4$ for some constant $C_4\in(-1,0)$.
	\end{lemma}
	
	Lemma \ref{lemma_w_escape} implies that the algorithm eventually attains $\cK_{C_4,m_a,M_a}$, where the dissipativity of the perturbed gradient with respect to $w$. Then we can bound the optimization error of $w_t$ by the next lemma.
	\begin{lemma}\label{stage-nice}
		Suppose $\rho_w^0=C_w kp^2\geq1$, $\rho_a^0=C_a$, $(w_0,a_0)\in \cK_{C_4,m_a,M_a}$ and $m_a\leq a_t^\top{a}^* \leq M_a$ holds for all $t$'s. For any $\delta\in(0,1),$  we choose step size $$\eta=C_6\Big(k^4p^6\cdot\max\Big\{1,p\log\frac{1}{\delta}\Big\}\Big)^{-1}.$$  Then with at least probability $1-\delta/3$, we have
		\begin{align}
			%m_a\leq a_t^\top{a}^* \leq M_a\quad \text{and}\quad
			\phi_t\leq 5\pi/12
		\end{align} 
		for all $t$'s such that $\tau_{13}\leq t\leq T=\tilde{O}(\eta^{-2})$, where
		$$\tau_{13}=\tilde{O}\Big(\frac{p}{\eta}\log\frac{1}{\delta}\Big).$$  
	\end{lemma}
	Lemma \ref{stage-nice} implies that after $(w_t,a_t)$ enters $\cK_{C_4,m_a,M_a}$, it starts to make progress towards $w^*$. Due to the large injected noise, however, the optimization error of $w_t$ can only attain a large optimization error. Although the optimization error of $a_t$ is also large, $(w_t,a_t)$ can be guaranteed to escape from the spurious local optimum.
	
	The proof of Lemmas \ref{stage1_a}--\ref{stage-nice} requires supermartingale-based analysis, which is very involved and technical. See more details in the appendix \ref{pf_stage1_w}. Combining all above lemmas, we take $T_1=\tau_{11}+\tau_{12}+\tau_{13}$, and complete the proof of Theorem \ref{stage1_w}.
\end{proof}

As can be seen, $a$ can not make further progress after escaping $\cA_{C_2,C_3},$ even when $w$ is more accurate.  This is because the injected noise is too large and ruins the accuracy of $w.$ We need decrease the noise level to guarantee convergence.

\subsection{Phase II: Converging to the Global Optimum}\label{section_stage2}

We then characterize the convergence behavior of the perturbed GD algorithm in Phase II. Recall that in Phase I, the injected noise helps perturbed GD get closer to the global optimum without being trapped in the spurious optimum. Without loss of generality, we restart the iteration index and assume that the initialization $(w_0,a_0)$ follows the result in Theorem \ref{stage1_w} :
$$0<m_a\leq a_0^\top{a}^* \leq M_a\quad\textrm{and}\quad\phi_0 \leq \frac{5}{12}\pi
,$$
where  $m_a=\frac{C_4}{p}\norm{{a}^*}_2^2$ and $M_a=4(\one^\top{a}^*)^2  + (3+\frac{C_7}{p}))\norm{{a}^*}_2^2$.
%Recall that this solution is inaccurate and the perturbed GD can hardly make any further progress due to the large noise. Thus, we decrease the noise level  in Phase II and show that our proposed algorithm finally converges to the global optimum with high probability.

The next theorem shows that given the reduced injected noise, the perturbed gradient of $L$ with respect to $w$ and $a$ satisfies dissipativity, respectively.
\begin{theorem}\label{thm_ASC}
	For any $\gamma>0,$ we choose $\rho_w^1\leq C_w^1\frac{\gamma}{kp}< 1$ and $\rho_a^1\leq C_a^1$ for small enough constants $C_w^1$ and $C_a^1$. Let $\xi\sim\mathrm{unif}(\BB_0(\rho_w^1))$ and $\epsilon\sim\mathrm{unif}(\BB_0(\rho_a^1))$. For any $C_9\in(-1,1]$ and $M>m>0$, the perturbed manifold gradient of $L$ w.r.t. $w$ satisfies
	\begin{align*}
		&\langle-\mathbb{E}_{\xi,\epsilon}(I-ww^\top)\nabla_w\cL(w+\xi,a+\epsilon),w^*-w\rangle	\geq \frac{m(1+C_9)}{16}\norm{w-w^*}_2^2-\gamma
	\end{align*}
	for any $(w,a)\in\cK_{C_9,m,M}$, where
	\begin{align*}%\label{one_point_w2}
		\cK_{C_9,m,M}=\big\{(w,a)~\big|~a^\top {a}^*\in [m,M],~ w^\top w^*\geq C_9,~w\in\SSS_0(1)\big\}.
	\end{align*}
	Moreover, for any $0<m<M$ and $C_{10}>0$, the perturbed gradient of L w.r.t. $a$ satisfies
	\begin{align*}
		\langle-\mathbb{E}_{\xi,\epsilon}\nabla_a\cL(w+\xi,a+\epsilon),a^*-a\rangle
		\geq \frac{\pi-1}{2\pi}\norm{a-a^*}_2^2 -\gamma
	\end{align*}
	for any $(w,a)\in\cR_{m,M,C_{10}}$, where
	\begin{align*}
		\cR_{m,M,C_{10}}=\big\{(w,a)\big|m\leq a^\top a^*\leq M,~ w\in\SSS_0(1),~\norm{w-w^*}_2^2\leq C_{10}\gamma\big\}.
	\end{align*}
\end{theorem}
The detailed proof is provided in Appendix \ref{pf_ASC}.  Note that the  dissipativity with respect to $a$ depends on the accuracy of $w,$ which indicates that  convergence of $a$ happens after that of $w.$ This phenomenon can be seen in the proof of next theorem  analyzing the algorithmic behavior in Phase II.

\begin{theorem}\label{stage2}
	Suppose $\phi_0 \leq \frac{5}{12}\pi,~
	0<m_a\leq a_0^\top{a}^* \leq M_a.$  For any $\gamma>0,$ we choose $\rho_w^1\leq C_w^1\frac{\gamma}{kp}< 1$ and $\rho_a^1\leq C_a^1$ for small enough constants $C_w^1$ and $C_a^1$.
	For any $\delta \in (0,1)$, we choose step size $$\eta=C_{11}\Big(\max\Big\{k^4p^6,\frac{k^2p}{\gamma}\Big\}\max\Big\{1,p\log\frac{1}{\gamma}\log
	\frac{1}{\delta}\Big\}\Big)^{-1}$$ for some constant $C_{11}$.
	Then with at least probability $1-\delta$, we have
	$$\norm{w_t-w^*}_2^2\leq C_{12}\gamma\quad\textrm{and}\quad\norm{a_t-a^*}_2^2\leq \gamma$$ for any $t$'s such that $T_2\leq t\leq T=\tilde{O}(\eta^{-2})$, where
	$C_{12}$ is a constant and $$T_2=\tilde{O}\Big(\frac{p}{\eta}\log\frac{1}{\gamma}\log
	\frac{1}{\delta}\Big).$$
\end{theorem}
Theorem \ref{stage2} shows that Phase II of our proposed algorithm only needs polynomial time to ensure  the convergence to the global optimum with high probability, when the noise is small enough.  Due to space limit, here we only provide a proof sketch of Theorem \ref{stage2}. See more details and the proof of all technical lemmas in Appendix \ref{pf_phase2}.

\begin{proof}[\textit{\textbf{Proof Sketch}}]
	The perturbed GD is already in the solution set of Phase I, which is actually in the dissipative region $\cK_{C_9,m,M}.$ The first lemma shows that even if the noise is reduced,  our proposed algorithm never escape this set.
	\begin{lemma}\label{lem_keep}
		Define $\phi_t(\xi)=\angle(x_t+\xi,x^*).$ Assume there exists some constant  $C_8$ such that $1+C_8/p\leq \EE_\xi g(\phi_t(\xi)) \leq \pi$  and $\EE_\xi\phi_t(\xi)\leq \frac{3\pi}{4}$ for all $t.$ Suppose $$0<m_a\leq a_0^\top{a}^* \leq M_a\quad\textrm{and}\quad\phi_0 \leq \frac{5}{12}\pi.$$
		For any $\delta \in (0,1)$, we choose step size$$\eta=C_{11}\Big(\max\Big\{k^4p^6,\frac{k^2p}{\gamma}\Big\}\max\Big\{1,p\log\frac{1}{\gamma}\log\frac{1}{\delta}\Big\}\Big)^{-1}$$ for some constant $C_{11}.$ Then with at least probability at lease $1-\delta/3$, we have for all $t\leq T=\tilde{O}(\eta^{-2})$,
		$$0<m_a^\prime\leq a_t^\top{a}^* \leq M_a^\prime\quad\textrm{and}\quad\phi_t\leq \frac{11}{24}\pi
		,$$
		where $m_a^\prime=m_a/2$, $M_a^\prime=3M_a$. 
	\end{lemma}
	Lemma \ref{lem_keep} shows that throughout sufficiently many iterations of Phase II, $(w_t, a_t)$'s are at least as accurate as the initial solution with high probability. Thus, we can guarantee that the perturbed GD algorithm stays away from the spurious local optimum, and the benign optimization landscape in Theorem \ref{thm_ASC} holds. 
	
	The next lemma characterizes the convergence properties of the perturbed GD algorithm for $w$.
	\begin{lemma}\label{stage2_w}
		Suppose $\phi_t \leq \frac{11}{24}\pi$ and $
		0<m_a^\prime\leq a_t^\top{a}^* \leq M_a^\prime$ hold for all $t.$  For any $\gamma>0,$ we choose $\rho_w^1\leq C_w^1\frac{\gamma}{kp}< 1$ and $\rho_a\leq C_a^1$ for small enough constant $C_w^1$ and $C_a^1$. For any $\delta \in (0,1)$, we choose step size $$\eta=C_{11}\Big(\max\Big\{k^4p^6,\frac{k^2p}{\gamma}\Big\}\max\Big\{1,p\log\frac{1}{\gamma}\log\frac{1}{\delta}\Big\}\Big)^{-1}$$ for some constant $C_{11}$. Then with at least probability at least $1-\delta/3$, we have
		$$\norm{w_t-w^*}_2^2\leq C_{12}\gamma$$ for all t's such that $\tau_{21}\leq t\leq\tilde{O}(\eta^{-2})$, where
		$C_{12}$ is a constant and $$\tau_{21}=\tilde{O}\Big(\frac{p}{\eta}\log\frac{1}{\gamma}\log
		\frac{1}{\delta}\Big).$$
	\end{lemma}
	Lemma \ref{stage2_w} shows that at $\tau_{21}$ iterations, the perturbed GD algorithm enters $\cR_{m_a^\prime,M_a^\prime,C_{12}}$. Then we can characterize its convergence properties for $a$, as shown in the next lemma.
	\begin{lemma}\label{stage2_a}
		Suppose $(w_t,a_t)\in\cR_{m_a^\prime,M_a^\prime,C_{12}}$ holds for all t.
		For any $\gamma>0,$ we choose $\rho_w^1\leq C_w^1\frac{\gamma}{kp}< 1$ and $\rho_a\leq C_a^1$ for small enough constant $C_w^1$ and $C_a^1$. For any $\delta \in (0,1)$, we choose step size$$\eta=C_{11}\Big(\max\Big\{k^4p^6,\frac{k^2p}{\gamma}\Big\}\max\Big\{1,p\log\frac{1}{\gamma}\log
		\frac{1}{\delta}\Big\}\Big)^{-1}$$ for some constant $C_{11}$. Then with at least probability $1-\delta/3$, we have
		$$\norm{a_t-{a}^*}_2^2\leq \gamma$$ for all $t$'s such that $\tau_{22}\leq t\leq\tilde{O}(\eta^{-2})$, where $$\tau_{22}=\tilde{O}\Big(\frac{p}{\eta}\log\frac{1}{\gamma}\log
		\frac{1}{\delta}\Big).$$
	\end{lemma}
	Similar to Lemmas \ref{stage1_a}--\ref{stage-nice}, the proof of Lemmas \ref{lem_keep}--\ref{stage2_a} also requires supermartingale-based analysis. See more details in Appendix \ref{pf_phase2}.
	
	Combining the above lemmas together, we take $T_2=\tau_{21}+\tau_{22}$, and complete the proof of Theorem \ref{stage2}.
\end{proof}

\section{Extension to Perturbed SGD}\label{extension}

Our analysis can be further extended to the perturbed mini-batch stochastic gradient descent (perturbed SGD) algorithm. Specifically, we solve
\begin{align}\label{opt-recall-bounded}
	\min_{w,a} \cL(w,a)\quad\text{subject to}\quad w^\top w=1,~\norm{a}_2\leq R,
\end{align}
where $R$ is some tuning parameter. At the $t$-the iteration, we independently sample Gaussian random matrices $Z^{(1)},...,Z^{(m)}$,  where $m$ is the batch size, and obtains the stochastic approximation of $\nabla\cL(w,a)$ by
\begin{align*}
	\nabla\hat{\cL}(w,a) = \nabla\frac{1}{m}\sum_{i=1}^m\ell(w,a,Z^{(i)})= \frac{1}{m}\sum_{i=1}^m\nabla\ell(w,a,Z^{(i)}),
\end{align*}
where $\nabla_w\ell(w,a,Z)$ and $\nabla_a\ell(w,a,Z)$ take the form as follows,
\begin{subequations}\label{stoc_approx}
	\begin{align}
		\nabla_w\ell(w,a,Z) &= \left(\sum_{j=1}^ka_ja_j^*Z_jZ_j^\top\mathds{1}(Z_j^\top w\geq 0, Z_j^\top w^*\geq0)+\sum_{j\neq i}a_ia_j^*Z_iZ_j^\top\mathds{1}(Z_j^\top w\geq 0, Z_j^\top w^*\geq0)\right)w^*,\\
		\nabla_a\ell(w,a,Z) &= \sigma(Zw)\sigma(Zw)^\top a-\sigma(Zw)\sigma(Zw^*)^\top a^*.
	\end{align}
\end{subequations}
The perturbed SGD algorithm then takes
\begin{align*}
	a_{t+1}& = \Pi_{\BB_{0}(R)}\big(a_{t} - \eta \nabla_a\hat \cL(w_t + \xi_t,a_t + \epsilon_t)\big),\\
	w_{t+1} &= \mathrm{Proj}_{\mathbb{S}_0(1)} \big(w_t-\eta (I-w_tw_t^\top)\nabla_w\hat \cL(w_t + \xi_t,a_t + \epsilon_t)\big),
\end{align*}
where $\eta$ is the learning rate, and $\Pi_{\BB_{0}(R)}(\cdot)$ denotes the projection operator to $\BB_{0}(R)$.

%\fbox{First, explain concentration}

Since $Z$ is a Gaussian random matrices with independent entries and $w$ is on the unit sphere, $\sigma(Zw)$ follows a half-normal distribution with variance $(1-\pi/2)$. Therefore, one can verify that all entries of $Z_jZ_j^\top\mathds{1}(Z_j^\top w\geq 0, Z_j^\top w^*\geq0),$ $Z_iZ_j^\top\mathds{1}(Z_j^\top w\geq 0, Z_j^\top w^*\geq0),$ $\sigma(Zw)\sigma(Zw)^\top$ and $\sigma(Zw)\sigma(Zw^*)^\top$ are sub-exponential  random variable with $O(1)$ mean and variance proxy. %Suppose $a$ is bounded, then the stochastic approximation of the gradient in \eqref{stoc_approx} is a linear combination of sub-exponential random variables.
We then can characterize the estimation error of the stochastic gradient as follows.
\begin{lemma}\label{stoc_approx_concentration}
	Suppose that for any $\delta,\epsilon>0,$ $w\in\SSS_0(1)$ and $a\in\BB_{0}(R),$ given a mini-batch size $$m=\mathrm{poly}\left(p,k,R,\frac{1}{\epsilon},\log\frac{1}{\delta}\right),$$ with at least probability $1-\delta$, we have 
	\begin{align*}
		\norm{\nabla_w\hat{\cL}(w,a)-\nabla_w\cL(w,a)}_2^2\leq\epsilon\quad\textrm{and}\quad\norm{\nabla_a\hat{\cL}(w,a)-\nabla_a\cL(w,a)}_2^2\leq\epsilon.
	\end{align*}
\end{lemma}
The proof of Lemma \ref{stoc_approx_concentration} is straightforward (by simple union bound and the concentration properties of sub-exponential random variable), and therefore omitted. Lemma \ref{stoc_approx_concentration} implies that as long as the batch size is sufficiently large, we can show the mini-batch stochastic gradient is sufficiently accurate with high probability. Then we can adapt the convergence analysis in Section \ref{sec_theory}, and show that P-SGD can avoid spurious local optimum with high probability in Phase I.
\begin{theorem}[P-SGD escapes the spurious local optimum]
	Suppose $\norm{a^*}_2\leq R$, $\rho_w^0=C_w^0 kp^2\geq 1, ~\rho_a^0=C_a^0, ~  a_0\in\BB_0\Big(\frac{|\one^\top{a^*}|}{\sqrt{k}}\Big)$ and $w_0\in \SSS_0(1).$ For any $\delta\in(0,1),$ we choose a small enough step size $$\eta=\left(\mathrm{poly}\left(p,k,R,\log\frac{1}{\delta}\right)\right)^{-1}$$ and a large enough mini batch-size $$m=\mathrm{poly}\left(p,k,R,\log\frac{1}{\delta}\right),$$ then with at least probability $1-\delta$, we have 
	\begin{align*}
		m_a\leq a_t^\top{a}^* \leq M_a\quad\text{and}\quad\phi_t\leq \frac{5}{12}\pi
	\end{align*}
	for all $t$'s such that $\hat T_1\leq t\leq \tilde{O}(\eta^{-2})$, where $m_a$ and $M_a$ are some constants and $$
	\hat T_1=\mathrm{poly}\Big(p,k,\log\frac{1}{\delta}\Big).$$
\end{theorem}
Similarly, for Phase II, we can show that P-SGD converges to the global optimum with high probability.
\begin{theorem}[P-SGD converges to the global optimum]
	Suppose $\norm{a^*}_2\leq R$,  $\phi_0 \leq \frac{5}{12}\pi,~
	0<m_a\leq a_0^\top{a}^* \leq M_a.$  For any $\gamma>0,$ we choose $\rho_w^1\leq C_w^1\frac{\gamma}{\sqrt{kp}}\leq 1$ and $\rho_a\leq M_a$ for some constant  $C_w^1.$
	For any $\delta \in(0,1)$, we choose a small enough step size $$\eta=\left(\mathrm{poly}\left(p,k,R,\frac{1}{\gamma},\log\frac{1}{\gamma},\log\frac{1}{\delta}\right)\right)^{-1},$$ and a large enough batch size $$m=\mathrm{poly}\left(p,k,R,\frac{1}{\gamma},\log\frac{1}{\delta}\right),$$
	then with at least probability $1-\delta$, we have
	$$\norm{w_t-w^*}_2^2\leq C_{13}\gamma\quad\textrm{and}\quad\norm{a_t-a^*}_2^2\leq \gamma$$ for all $t$'s such that $\hat T_2\leq t\leq T=\tilde{O}(\eta^{-2})$, where
	$C_{13}$ is a constant and $$\hat T_2=\mathrm{poly}\Big(p,k,\frac{1}{\gamma},\log\frac{1}{\delta}\Big).$$
\end{theorem}
The proof of Lemma \ref{stoc_approx_concentration} is straightforward and therefore omitted, as the error of the mimi-batch stochastic gradient has been well controlled by a sufficiently large batch-size. %We remark that we study the perturbed-SGD algorithm 

%\begin{proof}[Proof Sketch]The proof consists of two steps. We first show that when the mini-batch size is large enough 
%	
%\end{proof}

\section{Numerical Experiment}\label{sec_num}
%\vspace{-0.05in}

We present numerical experiments to compare our perturbed GD algorithm with GD and SGD.

We first demonstrate that our perturbed GD algorithm with the noise annealing guarantees global convergence to the global optimum. We consider the training of non-overlapping two-layer convolutional neural network in \eqref{NN-model} with varying $a^*$ and $k$. Specifically, we adopt the same experimental setting as in \citet{du2017gradient}. We set $p=6$ with $k\in\{25,36,49,64,81,100\}$ and $a^*$ satisfying
\begin{align*}
	\frac{\one^\top a^*}{\norm{a^*}_2^2} \in \{0,1,4,9,16,25\}.
\end{align*}
For the perturbed GD algorithm, we perform step size and noise annealing in an epoch-wise fashion: each simulation has $20$ epochs with each epoch consisting of $400$
iterations; The initial learning rate is $0.1$ for both $w$ and $a$, and geometrically decays with a ratio $0.8$; The initial noise levels are given by $(\rho_w, \rho_a) = (36, 1)$ and both geometrically decay with a ratio $0.4$. For GD, the learning rate is $0.1$ for both $w$ and $a$. For SGD, we adopt a batch size of $4$, and perform step size annealing in an epoch-wise fashion: The initial learning rate is $0.1$, and geometrically decays with a ratio $0.4$. For perturbed GD and SGD, we purposely initialize at the spurious local optimum. For GD, we adopt the random initialization, as suggested in \citet{du2017gradient}.

For each combination of $k$ and $a^*$, we repeat $1000$ simulations for all three algorithms, and report the success rate of converging to the global optimum in Table \ref{tab:prob_global}. As can be seen, perturbed GD and SGD are capable of escaping from the spurious local optimum (even if they are initialized there), and converge to the global optimum throughout all $1000$ simulations. However, GD with random initialization can be trapped at the spurious local optimum for up to about $500$ simulations. These results are consistent with our theoretical analysis and \citet{du2017gradient}.

\begin{table*}[htb!]
	\caption{\it Success rates of converging to the global optimum for perturbed GD/GD/SGD with varying $k$ and $a^*$ and $p=6$.}
	\begin{center}
		\footnotesize
		\begin{tabular}{l|cccccc }
			\hline
			$\one^\top a^*/\norm{a^*}_2^2$   & 0 & 1 & 4 & 9 & 16 & 25 \\
			\hline
			$k=25$ & 1.00/{\color{red} 0.50}/1.00 & 1.00/{\color{red}0.55}/1.00 & 1.00/{\color{red}0.73}/1.00 & 1.00/1.00/1.00 & 1.00/1.00/1.00 & 1.00/1.00/1.00 \\
			$k=36$ & 1.00/{\color{red} 0.50}/1.00 & 1.00 /{\color{red}0.53}/1.00 & 1.00/{\color{red}0.66}/1.00 & 1.00/{\color{red}0.89}/1.00 & 1.00/1.00/1.00 & 1.00/1.00/1.00\\
			$k=49$ & 1.00/{\color{red} 0.50}/1.00 & 1.00/{\color{red} 0.53}/1.00 & 1.00/{\color{red} 0.61}/1.00 & 1.00/{\color{red} 0.78}/1.00 & 1.00/1.00/1.00 & 1.00/1.00/1.00 \\
			$k=64$ & 1.00/{\color{red} 0.50}/1.00 & 1.00/{\color{red} 0.51}/1.00 & 1.00/{\color{red} 0.59}/1.00 & 1.00/{\color{red} 0.71}/1.00 & 1.00/{\color{red} 0.89}/1.00 & 1.00/1.00/1.00\\
			$k=81$ & 1.00/{\color{red} 0.50}/1.00 & 1.00/{\color{red} 0.53}/1.00 & 1.00/{\color{red} 0.57}/1.00 & 1.00/{\color{red} 0.66}/1.00 & 1.00/{\color{red} 0.81}/1.00 & 1.00/{\color{red} 0.97}/1.00\\
			$k=100$ & 1.00/{\color{red} 0.50}/1.00 & 1.00/{\color{red} 0.50}/1.00 & 1.00/{\color{red} 0.57}/1.00 & 1.00/{\color{red} 0.63}/100 & 1.00/{\color{red} 0.75}/1.00 & 1.00/{\color{red} 0.90}/1.00\\
			\hline
		\end{tabular}
	\end{center}
	\label{tab:prob_global}
	%\vspace{-0.2in}
\end{table*}

\begin{figure}[htb!]
	\centering
	\includegraphics[width = 0.9\linewidth]{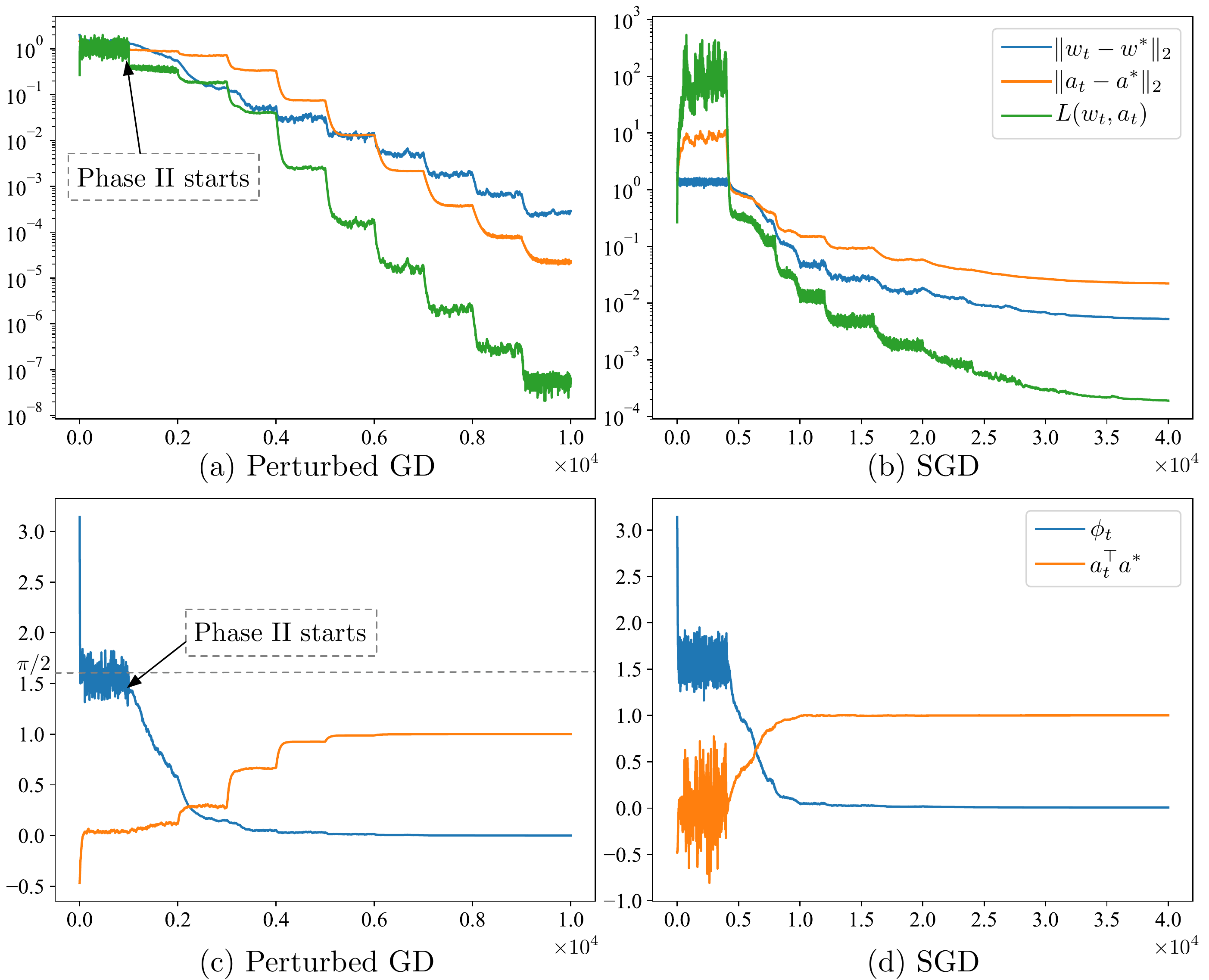}
	\caption{\it Algorithmic Behavior of Perturbed GD and SGD.}
	\label{fig:pgd_sgd}
\end{figure}

We then demonstrate the algorithmic behavior of the perturbed GD algorithm and compare it with SGD. We set $p=6$, $k=100$, $a_j^*=-0.1$ for $j=1,...,50$ and $a_j^*=0.1$ for $j=51,...,100$, and $w$ is randomly generated from the unit sphere. Our selected $a^*$ satisfies $\frac{\mathbf{1}^\top a^*}{\norm{a^*}_2^2} = 0$. As suggested by Table \ref{tab:prob_global}, this is a difficult case, where GD may get stuck at the spurious local optimum with about 0.5 probability. For the perturbed GD algorithm, we perform step size and noise annealing in an epoch-wise fashion: Each simulation has $10$ epochs with each epoch consisting of $1000$
iterations; The initial learning rate is $0.1$ for both $w$ and $a$, and geometrically decays with a ratio $0.8$; The initial noise levels are given by $(\rho_w, \rho_a) = (36, 1)$ and both geometrically decay with a ratio $0.4$.
For SGD, we adopt a batch size of $4$, and perform step size annealing in an epoch-wise fashion:
Each simulation has $10$ epochs with each epoch consisting of $4000$ iterations;
The initial learning rate is $0.1$, and geometrically decays with a ratio $0.4$.
We repeat $10$ simulations for both perturbed GD and SGD, and report their (averaged) trajectories in Figure \ref{fig:pgd_sgd}.

As can be seen, the trajectories of the perturbed GD algorithm have a phase transition by the end of the second epoch. At the first two epochs, the algorithm makes very slow progress in optimizing $a$ and $w$ due to the large injected noise. Starting from the third epoch, we see that $a_t^\top a^*$ becomes positive and gradually increases, and $\phi_t$ further decreases. This implies that the algorithm has escaped from the spurious local optimum. Eventually, at later epochs, we see that the algorithm converges to the global optimum, as the magnitude of the injected noise is reduced. These observations are consistent with our theory.

Moreover, we can see that the trajectories of SGD actually show similar patterns to those of the perturbed GD algorithm. At early epochs, only slow progress is made towards optimizing $w$ and $a$. At later epochs, once SGD escapes from the spurious local optimum, we observe its convergence to the global optimum. Since the noise of SGD comes from the data and has a larger variance than that of the injected noise for the perturbed GD algorithm, we observe more intense oscillation in the trajectories of SGD.

\section{Discussions}\label{sec_discuss}

\noindent {\bf Partial Dissipativity v.s. \citet{kleinberg2018alternative}.} \citet{kleinberg2018alternative} study the convolutional effect of the noise in nonconvex stochastic optimization, and provide new insights on training neural networks using SGD. Their analysis, however, involves an unconventional asumption. Specifically, they consider a general unconstrained minimization problem $\min_{x}\EE_{\xi} f(x,\xi)$, and assume 
\begin{align}\label{actual-condition}
	\langle -\nabla\EE_{\xi} f(x-\eta\nabla f(x,\xi)),~x^*-[x-\eta\nabla \EE_\xi f(x,\xi)]\rangle\geq c\norm{x^*-[x-\eta\nabla \EE_\xi f(x,\xi)]}_2^2,
\end{align}
where $\eta$ is the step size of the SGD algorithm. Note that their assumption is essentially imposed over both the optimization problem and the SGD algorithm\footnote{The conventional analyses usually impose assumptions on the optimization problem, and all properties of the algorithm need to be proved under the assumptions.}. However, they do not provide any theoretical evidence showing that such a complicated assumption holds, when applying SGD to any specific nonconvex optimization problem.

The experimental results in \citet{kleinberg2018alternative} attempt to make some empirical validations of their assumption for training neural networks. Specifically, throughout every iterations of training ResNets and DenseNets, they empirically verify that the following condition holds
\begin{align}\label{empirical-condition}
	\left<-\frac{1}{m}\sum_{i=1}^{m}\EE_{\xi}\nabla f({x}_t+\omega_i,\xi),~ x_t-x^*\right>\geq 0,
\end{align}
where $\omega_i$'s are independently sampled from a uniform distribution over $\BB_0(0.5)$ and $m=100$.
Note that \eqref{empirical-condition} is different from their actual assumption \eqref{actual-condition}.

In contrast, our analysis is dedicated to training two-layer non-overlapping convolutional neural networks in the teacher/student network setting. The partial dissipative condition used in our analysis can been rigorously verified in Theorems \ref{thm_SC} and \ref{thm_ASC}. Moreover, we want to remark that the partial dissipative condition in our analysis is theoretically more challenging, since (1) it does not hold globally; (2) it does not jointly hold over the convolutional weight $w$ and the output weight $a$; (3) we need to handle the additional errors (e.g., $\gamma_a$ and $\gamma_w$).

%is essentially a weaker condition, is needed and more suitable from the following two aspects. First, all the layers do not necessarily converge simultaneously, i.e., some layers may need other layers to be precise enough to make progress.  Second, to find an $\gamma$-approximate solution, a small negative constant depending on the accuracy is tolerated in the condition like in \eqref{con_sc}. 

\noindent{\bf Connections to SGD.} The motivation of this paper is to understand the role of the noise in training neural network, however, due to the technical limit, directly analyzing SGD is very difficult. The noise of SGD comes from the random sampling of the training data, and it may have a very complex distribution. Moreover, the noise of SGD depends on the iterate, and therefore yields very complicated dependence through iterations. These challenging aspects are beyond our theoretical understanding.

The perturbed GD algorithm considered in this paper is essentially imitating SGD, but easier to analyze: The injected noise follows a uniform distribution and independent on the iterates. Though simpler than SGD, the perturbed GD algorithm has often been observed sharing similar algorithmic behavior to SGD. We remark that from a theoretical perspective, the perturbed GD algorithm is still highly non-trivial and challenging.

\noindent{\bf Connection to Step Size Annealing.} The noise annealing approach is actually closely related to the step size annealing, which has been widely used in training neural networks by SGD. The variance of the noise of SGD has an explicit quadratic dependence on the step size. Therefore, a commonly used practical step size annealing is essentially annealing the noise in training neural networks.

However, we remark that varying step size is actually more complicated than varying noise. When the step size is large, it not only enlarges the noise of SGD, but also encourages aggressive overshooting. This is still beyond our theoretical understanding, as our analysis for the perturbed GD algorithm uses small step sizes with large injected noise.

\noindent{\bf Algorithmic Behaviors for Training Different Layers.} Our analysis shows that the perturbed GD algorithm behaves differently for training the convolutional weight $w$ and the output weight $a$ in Phase I: the algorithm first makes progress in training $a$, and then makes progress in training $w$. It is not clear whether this is an artifact of our proof. We believe that some empirical investigations are needed, e.g., examining the training of practical large networks.

\noindent{\bf Multi-epoch Noise Annealing.} Our analysis in Section \ref{sec_theory} can be extended to the multi-epoch setting. For instance, we consider a noise level schedule $\{\rho_w^s\}_{s=1}^S$, $\{\rho_a^s\}_{s=1}^S$. When applying our analysis, we can show that there exists a phase transition along the schedule. For the earlier epochs with $\rho_w^s\geq C_w^0kp^2$ and $\rho_a^s\geq C_a^0$, the algorithm is gradually escaping from the spurious local optimum, which is similar to our analysis for Phase I; For the later epochs with smaller noises, the algorithm is gradually converging to the global optimum, which is similar to our analysis for Phase II.

\noindent{\bf Overparameterized Neural Networks.} Our analysis only considers the regime, where the student network has the same architecture as the {\bf teacher} network. This is different from practical situations, where the {\bf student} network is often overparameterized. We conduct some empirical studies on a simple overparameterized case, where the {\bf student} network has two convolutional filters and the {\bf teacher} network has only one convolutional filter. Our studies suggest that such a simple overparameterization does not necessarily lead to a better optimization landscape. There still exist spurious local optima, which can trap the GD algorithm. Due to the space limit, we present the details in Appendix \ref{add_exp}.

\noindent{\bf Other Related Works.} We briefly discuss several other related works. These works consider different problems, algorithms and assumptions. Therefore, the results are not directly comparable. Specifically, \citet{zhou2017stochastic} study the stochastic mirror descent ({\it different from ours}) under a global variational coherent assumption ({\it does not hold for our target problem}); \citet{li2017convergence} study SGD ({\it different from ours}) for training ResNet-type two-layer neural networks. They assume that the weight of the second layer is known (all one), and prove that the optimization landscape satisfies the one-point convexity over a small neighborhood of the global optimum ({\it does not hold for our target problem}); \citet{jin2018local} show that the perturbed SGD algorithm ({\it different from ours}) for minimizing the empirical risk ({\it we consider the population risk}), and show that the injected noise rules out the spurious local optima of the empirical risk. However, their assumption requires the population risk to have no spurious local optima ({\it our population risk contains a spurious local optimum}).

\newpage
\appendix
\section{Preliminaries}\label{app_pre}

We first present the following proposition, which computes the explicit form of the loss function and the gradient of the loss function with respect to $a$ and $w$.
\begin{proposition}[\citet{du2017gradient}]
	Let  $\phi\in[0,\pi]$ be the angle between $w$ and $w^*$. Then, the loss function $L\left(w,a\right)$ and the gradient w.r.t $\left(w,a\right)$, i.e., $\nabla_aL\left(w,a\right)$ and $\nabla_wL\left(w,a\right)$ have the following analytic forms.
	
	\begin{align*}
		L\left(w,a\right) &=\frac{1}{2}[\frac{\left(\pi-1\right)\norm{w^*}_2^2}{2\pi}\norm{a^*}_2^2
		+\frac{\left(\pi-1\right)}{2\pi}\norm{a}_2^2
		-\frac{\norm{w^*}_2}{\pi}\left(g\left(\phi\right)-1\right)a^\top a^*\\
		&\hspace{2.0in}+\frac{\norm{w^*}_2^2}{2\pi}\left(\one^\top a^*\right)^2
		+\frac{1}{2\pi}\left(\one^\top a\right)^2
		-\frac{\norm{w^*}_2}{\pi}\one^\top a^* a^\top \one],\\
		\nabla_aL\left(w,a\right)&=\frac{1}{2\pi}\left(\one\one^\top+\left(\pi-1\right)I\right)a
		-\frac{1}{2\pi}\left(\one\one^\top+\left(g\left(\phi\right)-1\right)I\right)a^*\norm{w^*}_2,\\
		\nabla_wL\left(w,a\right)
		&=-\frac{a^\top a^* \left(\pi -\phi\right)}{2\pi}w^* 
		+\left(\frac{\norm{a}}{2}+\frac{\sum_{i\neq j}a_ia_j}{2\pi}
		-\frac{a^\top a^* \sin\phi}{2\pi}\frac{\norm{w^*}_2}{\norm{w}_2}
		-\frac{\sum_{i\neq j}a_i a_j^*}{2\pi}\frac{\norm{w^*}_2}{\norm{w}_2}\right)w,
	\end{align*}
	where $
	g\left(\phi\right)=\left(\pi-\phi\right)\cos\phi+\sin\phi.
	$
\end{proposition}

As can be seen, both $\nabla_w L\left(w,a\right)$ and  $\nabla_a L\left(w,a\right)$ depend on $\phi$, which is the angle between $w$ and $w^*$. After injecting noise, we have $$\EE{\phi_\xi}=\EE\left\{\arccos{\frac{\left(w+\xi_t\right)^\top w^*}{\norm{w+\xi_t}_2\norm{w^*}_2}}\right\}\neq\arccos{\frac{w^\top w^*}{\norm{w}_2\norm{w^*}_2}}=\phi.$$

As a direct result, we have
\begin{align*}
	\EE_{\xi_t, \epsilon_t}{\nabla_a L\left(w_t + \xi_t,a_t + \epsilon_t\right)}&\neq \nabla_aL\left(w_t ,a_t\right),\\
	\EE_{\xi_t, \epsilon_t}{\nabla_w L\left(w_t + \xi_t ,a_t + \epsilon_t\right)}&\neq \nabla_wL\left(w_t  ,a_t\right),
\end{align*}
which indicate that the perturbed gradient $\nabla_a\left(w_t + \xi_t,a_t + \epsilon_t\right)$, $\nabla_w\left(w_t + \xi_t ,a_t + \epsilon_t\right)$ are  biased estimates of the gradient (as we mentioned in Section \ref{sec_model}). 

For notational simplicity, we introduce an auxiliary iterate $\tilde{w}_{t+1}$ and rewrite our perturbed GD algorithm as follows.
\begin{align*}
	a_{t+1}& = a_{t} - \eta \nabla_a L\left(w_t + \xi_t,a_t + \epsilon_t\right),\\
	\tilde{w}_{t+1} &= w_t-\eta \left(I-w_tw_t^\top\right)\nabla_w L\left(w_t + \xi_t,a_t + \epsilon_t\right),\\
	w_{t+1}&=\mathrm{Proj}_{\mathbb{S}_0\left(1\right)} \left(\tilde{w}_{t+1}\right).
\end{align*}
In the later proof, we use $\cF_t=\sigma\{(w_\tau, a_\tau)\big|\tau\leq t\}.$ as the sigma algebra generated by previous $t$ iterations and $V(\rho)=\frac{\pi^{\frac{p}{2}}}{\Gamma(\frac{p}{2}+1)}\rho^p$ 
as the volume of $p$-dimensional ball $\BB_{0}(\rho).$

\section{$d$-Dimensional Polar Coordinate and Some Important Lemmas}\label{section_polar_coordinate}
To calculate the expectation in our following analysis, we often need  the $d$-dimension polar coordinate system. Specifically, if we write a vector $\nu$ under Cartesian coordinate as $\nu=(\nu_1,\nu_2,...,\nu_d),$ then under the polar coordinate, $\nu$ can be written as  $\nu=(r,\theta_1,\theta_2,...,\theta_{d-1}),$ where
\begin{align*}
	&\nu_1=r\cos(\theta_1),\\
	&\nu_i=r\Pi_{j=1}^{i-1}\sin(\theta_j)\cos(\theta_i),~i=2,...,d-1,\\
	&\nu_d=r\Pi_{j=1}^{d-1}\sin(\theta_j),
\end{align*}
where $r\geq 0,$ $0\leq \theta_i\in[0,\pi],~i=1,2,...,d-2,$ 
$\theta_{d-1}\in[0,2\pi]$. 
%And $\theta_1=\angle(\nu,w^*)=\phi$, if we assume $w^*=\left(1,0,...,0\right)^\top$.

To use polar coordinate to calculate integral, we also need the following Jacobian Matrix.
$$\pdv{(\nu_1,\nu_2,\nu_3...,\nu_d)}{(r,\theta_1,\theta_2,...,\theta_{d-1})}
=r^{d-1}\sin^{d-2}\theta_1\sin^{d-3}\theta_2\cdots\sin\theta_{d-2}.$$

The following important equation is required.
\[I_n\triangleq\int_{0}^{\pi}\sin^n\left(x\right)dx =
\frac{\sqrt{\pi}\Gamma\left(\frac{1+n}{2}\right)}{\Gamma\left(1+\frac{n}{2}\right)}.
\]

Then we have the following useful lemma here.
\begin{lemma}\label{lem_gphi_gap}
	Let  $ f\left(\theta\right)$ be a positive bounded function defined on $[0,\pi],$ that is  there exits a constant $C\geq0$ such that $0\leq f\left(\theta\right)\leq C$, $\forall \theta \in [0, \pi]$. For any $\epsilon>0$ and positive integer $d$, define	
	\[A_d\left(f\right) \triangleq \frac{\Gamma\left(\frac{d}{2}+1\right)}{\pi^{d/2}} \int_{0}^{\pi}\cdots\int_{0}^{\pi}\int_{0}^{2\pi}f\left(\theta_1\right)\sin^{d-2}\theta_1\sin^{d-3}\theta_2\cdots\sin\theta_{d-2}d\theta_1 \cdots d\theta_{d-1},\]
	\[
	M_d\triangleq\int_{0}^{1}r^{d-1} dr, \quad
	L_d\left(\epsilon\right)\triangleq\int_{0}^{1-\epsilon}r^{d-1} dr,\quad
	H_d\left(\epsilon\right)\triangleq\int_{0}^{1+\epsilon}r^{d-1} dr.
	\]
	Then we have
	\begin{equation}
		A_d\left(f\right) L_d\left(\epsilon\right) + O\left(\epsilon d\right) > A_d\left(f\right) M_d > A_d\left(f\right) H_d\left(\epsilon\right) - O\left(\epsilon d\right).
		\label{lower}
	\end{equation}
\end{lemma}
\begin{proof}
	For simplicity, we only give the proof of the left side. The proof of the right side follows similar lines.
	
	We compute $A_d\left(f\right)$.
	\begin{align*}
		0\leq A_d\left(f\right)
		&=\frac{2\Gamma\left(\frac{d}{2}+1\right)}{\pi^{\frac{d}{2}-1}} I_{d-3}\cdots I_1\int_{0}^{\pi}f\left(\theta_1\right)\sin^{d-2}\theta_1d\theta_1\\
		&\leq \frac{2\Gamma\left(\frac{d}{2}+1\right)}{\pi^{\frac{d}{2}-1}} CI_{d-2}\cdots I_1\leq Cd.\\
	\end{align*}
	We give the lower bound on $L_d\left(\epsilon\right)-M_d$.
	\begin{align*}
		L_d\left(\epsilon\right)-M_d
		&=\int_{0}^{1-\epsilon}r^{d-1} dr - \int_{0}^{1}r^{d-1} dr=\frac{\left(1-\epsilon\right)^d-1}{d}\geq-\epsilon.
	\end{align*}
	Hence, we have
	\begin{equation}\label{upper}
		\left(L_{d}\left(\epsilon\right)-M_{d}\right)A_d\left(f\right)\geq-Cd\epsilon.
	\end{equation}
\end{proof}

\section{Proof for Phase I}\label{pf_phase I}

\subsection{Proof of Theorem \ref{thm_SC}}\label{pf_SC}
\begin{proof}
	We first derive the dissipativity w.r.t $a$ in region
	\begin{align*}
\cA_{C_2,C_3}=\{\left(w,a\right)~|-4\left(\one^\top{a}^*\right)^2&	\leq \one^\top {a}^*\one^\top a - \left(\one^\top {a}^*\right)^2
\leq  \frac{C_{2}}{p}\norm{{a}^*}_2^2,\\
& a^\top a^*\leq \frac{C_3}{p}\norm{a^*}_2^2\text{ or}
\norm{a-a^*/2}_2^2\geq\norm{a^*}_2^2,
\quad w\in\SSS_0\left(1\right)\}.
	\end{align*} 
	
	Assume $\left(w,a\right)\in\cA_{C_2,C_3}$, and $a^\top a^*\leq \frac{C_3}{p}\norm{a^*}_2^2$, we have $$\norm{a-a^*}_2^2\geq\left(1-\frac{2C_3}{p}\right)\norm{a^*}_2^2.$$
	Combining the above inequality with $\mathbb{E}g\left(\phi_\xi\right)\geq 1+\frac{C}{p}$ in Lemma \ref{lem_gphi_1}, we get
	\begin{align*}
		\langle-\mathbb{E}_{\xi,\epsilon}\nabla_aL\left(w+\xi,a+\epsilon\right),a^*-a\rangle
		&= \frac{1}{2\pi}\left(\one^\top a- \one^\top a^*\right)^2
		+\frac{1}{2\pi}\left(\left(\pi-1\right)a-\left(\mathbb{E}_\xi g\left(\phi_\xi\right)-1\right)a^*\right)^\top \left(a-a^*\right)\\
		&=\frac{1}{2\pi}\left(\one^\top a- \one^\top a^*\right)^2
		+\frac{1}{2\pi}\left(\pi-\mathbb{E}_\xi g\left(\phi_\xi\right)\right)a^\top \left(a-a^*\right)+\frac{\mathbb{E}_\xi g\left(\phi_\xi\right)-1}{2\pi}\norm{a-a^*}_2\\
		&\geq -\frac{1}{2\pi}\left(\pi-\mathbb{E}_\xi g\left(\phi_\xi\right)\right)a^\top a^*+\frac{\mathbb{E}_\xi g\left(\phi_\xi\right)-1}{2\pi}\norm{a-a^*}_2^2\\
		&= -\frac{1}{2\pi}\left(\pi-\mathbb{E}_\xi g\left(\phi_\xi\right)\right)a^\top a^*+\frac{\mathbb{E}_\xi g\left(\phi_\xi\right)-1}{4\pi}\norm{a-a^*}_2^2+\frac{\mathbb{E}_\xi g\left(\phi_\xi\right)-1}{4\pi}\norm{a-a^*}_2^2\\	
		&\geq-\frac{C_3}{2p}\norm{a^*}_2^2 +\frac{C}{4\pi p}\left(1-\frac{2C_3}{p}\right)\norm{a^*}_2^2
		+   \frac{C}{4\pi p}\norm{a-a^*}_2^2\geq  \frac{C}{4\pi p}\norm{a-a^*}_2^2\\
	\end{align*}
	for some constant $C_3\leq\frac{C}{4\pi}$.
	
	Moreover, if $\left(w,a\right)\in\cA_{C_2,C_3}$ and $\norm{a-a^*/2}_2^2\geq\norm{a^*}_2^2,$ we have $$a^\top\left(a-a^*\right)>0.$$ Following the similar lines  above, we have the same results. Thus, the dissipativity w.r.t $a$ holds in region $\cA_{C_2,C_3}$
	
	Next, we derive the dissipativity w.r.t $w$ in region
	$$	\mathcal{K}_{C_4,m,M}=\{\left(w,a\right)~|~a^\top {a}^*\in [m,M],\quad w^\top w^*\geq C_4,\quad w\in\SSS_0\left(1\right)\}.$$
	
	Assume $\left(w,a\right)\in\mathcal{K}_{C_4,m,M}$ for some constant $C_4\in(-1,1]$ and $0<m<M$. We could write $w$ as $w=\sum_{i=1}^{p} c_iv_i$, where $\{v_i\}_{i=1}^{p}$ is an orthonormal basis for $\mathbb{R}^p$, $w^*=v_1$, $\norm{w}_2=1$ and $c_1\geq C_4$. Without loss of generality, we assume $w^*=\left(1,0,...,0\right)^\top.$ We have the following equation.
	\begin{align*}
		(I-ww^\top)	(w^*-w)&=w^*-w-ww^\top w^*+w w^\top w=w^*-ww^\top w^*\\&=(1,0,...,0)^\top-(c_1^2,c_1c_2,...,c_1c_p)^\top=(1-c_1^2,-c_1c_2,...,-c_1c_p)^\top,
	\end{align*}
	The norm of this vector is 
	\begin{align*}
		\norm{(I-ww^\top)	(w^*-w)}_2&=\sqrt{(1-c_1^2)^2+c_1^2(c_2^2+...+c_p^2)}\\
		&=\sqrt{(1-c_1^2)^2+c_1^2(1-c_1^2)}=\sqrt{1-c_1^2}
	\end{align*}

	By $\mathbb{E}_\xi\left(\phi_\xi\right)\leq\frac{3\pi}{4}$ in Lemma \ref{lem_gphi_1}, we have
	\begin{align*}
		&\langle-\mathbb{E}_{\xi,\epsilon}\left(I-ww^\top\right)\nabla_wL\left(w+\xi,a+\epsilon\right),w^*-w\rangle
		=\frac{a^\top a^* \left(\pi -\mathbb{E}_\xi\phi_\xi\right)}{2\pi}\left(1-c_1^2\right)\\
		&+\EE_{\xi,\epsilon}\left(w^*-w^\top w^*w\right)^T\Big(\frac{\norm{a+\epsilon}_2^2}{2}+\frac{\sum_{i\neq j}\left(a_i+\epsilon_i\right)\left(a_j+\epsilon_j\right)}{2\pi}
		-\frac{\left(a+\epsilon\right)^\top a^* \sin\phi_\xi}{2\pi}\frac{1}{\norm{w+\xi}_2}\\
		&\hspace{+2in}-\frac{\sum_{i\neq j}\left(a_i+\epsilon_i\right) a_j^*}{2\pi}\frac{1}{\norm{w+\xi}_2}\Big)\xi\\
		&=\frac{a^\top a^* \left(\pi -\mathbb{E}_\xi\phi_\xi\right)}{2\pi}\left(1-c_1^2\right)
		+ \EE_{\xi}\left(w^*-w^\top w^*w\right)^T\left(
		-\frac{a^\top a^* \sin\phi_\xi}{2\pi}\frac{1}{\norm{w+\xi}_2}
		-\frac{\sum_{i\neq j}{a_i a_j^*}}{2\pi}\frac{1}{\norm{w+\xi}_2}\right)\xi
		%&=\frac{a^\top a^* \left(\pi -\mathbb{E}_\xi\phi_\xi\right)}{2\pi}\left(1-c_1^2\right).
	\end{align*}
	
	We next show that
	\begin{align}\label{eq1}
		\EE_{\xi}
		\frac{ \left(w^*-w^\top w^*w\right)^T\xi}{\norm{w+\xi}_2}=0
	\end{align}
	and
	\begin{align}\label{ineq1}
		\EE_{\xi}
		\frac{ \sin\phi_\xi\left(w^*-w^\top w^*w\right)^T\xi}{\norm{w+\xi}_2}\leq C\frac{\sqrt{1-c_1^2}}{\rho_w}
	\end{align}
	for some constant $C$.
	
	For \eqref{eq1}, recall that $V(\rho_w)$ is the volume of $\BB_{0}(\rho_w).$ Then we have
	\begin{align*}
		\EE_{\xi}\frac{\left(w^*-w^\top w^*w\right)^\top\xi}{\norm{w+\xi}_2}
		&=\int_{\BB_0\left(\rho_w\right)}\frac{1}{V\left(\rho_w\right)}\frac{\left(w^*-w^\top w^*w\right)^\top x}{\norm{w+x}_2} dx\\
		%&=\int_{\BB_{w}\left(\rho_w\right)}\frac{1}{V\left(\rho_w\right)}\frac{\left(w^*-w^\top w^*w\right)^\top x}{\norm{w+x}_2}dx\\
		&=\int_{\BB_{w}\left(\rho_w\right)}\frac{1}{V\left(\rho_w\right)}\frac{\left(w^*-w^\top w^*w\right)^\top x-\left(w^*-w^\top w^*w\right)^\top w}{\norm{x}_2}dx\\
		&=\int_{\BB_{w}\left(\rho_w\right)}\frac{1}{V\left(\rho_w\right)}\frac{\left(w^*-w^\top w^*w\right)^\top x}{\norm{x}_2}dx\\
		&=\int_{\BB _{w}\left(\rho_w\right),\left(w^*-w^\top w^*w\right)^\top x>0}\frac{1}{V\left(\rho_w\right)}\frac{\left(w^*-w^\top w^*w\right)^\top x}{\norm{x}_2}\\
		&\quad+\int_{\BB_{w}\left(\rho_w\right),\left(w^*-w^\top w^*w\right)^\top x<0}\frac{1}{V\left(\rho_w\right)}\frac{\left(w^*-w^\top w^*w\right)^\top x}{\norm{x}_2}dx\\
	\end{align*}
	For any $x$ such that $\left(w^*-w^\top w^*w\right)^\top x>0,$ its symmetric point with respect to vector $w$ is $\tilde{x}=2w^\top xw-x.$ We further have 
	$$\left(w^*-w^\top w^*w\right)^\top \tilde{x}
	=\left(w^*-w^\top w^*w\right)^\top (2w^\top xw-x)
	=-\left(w^*-w^\top w^*w\right)^\top x <0.$$
	By this symmetric property with respect to vector $w$, we know
	\begin{align*}
		\EE_{\xi}\frac{\left(w^*-w^\top w^*w\right)^\top\xi}{\norm{w+\xi}_2}&=\int_{\BB_0\left(\rho_w\right)}\frac{1}{V\left(\rho_w\right)}\frac{\left(w^*-w^\top w^*w\right)^\top x}{\norm{w+x}_2} dx\\
		&=\int_{\BB _{w}\left(\rho_w\right),\left(w^*-w^\top w^*w\right)^\top x>0}\frac{1}{V\left(\rho_w\right)}\frac{\left(w^*-w^\top w^*w\right)^\top x}{\norm{x}_2}\\&\quad+\int_{\BB_{w}\left(\rho_w\right),\left(w^*-w^\top w^*w\right)^\top x<0}\frac{1}{V\left(\rho_w\right)}\frac{\left(w^*-w^\top w^*w\right)^\top x}{\norm{x}_2}dx\\
		&=0.
	\end{align*}

	Now we prove \eqref{ineq1}. Denote that $\phi_x=\angle(x,w^*)$. When $\rho_w>1,$ we have
	
	\begin{align*}
		&\quad\EE_{\xi}\frac{\sin\phi_\xi\left(w^*-w^\top w^*w\right)^\top\xi}{\norm{w+\xi}_2}\\
		%&=\int_{\BB_0\left(\rho_w\right)}\frac{\sin\phi_{x+w}}{V\left(\rho_w\right)}\frac{\left(w^*-w^\top w^*w\right)^\top x}{\norm{w+x}_2} dx\\
		%&=\int_{\BB_{w}\left(\rho_w\right)}\frac{1}{V\left(\rho_w\right)}\frac{\left(w^*-w^\top w^*w\right)^\top x}{\norm{w+x}_2}dx\\
		&=\int_{\BB_{w}\left(\rho_w\right)}\frac{\sin\phi_{x}}{V\left(\rho_w\right)}\frac{\left(w^*-w^\top w^*w\right)^\top x}{\norm{x}_2}dx\\
		&=\int_{\BB _{w}\left(\rho_w\right),\left(w^*-w^\top w^*w\right)^\top x>0}\frac{\sin\phi_{x}}{V\left(\rho_w\right)}\frac{\left(w^*-w^\top w^*w\right)^\top x}{\norm{x}_2}
		\\&~~~+ \int_{\BB_{w}\left(\rho_w\right),\left(w^*-w^\top w^*w\right)^\top x<0}\frac{\sin\phi_{x}}{V\left(\rho_w\right)}\frac{\left(w^*-w^\top w^*w\right)^\top x}{\norm{x}_2}dx\\
		&\leq \int_{\BB_{0}\left(\rho_w+1\right),\left(w^*-w^\top w^*w\right)^\top x>0}\frac{\sin\phi_x}{V\left(\rho_w\right)}\frac{\left(w^*-w^\top w^*w\right)^\top x}{\norm{x}_2}dx\\
		&~~~+ \int_{\BB_{0}\left(\rho_w-1\right),\left(w^*-w^\top w^*w\right)^\top x<0}\frac{\sin\phi_x}{V\left(\rho_w\right)}\frac{\left(w^*-w^\top w^*w\right)^\top x}{\norm{x}_2}dx\\
		&\leq \int_{\BB_{0}\left(\rho_w+1\right)\backslash\BB_{0}\left(\rho_w-1\right),\left(w^*-w^\top w^*w\right)^\top x>0}\frac{1}{V\left(\rho_w\right)}\frac{\left(w^*-w^\top w^*w\right)^\top x}{\norm{x}_2}dx\\
		&=\frac{\Gamma\left(\frac{p}{2}+1\right)}{\rho_w^p\pi^{p/2}}\left(\int_{\rho_w-1}^{\rho_w+1} r^{p-1} dr  \int_{0}^{\pi/2}\sqrt{1-c_1^2}\cos(\theta_1)\sin^{p-2}\left(\theta_1\right)d\theta_1\right)\\
		&~\cdot\left( \int_{0}^{\pi}\cdots\int_{0}^{\pi}\int_{0}^{2\pi}\sin^{p-3}\theta_2\cdots\sin\theta_{p-2}d\theta_1 \cdots d\theta_{p-1}\right)
		\qquad \text{(Convert to polar coordinate)}\\
		&=\frac{\Gamma\left(\frac{p}{2}+1\right)}{\rho_w^p\pi^{p/2}}\frac{\left(\rho_w+1\right)^{p}-\left(\rho_w-1\right)^{p}}{p}\frac{\sqrt{1-c_1^2}}{p-1}I_{p-3}\cdots I_1\\
		&=\frac{\Gamma\left(\frac{p}{2}+1\right)}{\rho_w^p\pi^{p/2}}\frac{\left(\rho_w+1\right)^{p}-\left(\rho_w-1\right)^{p}}{p}\frac{\sqrt{1-c_1^2}}{p-1}\frac{\pi^{\frac{p-3}{2}}}{\Gamma\left(\frac{p-1}{2}\right)}=C\frac{\sqrt{1-c_1^2}}{\rho_w}
	\end{align*}
	
	%Moreover, the term $\sum_{i\neq j}{a_i a_j^*}$ can be bounded as follows:
	%\begin{align*}
	%	|\sum_{i\neq j}{a_i a_j^*}|\leq\sum_{i\neq j}{|a_i|| a_j^*|}\leq\sum_{i\neq j}{\frac{a_i^2+ \left(a_j^*\right)^2}{2}}\leq \frac{k-1}{2}\norm{a^*}_2^2+\frac{k-1}{2}\norm{a}_2^2=Ck.
	%\end{align*}
	
	Thus when $\rho_w\geq C\frac{M_a}{\gamma}$, combining \eqref{eq1} and \eqref{ineq1}, we have for any $\gamma>0$
	\begin{align*}
		-\frac{a^\top a^*}{2\pi}\EE_{\xi}
		\frac{ \sin\phi_\xi\left(w^*-w^\top w^*w\right)^T\xi}{\norm{w+\xi}_2}
		-\frac{\sum_{i\neq j}{a_i a_j^*}}{2\pi}\EE_{\xi}\frac{\left(w^*-w^\top w^*w\right)^T\xi}{\norm{w+\xi}_2}\geq -\gamma.
	\end{align*}
	Then we have 
	\begin{align*}
		&\langle-\mathbb{E}_{\xi,\epsilon}\left(I-ww^\top\right)\nabla_wL\left(w+\xi,a+\epsilon\right),w^*-w\rangle\geq \frac{m\left(1+C_4\right)}{16}\norm{w-w^*}_2^2-\gamma.
	\end{align*}
\end{proof}

\subsection{Some Important Lemmas}
These lemmas give proper bounds that we will use in our later proof.
\begin{lemma}[Bound on $\left(\one^\top a^*\right)^2 - \one^\top a^*\one^\top a_t$]\label{lem_sum_a}
	Suppose $-A\leq\one^\top a^*\one^\top a_0 - \left(\one^\top a^*\right)^2\leq \frac{C_1}{p}\norm{a^*}_2^2$ for some constants $C_1,A\geq0$.
	For any $\delta\in\left(0,1\right)$, if we choose $\eta\leq C \left(k^4p^6\max\{1,\log\frac{1}{\delta}\}\right)^{-1},$ 
	then with at least probability $1-\delta$, we have
	$$ -A-2\left(\one^\top a^*\right)^2
	\leq \one^\top a^*\one^\top a_{t} - \left(\one^\top a^*\right)^2
	\leq \frac{C_2}{p}\norm{a^*}_2^2,$$
	for $\forall t\leq T= \tilde{O}\left(\frac{1}{\eta^2}\right)$ and some constant $C_2>C_1.$
\end{lemma}
\begin{proof}
	We only give the prove for the right side, and the left side follows similar lines.
	
	We start with
	\begin{align*}
		\mathbb{E}[\one^\top a^*\one^\top a_{t+1}|\mathcal{F}_t]
		&=\left(1-\frac{\eta\left(k+\pi-1\right)}{2\pi}\right)\one^\top a^*\one^\top a_t
		+ \frac{\eta\left(k+\mathbb{E}g\left(\phi\right)-1\right)}{2\pi}\left(\one^\top a^*\right)^2\\
		&\leq \left(1-\frac{\eta\left(k+\pi-1\right)}{2\pi}\right)\one^\top a^*\one^\top a_t
		+ \frac{\eta\left(k+\pi-1\right)}{2\pi}\left(\one^\top a^*\right)^2.\\
	\end{align*}
	Denote $G_t= \left(1-\frac{\eta\left(k+\pi-1\right)}{2\pi}\right)^{-t}\left(\one^\top a^*\one^\top a_t-\left(\one^\top a^*\right)^2\right)$. Thus, we have
	\begin{align*}
		\mathbb{E}[G_{t+1} |\mathcal{F}_t] \leq G_t.\\
	\end{align*}
	Denote $\mathscr{E}_t=\{\forall\tau\leq t,\one^\top a^*\one^\top a_\tau-\left(\one^\top a^*\right)^2 \leq \frac{C_2}{p}\norm{a^*}_2^2\}\subset\mathcal{F}_t$. We have 
	$$\EE[G_{t+1}\mathds{1}_{\mathscr{E}_t} |\mathcal{F}_t] \leq G_t\mathds{1}_{\mathscr{E}_t}
	\leq G_t\mathds{1}_{\mathscr{E}_{t-1}}.$$
	Thus, $G_t\mathds{1}_{\mathscr{E}_{t-1}}$ is a supermartingale with initial value $G_0$.
	
	We have the following bound.
	\begin{align*}
		d_t 
		&\triangleq
		|G_t\mathds{1}_{\mathscr{E}_{t-1}}-\mathbb{E}[G_t\mathds{1}_{\mathscr{E}_{t-1}}|\mathcal{F}_{t-1}]|\\
		&=\left(1-\frac{\eta\left(k+\pi-1\right)}{2\pi}\right)^{-t} \eta \left|-\frac{k+\pi-1}{2\pi}\one^\top a^*\one^\top\epsilon_{t-1} +\frac{\mathbb{E}g\left(\phi_{\xi_{t-1}}\right) 
			-g\left(\phi_{\xi_{t-1}}\right)}{2\pi}\left(\one^\top a^*\right)^2\right|\\
		&\leq \left(1-\frac{\eta\left(k+\pi-1\right)}{2\pi}\right)^{-t}\eta\left(\frac{(k+\pi-1)k\rho_a}{2\pi}|\one^\top a^*|+\frac{\left(\one^\top a^*\right)^2}{2}\right)\\
		&=\left(1-\lambda\right)^{-t}M,
	\end{align*}
	where $\lambda=\frac{\eta\left(k+\pi-1\right)}{2\pi}$ and $M=\eta\left(\frac{(k+\pi-1)k\rho_a}{2\pi}|\one^\top a^*|+\frac{\left(\one^\top a^*\right)^2}{2}\right)$. Denote $r_t=\sqrt{\sum_{i=1}^{t}d_i^2}$. Then Azuma's Inequality can be applied, and we have
	$$\mathbb{P}\left(G_t\mathds{1}_{\mathscr{E}_{t-1}}-G_0\geq \tilde{O}\left(1\right)r_t\log^{1/2}\left(\frac{1}{\eta\delta}\right)\right)
	=\exp\left(-\frac{\tilde{O}\left(1\right)r_t^2\log\left(\frac{1}{\eta\delta}\right)}{2\sum_{i=1}^{t}d_i^2}\right)=\exp\left(-\frac{\tilde{O}\left(1\right)r_t^2\log\left(\frac{1}{(\eta\delta)^2}\right)}{2\sum_{i=1}^{t}d_i^2}\right)
	=\tilde{O}\left(\eta^2\delta\right).$$
	Therefore, with at least probability $1-\tilde{O}\left(\eta^2\delta\right),$ we have
	$$G_t\mathds{1}_{\mathscr{E}_{t-1}}\leq G_0 + \tilde{O}\left(1\right)r_t\log^{1/2}\left(\frac{1}{\eta\delta}\right).$$
	
	We next prove that conditioning on $\mathscr{E}_{t-1}$, we have $\mathscr{E}_t$ with at least probability $1-\tilde{O}\left(\eta^2\delta\right)$. Thus, $\mathscr{E}_T$ holds with at least probability $1-\delta$, when $T=\tilde{O}\left(\frac{1}{\eta^2}\right)$.
	
	From $G_t\mathds{1}_{\mathscr{E}_{t-1}}\leq G_0 + \tilde{O}\left(1\right)r_t\log^{1/2}\left(\frac{1}{\eta\delta}\right)$, we know
	\begin{align*}
		\one^\top a^*\one^\top a_t-\left(\one^\top a^*\right)^2
		&\leq \left(1-\lambda\right)^{t}\left(\one^\top a^*\one^\top a_0-\left(\one^\top a^*\right)^2+ \tilde{O}\left(1\right)r_t\log^{1/2}\left(\frac{1}{\eta\delta}\right)\right)\\
		&\leq \one^\top a^*\one^\top a_0-\left(\one^\top a^*\right)^2 + \tilde{O}\left(1\right)\left(1-\lambda\right)^{t}r_t\log^{1/2}\left(\frac{1}{\eta\delta}\right).
	\end{align*}
	
	We have
	\begin{equation}
		\begin{aligned}
			\left(1-\lambda\right)^{t}r_t &=\left(1-\lambda\right)^{t}M\sqrt{\sum_{i=1}^{t}\left(1-\lambda\right)^{-2i}}
			=M\sqrt{\sum_{i=0}^{t-1}\left(1-\lambda\right)^{2i}}\\
			&\leq M\sqrt{\frac{1}{1-\left(1-\lambda\right)^2}} 
			\leq \frac{M}{\sqrt{\lambda}}\\
			&=\left(\frac{k+\pi-1}{2\pi}k\rho_a|\one^\top a^*|+\frac{\left(\one^\top a^*\right)^2}{2}\right)\sqrt{\frac{2\pi\eta}{k+\pi-1}}.
		\end{aligned}
		\label{rt}
	\end{equation}

	By carefully choosing $\eta_{\max}=\tilde{O}\left(\frac{1}{k^4p^6}\right)$ and let $\eta=\frac{\eta_{\max}}{\max\{1,\log\frac{1}{\delta}\}}$, we have
	\begin{align*}
		\one^\top a^*\one^\top a_t-\left(\one^\top a^*\right)^2
		&=
		\frac{C_1}{p}\norm{a^*}_2^2 
		+\tilde{O}\left(1\right)\log^{1/2}\left(\frac{1}{\eta\delta}\right)\left(\frac{k+\pi-1}{2\pi}k\rho_a|\one^\top a^*|+\frac{\left(\one^\top a^*\right)^2}{2}\right)\sqrt{\frac{2\pi\eta}{k+\pi-1}}\\
		&\leq \frac{C_2}{p}\norm{a^*}_2^2.
	\end{align*}
\end{proof}

\begin{lemma}[Bound on $a_t^\top a^*$]\label{lem_inner_a} 
	Suppose $\frac{C_1}{p}\norm{a^*}_2^2\leq a_0^\top a^*\leq M_a$, $\mathbb{E}_{\xi_{t-1}}g\left(\phi_t\right)\geq 1+\frac{C_2}{p}$ and
	$$ -A-2\left(\one^\top a^*\right)^2
	\leq \one^\top a^*\one^\top a_{t} - \left(\one^\top a^*\right)^2
	\leq  \frac{C_3}{p}\norm{a^*}_2^2$$
	holds for $\forall t\leq T= \tilde{O}\left(\frac{1}{\eta^2}\right)$ and some positive constants $M_a,C_1,C_2>C_3$. 
	If we take $\eta\leq C \left(k^4p^6\max\{1,\log\frac{1}{\delta}\}\right)^{-1},$ then with at least probability  $1-\delta$, we have
	$$
	\frac{C_4}{p}\norm{a^*}_2^2 \leq
	a_{t}^\top a^*
	\leq A + M_a + \left(1+\frac{C_5}{p}\right)\norm{a^*}_2^2+2\left(\one^\top a^*\right)^2
	$$
	for $\forall t\leq T=\tilde{O}\left(\frac{1}{\eta^2}\right)$ and some positive constants $C_4,C_5$.
	%\footnote[1]{The right side of the bound is not the tightest. However, since we do %not focus on the dependence on $\norm{a^*}_2$, we give this bound to simplify %notations.}
\end{lemma}
\begin{proof}
	We only give the proof for the left side, the right side follows similar lines.
	\begin{align*}
		\mathbb{E}[a_{t+1}^\top a^*|\mathcal{F}_t]
		&= \left(1-\frac{\eta\left(\pi-1\right)}{2\pi}\right)a_t^\top a^* + \frac{\eta\left(\mathbb{E}g\left(\phi_t\right)-1\right)}{2\pi}\norm{a^*}_2^2
		+\frac{\eta}{2\pi}\left(\left(\one^\top a^*\right)^2 -\one^\top a^*\one^\top a_t\right)\\
		&\geq \left(1-\frac{\eta\left(\pi-1\right)}{2\pi}\right)a_t^\top a^* + 
		\eta \frac{C_2-C_3}{2\pi p}\norm{a^*}_2^2.
	\end{align*}
	
	Denote $C = \frac{C_2-C_3}{2\pi p}\norm{a^*}_2^2$ and $G_t = \left(1-\frac{\eta\left(\pi-1\right)}{2\pi}\right)^{-t}\left(a_t^\top a^*-\frac{2\pi}{\pi-1}C\right)$. The above inequality changes to
	$$\mathbb{E}[G_{t+1}|\mathcal{F}_t]\geq G_t.$$
	
	%	Thus,$$\mathbb{E}G_{t+1}\geq\mathbb{E}G_{t}\geq\mathbb{E}G_0\geq-\frac{2\pi}{\pi-1}C.$$
	%	For $t\geq T_0 = \frac{4\pi}{\eta\left(\pi-1\right)}$,
	%	$$\mathbb{E}a_t^\topa^*\geq\left(1-\left(1-\frac{\eta\left(\pi-1\right)}{2\pi}\right)^{t}\right)\frac{2\pi}{\pi-1}C
	%	\geq\frac{\pi}{\pi-1}C$$
	
	Denote $\mathscr{E}_t=\{\forall \tau\leq t, a_\tau^\top a^*\geq\frac{C_4}{p}\norm{a^*}_2^2\}$ for some constant $C_4=\min\{\frac{C_1}{2},\frac{C_2-C_3}{2(\pi-1)}\}$. Then, for all $t$, $G_{t+1}\mathds{1}_{\mathscr{E}_t}$ satisfies
	$$\mathbb{E}[G_{t+1}\mathds{1}_{\mathscr{E}_t}|\mathcal{F}_t]\geq G_t\mathds{1}_{\mathscr{E}_t}
	\geq G_t\mathds{1}_{\mathscr{E}_{t-1}}.$$
	Thus, $\{G_{t+1}\mathds{1}_{\mathscr{E}_t}\}$ is submartingale.
	
	%To simplify the notation, we assume $T_0=0$ in the following prove.
	We have the following bound of the difference between $G_{t+1}\mathds{1}_{\mathscr{E}_t}$ and $\mathbb{E}[G_{t+1}\mathds{1}_{\mathscr{E}_t}]$.
	\begin{align*}
		d_{t+1}
		&\triangleq
		|G_{t+1}\mathds{1}_{\mathscr{E}_t}-\mathbb{E}[G_{t+1}\mathds{1}_{\mathscr{E}_t}|\mathcal{F}_t]|\\
		&=\eta\left(1-\frac{\eta\left(\pi-1\right)}{2\pi}\right)^{-t-1}
		\left|-\frac{\pi-1}{2\pi}\epsilon_t^\top a^*+\frac{\mathbb{E}g\left(\phi_{\xi_t}\right)-g\left(\phi_{\xi_t}\right)}{2\pi}\norm{a^*}_2^2-\frac{\one^\top a^*\one^\top\epsilon_t}{2\pi}\right|\\
		&\leq\eta\left(1-\frac{\eta\left(\pi-1\right)}{2\pi}\right)^{-t-1}\left(\frac{\pi-1}{2\pi}\rho_a\norm{a^*}_2 +\frac{\norm{a^*}_2^2}{2\pi}+\frac{|\one^\top a^*|k\rho_a}{2\pi}\right)\\
		&=\left(1-\lambda\right)^{-t-1}M,
	\end{align*}
	where $\lambda=\frac{\eta\left(\pi-1\right)}{2\pi}$ and $M=\eta\left(\frac{\pi-1}{2\pi}\rho_a\norm{a^*}_2 +\frac{\norm{a^*}_2^2}{2\pi}+\frac{|\one^\top a^*|k\rho_a}{2\pi}\right)$.
	
	Denote $r_t=\sqrt{\sum_{i=0}^{t}d_i^2}$. By Azuma's Inequality again, we have 
	$$\mathbb{P}\left(G_t\mathds{1}_{\mathscr{E}_{t-1}}-G_0\leq-\tilde{O}\left(1\right)r_t\log^{\frac{1}{2}}\left(\frac{1}{\eta\delta}\right)\right)\leq\exp\left(-\frac{\tilde{O}\left(1\right)r_t^2\log\left(\frac{1}{\eta\delta}\right)}{2\sum_{i=0}^{t}d_i^2}\right)=\tilde{O}\left(\eta^2\delta\right).$$
	
	Therefore, with at least probability $1-\tilde{O}\left(\eta^2\delta\right)$, we have 
	$$G_t\mathds{1}_{\mathscr{E}_{t-1}}
	\geq G_0 -\tilde{O}\left(1\right)r_t\log^{\frac{1}{2}}\left(\frac{1}{\eta\delta}\right).$$
	This means that when $\mathscr{E}_{t-1}$ holds, with at least probability $1-\tilde{O}\left(\eta^2\delta\right)$,
	\begin{align*}
		a_t^\top a^*
		&\geq \left(1-\lambda\right)^{t}\left(a_0^\top a^*-\frac{2\pi}{\pi-1}C
		-\tilde{O}\left(1\right)r_t\log^{\frac{1}{2}}\left(\frac{1}{\eta\delta}\right)\right)
		+\frac{2\pi}{\pi-1}C\geq C_6
		-\tilde{O}\left(1\right)\left(1-\lambda\right)^{t}r_t\log^{\frac{1}{2}}\left(\frac{1}{\eta\delta}\right),\\
	\end{align*}	
	where $C_6=\min\{\frac{C_1}{p}\norm{a^*}^2_2,\frac{2\pi}{\pi-1}C\}$.
	Following similar lines to \eqref{rt}, we have
	\begin{align*}
		\left(1-\lambda\right)^{t}r_t\leq \frac{M}{\sqrt{\lambda}}
		=\left(\frac{\pi-1}{2\pi}\rho_a\norm{a^*}_2 +\frac{\norm{a^*}_2^2}{2\pi}+\frac{|\one^\top a^*|k\rho_a}{2\pi}\right)
		\sqrt{\frac{2\pi\eta}{\pi-1}}.
	\end{align*}	
	With a proper step size $\eta\leq \left(k^4p^6\max\{1,\log\frac{1}{\delta}\}\right)^{-1},$ we then have 
	\begin{align*}
		a_t^\top a^*
		&\geq C_6
		-\tilde{O}\left(1\right)\log^{\frac{1}{2}}\left(\frac{1}{\eta\delta}\right)\left(\frac{\pi-1}{2\pi}\rho_a\norm{a^*}_2 +\frac{\norm{a^*}_2^2}{2\pi}+\frac{|\one^\top a^*|k\rho_a}{2\pi}\right)
		\sqrt{\frac{2\pi\eta}{\pi-1}}\geq \frac{C_6}{2}=\frac{C_4}{p}\norm{a^*}_2^2.
	\end{align*}
\end{proof}

\subsection{Detailed Proof of Theorem \ref{stage1_w}}\label{pf_stage1_w}
\subsubsection{Proof of Lemma \ref{lem_initial}}
\begin{proof}
	For any $a\in\BB_0\left(\frac{|\one^\top a^*|}{\sqrt{k}}\right),$ we have
	\begin{align*}
		&~\one^\top a\leq\norm{a}_1\leq\sqrt{k}\norm{a}_2\leq|\one^\top a^*|,\\
		&\one^\top a\geq-\norm{a}_1\geq-\sqrt{k}\norm{a}_2\geq-|\one^\top a^*|.
	\end{align*}
	Thus, we have 
	\begin{align*}
		|\one^\top a|\leq|\one^\top a^*|,
	\end{align*}
	which is equivalent to the following inequality.
	$$-2\left(\one^\top{a}^*\right)^2\leq \one^\top {a}^*\one^\top a - \left(\one^\top {a}^*\right)^2\leq 0.$$
	By this inequality, we prove that $a\in \cA_{C_3}.$
\end{proof}

\subsubsection{Proof of Lemma \ref{lem_gphi_1}}\label{pf_gphi}
\begin{proof}
	For calculation simplicity, we rescale $w$ and $\xi$ by $\rho_w.$ Specifically,
	For any $w\in \SSS_{0}(1),$	define $\nu=\rho_w^{-1}w$ and $r_\nu\triangleq\norm{\nu}_2=\norm{w}_2/\rho_w=1/\rho_w,$ where $\rho_w=\Omega\left(p^2\right)$ and $r_\nu = O\left(p^{-2}\right)$. Moreover, let $\zeta\triangleq\xi/\rho_w\sim \text{unif}(\BB_0(1)).$ Then we have $\angle(v+\zeta,w^*)=\angle(w+\zeta,w^*).$\\
	Without loss of generality, we assume $w^*=\left(1,0,...,0\right)^\top.$  To calculate the expectation, we need to rewrite $\nu$ in the $p$-dimension polar coordinate system as discussed in Section \ref{section_polar_coordinate}. Specifically, $\nu$ can be written as  $\nu=(r,\theta_1,\theta_2,...,\theta_{p-1}),$ where $r\geq 0,$ $\theta_i\in[0,\pi],~i=1,2,...,p-2,$ $\theta_{p-1}\in[0,2\pi]$ and $\theta_1=\arccos(\nu_1/\norm{\nu})=\angle(w,w^*).$ Moreover, under the polar coordinate,  $\nu+\zeta$ is expressed as $(r^\zeta,\theta_1^\zeta,\theta_2^\zeta,...,\theta_{p-1}^\zeta).$ We then have  
	\[ \theta_1^\zeta
	=\arccos\frac{\nu_1+\zeta_1}{\norm{\nu+\zeta}_2}
	=\angle(v+\zeta,w^*)
	=\phi_\xi,\]
	where $\zeta = \left(\zeta_1, \zeta_2,\cdots, \zeta_p\right)$.\\
	Therefore, for sufficiently large $\rho_w$ we have
	\begin{align*}
		&\mathbb{E}_\xi \left(\pi-\phi_\xi\right)\cos\phi_\xi+\sin\phi_\xi=\mathbb{E}_\zeta \left(\pi-\theta_1^\zeta\right)\cos\theta_1^\zeta+\sin\theta_1^\zeta\\
		%	&=\int_{\mathbb{B}_\epsilon\left(1\right)}[\left(\pi-\arccos\frac{x_1}{\norm{x}_2}\right)\frac{x_1}{\norm{x}_2}
		%	+\sqrt{1-\frac{x_1}{\norm{x}_2}}] \frac{\Gamma\left(\frac{d}{2}+1\right)}{\pi^{d/2}}dx\\
		&=\int_{\mathbb{B}_0\left(1\right)}
		\left[\left(\pi-\arccos\frac{\nu_1+x_1}{\norm{\nu+x}_2}\right)\frac{\nu_1+x_1}{\norm{\nu+x}_2}
		+\sin\arccos\frac{\nu_1+x_1}{\norm{\nu+x}_2}\right] \frac{1}{V(1)}dx\\
		&=\int_{\mathbb{B}_\nu\left(1\right)}
		\left[\left(\pi-\arccos\frac{x_1}{\norm{x}_2}\right)\frac{x_1}{\norm{x}_2}
		+	\sin\arccos\frac{x_1}{\norm{x}_2}\right] \frac{\Gamma\left(\frac{p}{2}+1\right)}{\pi^{p/2}}dx\\
		&\geq\int_{\mathbb{B}_0\left(1-r_\nu\right)}
		\left[\left(\pi-\arccos\frac{x_1}{\norm{x}_2}\right)\frac{x_1}{\norm{x}_2}
		+	\sin\arccos\frac{x_1}{\norm{x}_2}\right] \frac{\Gamma\left(\frac{p}{2}+1\right)}{\pi^{p/2}}dx\\
		&=\int_{0}^{1-r_\nu} r^{p-1} dr \frac{\Gamma\left(\frac{p}{2}+1\right)}{\pi^{p/2}} \int_{0}^{\pi}\left(\left(\pi-\theta_1\right)\cos\theta_1+\sin\theta_1\right)\sin^{p-2}\theta_1d\theta_1\\
		&\quad \int_{0}^{\pi}\cdots\int_{0}^{\pi}\int_{0}^{2\pi}\sin^{p-3}\theta_2\cdots\sin\theta_{p-2}d\theta_2 \cdots d\theta_{p-1}\\
	\end{align*}
	We then apply Lemma \ref{lem_gphi_gap} by taking $f(\theta)=\left(\pi-\theta\right)\cos\theta+\sin\theta$ and get the following result.
	\begin{align*}
		&\mathbb{E}_\xi \left(\pi-\phi_\xi\right)\cos\phi_\xi+\sin\phi_\xi\\
		&= L_{p}\left(r_\nu\right)A_p\left(f\right)= M_{p}A_p\left(f\right) + \left(L_{p}\left(r_\nu\right)-M_{p}\right)A_p\left(f\right)\\
		&=\dfrac{\Gamma\left(\dfrac{p}{2}\right)\Gamma\left(\dfrac{p+2}{2}\right)} {\Gamma\left(\dfrac{p+1}{2}\right)^2} - O\left(\frac{1}{p}\right) =1+\frac{1}{2p}+\frac{1}{8p^2}+\cdots-O\left(\frac{1}{p}\right)= 1+\Omega\left(\frac{1}{p}\right),
	\end{align*}
	where the last equality is due to the Taylor expansion of first term at $p=+\infty$, i.e.
	\[\dfrac{\Gamma\left(\dfrac{p}{2}\right)\Gamma\left(\dfrac{p+2}{2}\right)} {\Gamma\left(\dfrac{p+1}{2}\right)^2}=1+\dfrac{1}{2p}+\dfrac{1}{8p^2}+o\left(\dfrac{1}{p^2}\right).\]
	
	Similarly, we have
	\begin{align*}
		\mathbb{E}_\xi\phi_\xi=\EE_\zeta \theta_1^\zeta
		&=\int_{\mathbb{B}_\nu\left(1\right)}\left(\arccos\frac{x_1}{\norm{x}_2}\right) \frac{\Gamma\left(\frac{p}{2}+1\right)}{\pi^{p/2}}dx
		\leq \int_{\mathbb{B}_0\left(1+r_\nu\right)}\left(\arccos\frac{x_1}{\norm{x}_2}\right) \frac{\Gamma\left(\frac{p}{2}+1\right)}{\pi^{p/2}}dx\\
		&=\left(\int_{0}^{1+r_\nu} r^{p-1} dr \frac{\Gamma\left(\frac{p}{2}+1\right)}{\pi^{p/2}} \int_{0}^{\pi}\theta_1\sin^{p-2}\left(\theta_1\right)d\theta_1\right)\\
		&~~~\cdot\left( \int_{0}^{\pi}\cdots\int_{0}^{\pi}\int_{0}^{2\pi}\sin^{d-3}\theta_2\cdots\sin\theta_{d-2}d\theta_1 \cdots d\theta_{d-1}\right).\\
	\end{align*}	
	We then apply Lemma \ref{lem_gphi_gap} by taking $g(\theta)=\theta$ and get the following result.
	\begin{align*}
		\mathbb{E}_\xi\phi_\xi
		&= H_{p}\left(r_\nu\right)A_p\left(g\right)
		= M_{p}A_p\left(g\right) + \left(H_{p}\left(r_\nu\right)-M_{p}\right)A_p\left(g\right)\\
		&=\frac{\pi}{2}+O\left(\frac{1}{p}\right)\leq \frac{3\pi}{4}.
	\end{align*}
\end{proof}

\subsubsection{Proof of Lemma \ref{stage1_a}}
\begin{proof}
	%	We first show that if 
	%	$\left(w_t,a_t\right)\in\cA_{C_2,C_3}$ holds for $\forall t\leq T_0$, then $\mathbb{E}\norm{a_{T_0}-a^*}_2^2$ is sufficiently small. Then we have $$a_{T_0}^\top a^* =O\left(\frac{1}{p}\right)\norm{a^*}_2^2 \quad\mathrm{and}\quad\norm{a-a^*/2}_2^2\leq\norm{a^*}_2^2$$ with high probability, which is a contradiction. Thus, there exists $T_1\leq T_0$ such that $a_{T_1}^\top a^* \geq \frac{C_3}{p}\norm{a^*}_2^2$ and $\norm{a_{T_1}-a^*/2}_2^2\leq\norm{a^*}_2^2$ for some constant $C_3.$
	
	Denote $\mathscr{E}_t=\{\forall\tau\leq t, a_\tau\in \cA_{C_2,C_3}\}$. Then, by Theorem \ref{thm_SC}, when $(w,a)\in\cA_{C_2,C_3}$, we have
	\begin{align*}
		\langle-\mathbb{E}_{\xi,\epsilon}\nabla_aL\left(w_t+\xi,a_t+\epsilon\right),a^*-a_t\rangle
		&\geq \frac{C}{p} \norm{a_t-a^*}_2^2,
	\end{align*}
	for some constant $C$.
	
	We next bound the expectation of the norm of the perturbed gradient.
	\begin{align*}
		\mathbb{E}_{\xi,\epsilon}\norm{\nabla_aL\left(w_t+\xi,a_t+\epsilon\right)}_2^2
		&=\mathbb{E}_{\xi,\epsilon}\norm{\nabla_aL\left(w_t+\xi,a_t+\epsilon\right)-\nabla_aL\left(w^*,a^*\right)}_2^2\\
		&=\mathbb{E}_{\xi,\epsilon}
		\norm{\frac{1}{2\pi}\left(\mathds{1}\one^\top+\left(\pi-1\right)I\right)\left(a_t+\epsilon-a^*\right)
			-\frac{g\left(\phi_\xi\right)-\pi}{2\pi}a^*}_2^2\\
		&\leq \frac{1}{2\pi^2}\mathbb{E}_{\xi,\epsilon}\norm{\left(\mathds{1}\one^\top+\left(\pi-1\right)I\right)\left(a_t+\epsilon-a^*\right)}_2^2 
		+\frac{1}{2}\norm{a^*}_2^2\\
		&\leq \frac{\left(k+\pi-1\right)^2}{\pi^2}\left(\norm{a_t-a^*}_2^2+\rho_a^2\right) 
		+\frac{1}{2}\norm{a^*}_2^2.
	\end{align*}
	Therefore,  the expectation $\mathbb{E}\left[\norm{a_{t+1}-a^*}_2^2\mathds{1}_{\mathscr{E}_t}\right]
	$ can be bounded as follows.
	\begin{align*}
		\mathbb{E}\left[\norm{a_{t+1}-a^*}_2^2\mathds{1}_{\mathscr{E}_t}\right]
		&=\mathbb{E}\left[\norm{a_t-a^*}_2^2\mathds{1}_{\mathscr{E}_t}\right]
		-2\mathbb{E}\left[\langle-\eta\mathbb{E}_{\xi,\epsilon}\nabla_aL\left(w_t+\xi,a_t+\epsilon\right),a^*-a_t\rangle\mathds{1}_{\mathscr{E}_t}\right]\\
		&+\mathbb{E}\left[\norm{\eta \nabla_aL\left(w_t+\xi,a_t+\epsilon\right)}_2^2\mathds{1}_{\mathscr{E}_t}\right]\\
		&\leq \left(1-\eta \frac{C}{p} + \eta^2\frac{\left(k+\pi-1\right)^2}{\pi^2}\right)\left[ \mathbb{E}\norm{a_t-a^*}_2^2\mathds{1}_{\mathscr{E}_t}\right]\\
		&+\frac{\eta^2\left(k+\pi-1\right)^2}{\pi^2}\rho_a^2 +\frac{\eta^2}{2}\norm{a^*}_2^2\\
		&\leq\left(1-\lambda_1\right)\mathbb{E}\norm{a_t-a^*}_2^2\mathds{1}_{\mathscr{E}_{t-1}} + b_1,
	\end{align*}
	where $\lambda_1=\eta \frac{C}{p} - \eta^2\frac{\left(k+\pi-1\right)^2}{\pi^2}$ and $b_1=\frac{\eta^2\left(k+\pi-1\right)^2}{\pi^2}\rho_a^2 +\frac{\eta^2}{2}\norm{a^*}_2^2$.
	
	Note that $\mathscr{E}_t$ implies that $\norm{a_t-a^*}_2^2\geq \frac{1}{4}\norm{a^*}_2^2$, then we have
	\begin{equation}
		\begin{aligned}
			\frac{1}{4}\norm{a^*}_2^2\mathbb{P}\left(\mathscr{E}_t\right)
			&\leq \mathbb{E}\left[\norm{a_t-a^*}_2^2\mathds{1}_{\mathscr{E}_t}\right]\leq \mathbb{E}\left[\norm{a_t-a^*}_2^2\mathds{1}_{\mathscr{E}_{t-1}}\right]
			\leq \left(1-\lambda_1\right)^t \norm{a_0-a^*}_2^2 + \frac{b_1}{\lambda_1}.
		\end{aligned}
	\end{equation}
	
	With our choice of small $\eta$, we have $\lambda_1=O\left(\eta/p\right)\in\left(0,1\right)$,$\frac{b_1}{\lambda_1}\leq \frac{1}{8}\norm{a^*}_2^2$. Thus when $t=\tilde{O}\left(\frac{p}{\eta}\right)$, we have $\mathbb{P}\left(\mathscr{E}_t\right) \leq \frac{1}{2}$. We recursively apply the same procedure with $\log\frac{1}{\delta}$ times, and after $T_0=\tilde{O}\left(\frac{p}{\eta}\log\frac{1}{\delta}\right)$, we have $\mathbb{P}\left(\mathscr{E}_{T_0}\right) < \delta$, which implies that with probability at least $1-\delta$, there exists $\tau_{11}\leq T_0$ such that $$\frac{C_2}{p}\norm{a^*}_2^2\leq a_{\tau_{11}}^\top a^*~\text{and } \norm{a_{\tau_{11}}-a^*/2}_2^2\leq\norm{a^*}_2^2,$$ for some constant $C_2$. 
	Moreover, $\norm{a_{\tau_{11}}-a^*/2}_2^2\leq\norm{a^*}_2^2$ further implies $a_{\tau_{11}}^\top a^*\leq2\norm{a^*}_2^2.$
	Thus, we have $$\frac{C_2}{p}\norm{a^*}_2^2\leq a_{\tau_{11}}^\top a^*\leq2\norm{a^*}_2^2.$$	
	Then by Lemma \ref{lem_sum_a} and Lemma \ref{lem_inner_a}, we get the desired result.
\end{proof}

\subsubsection{Proof of Lemma \ref{lemma_w_escape}}

\begin{proof}
	
	%	We first compute the noise term in the gradient.
	%	\begin{align*}
	%	&\left(\frac{\norm{a+\epsilon}_2^2}{2}+\frac{\sum_{i\neq j}\left(a_i+\epsilon_i\right)\left(a_j+\epsilon_j\right)}{2\pi}
	%	-\frac{\left(a+\epsilon\right)^\top a^* \sin\phi_\xi}{2\pi}\frac{1}{\norm{w+\xi}_2}
	%	-\frac{\sum_{i\neq j}\left(a_i+\epsilon_i\right) a_j^*}{2\pi}\frac{1}{\norm{w+\xi}_2}\right)\xi\\
	%	=& \left(\frac{2\pi-1}{2\pi}\norm{a+\epsilon}_2^2+\frac{1}{2\pi}\left(\one^\top a+\one^\top \epsilon\right)^2-\frac{1}{2\pi}\left(\one^\top a+\one^\top \epsilon\right)\frac{\one^\top a^*}{\norm{w+\xi}}+\frac{1}{2\pi}\frac{\left(a+\epsilon\right)^\top a^*\left(1-\sin\phi_\epsilon\right)}{\norm{w+\xi}}\right)\xi.
	%	\end{align*}
	When $\rho_w$ is sufficiently large, with probability $1-\delta$, the norm of perturbed gradient w.r.t. $w$, i.e., $\norm{\nabla_w L(w,a)}$, is at least $O\left(\rho_w\right)$ once in $\log\frac{1}{\delta}$ steps. Thus,  with  at least probability  $1-\delta,$ there exists $t\leq\log\frac{1}{\delta}$ such that $$w_t^\top w^*=-1+O\left(\eta^2\rho_w^2\right).$$ We take this point as $w_0$ in the later proof.
	Recall that $\tilde{w}_t=w_{t-1}-\eta\left(I-w_{t-1}w_{t-1}^\top\right)\nabla_wL$ and $w_t=\mathrm{Proj}_{\SSS_0\left(1\right)}\left(\tilde{w}_t\right).$	Without loss of generality, we assume $\tilde{w}_{t+1}^\top w^*\leq0$, otherwise we already have $1+w_{t+1}^\top w^*\geq1.$
	Notice that $\norm{\tilde{w}_{t+1}}_2\geq 1$, we then have
	\begin{align*}
		1+{w}_{t+1}^\top w^*
		=&1+\frac{\tilde{w}_{t+1}^\top w^*}{\norm{\tilde{w}_{t+1}}_2}\geq 1+{\tilde{w}_{t+1}^\top w^*}\\
		=& 1+w_{t}^\top w^*-\eta {w^*}^\top\left(I-w_{t}w_{t}^\top\right)\nabla_wL(w_t+\xi,a_t+\epsilon)\\
		=& 1+w_{t}^\top w^*+\frac{\eta}{2\pi}(1+w_{t}^\top w^*)(1-w_t^\top w^*)(a_t+\epsilon)^\top a^* \left(\pi-\phi_\xi\right)-\eta {w^*}^\top\left(I-w_{t}w_{t}^\top\right)\Big(\frac{\norm{a_t+\epsilon}_2^2}{2}\\
		&+\frac{\sum_{i\neq j}\left(a_{t,i}+\epsilon_i\right)\left(a_{t,j}+\epsilon_j\right)}{2\pi}
		-\frac{\left(a_t+\epsilon\right)^\top a^* \sin\phi_\xi}{2\pi}\frac{1}{\norm{w_t+\xi}_2}
		-\frac{\sum_{i\neq j}\left(a_{t,i}+\epsilon_i\right) a_j^*}{2\pi}\frac{1}{\norm{w_t+\xi}_2}\Big)\xi.
	\end{align*}
	Thus, we have 
	\begin{align*}
		\EE\left(1+{w}_{t+1}^\top w^*|\mathcal{F}_t\right)
		\geq&(1+w_{t}^\top w^*)\left(1+\frac{\eta}{2\pi}(1-w_t^\top w^*)a_t^\top a^* \left(\pi-\EE_\xi\phi_\xi\right)\right)\\
		&+\eta\EE_\xi\left(w^*-w_t^\top w^*w_t\right)^T\left(
		\frac{a_t^\top a^* \sin\phi_\xi}{2\pi}\frac{1}{\norm{w_t+\xi}_2}
		+\frac{\sum_{i\neq j}{a_{t,i} a_j^*}}{2\pi}\frac{1}{\norm{w_t+\xi}_2}\right)\xi\\
		=&(1+w_{t}^\top w^*)\left(1+\frac{\eta}{2\pi}(1-w_t^\top w^*)a_t^\top a^* \left(\pi-\EE_\xi\phi_\xi\right)\right)\\
		&+\frac{\eta}{2\pi}{a_t^\top a^*}\EE_{\xi}
		\frac{ \sin\phi_\xi\left(w^*-w_t^\top w^*w_t\right)^T\xi}{\norm{w_t+\xi}_2},
	\end{align*}
	where the last line is due to \eqref{eq1}.
	
	We next show that
	\begin{align}\label{ineq2}
		\EE_{\xi}
		\frac{ \sin\phi_\xi\left(w^*-w^\top w^*w\right)^T\xi}{\norm{w+\xi}_2}\geq 0
	\end{align}
	
	Let $\phi_x=\angle(x,w^*).$ 
	\begin{align}
		\EE_{\xi}
		\frac{ \sin\phi_\xi\left(w^*-w^\top w^*w\right)^T\xi}{\norm{w+\xi}_2}
		=&\int_{\BB_{w}\left(\rho_w\right)}\frac{1}{V\left(\rho_w\right)}\frac{\sin(\phi_x)\left(w^*-w^\top w^*w\right)^\top x}{\norm{x}_2}dx.\\
		=&\int_{\BB_{w}\left(\rho_w\right)\cap\BB_{-w}\left(\rho_w\right)}\frac{1}{V\left(\rho_w\right)}\frac{\sin(\phi_x)\left(w^*-w^\top w^*w\right)^\top x}{\norm{x}_2}dx\\
		&+\int_{\BB_{w}\left(\rho_w\right)\backslash\BB_{-w}\left(\rho_w\right)}\frac{1}{V\left(\rho_w\right)}\frac{\sin(\phi_x)\left(w^*-w^\top w^*w\right)^\top x}{\norm{x}_2}dx.
	\end{align}
	Let's calculate these two integrals separately. For the first integral,
	\begin{align*}
		&\int_{\BB_{w}\left(\rho_w\right)\cap\BB_{-w}\left(\rho_w\right)}\frac{1}{V\left(\rho_w\right)}\frac{\sin(\phi_x)\left(w^*-w^\top w^*w\right)^\top x}{\norm{x}_2}dx\\
		=&
		\int_{\BB_{w}\left(\rho_w\right)\cap\BB_{-w}\left(\rho_w\right),\left(w^*-w^\top w^*w\right)^\top x>0}\frac{1}{V\left(\rho_w\right)}\frac{\sin(\phi_x)\left(w^*-w^\top w^*w\right)^\top x}{\norm{x}_2}dx\\
		&+\int_{\BB_{w}\left(\rho_w\right)\cap\BB_{-w}\left(\rho_w\right),\left(w^*-w^\top w^*w\right)^\top x<0}\frac{1}{V\left(\rho_w\right)}\frac{\sin(\phi_x)\left(w^*-w^\top w^*w\right)^\top x}{\norm{x}_2}dx
	\end{align*}
	By the symmetric property with respect to the origin, we have
	\begin{align*}
		&\int_{\BB_{w}\left(\rho_w\right)\cap\BB_{-w}\left(\rho_w\right)}\frac{1}{V\left(\rho_w\right)}\frac{\sin(\phi_x)\left(w^*-w^\top w^*w\right)^\top x}{\norm{x}_2}dx=0.
	\end{align*}
	For the second integral, let's consider the the symmetric point of $x$ with respect to the vector $w$, i.e., $\tilde{x}=2w^\top xw-x.$ We have the following properties:
	$$\norm{\tilde{x}}_2=\norm{x}_2,~\norm{w+x}_2=\norm{w+\tilde{x}}_2,~\text{and }\left(w^*-w^\top w^*w\right)^\top \tilde{x}=-\left(w^*-w^\top w^*w\right)^\top x.$$
	We further have
	\begin{align*}
		\sin\phi_{\tilde{x}}=\sqrt{1-\cos^2\phi_{\tilde{x}}}=\sqrt{1-\left(\frac{(w^*)^\top\tilde{x}}{\norm{\tilde{x}}_2}\right)^2}
		&=\sqrt{1-\left(\frac{(2w^\top x(w^*)^\top w-(w^*)^\top x)}{\norm{{x}}_2}\right)^2}\\
		&=\sqrt{1-\left(\frac{(w^*)^\top x}{\norm{{x}}_2}\right)^2
			+\left(\frac{(4w^\top x(w^*)^\top w\left[(w^*)^\top x-w^\top x(w^*)^\top w\right]} {\norm{{x}}_2^2}\right)}.
	\end{align*}
	Since $x\in \BB_{w}\left(\rho_w\right)\backslash\BB_{-w}\left(\rho_w\right),$ we have $\norm{x+w}_2^2\geq \rho_w^2\geq \norm{x-w}_2^2,$ which implies  $w^\top x\geq 0.$
	Moreover, when $,\left(w^*-w^\top w^*w\right)^\top x>0,$ we have $(w^*)^\top x-w^\top x(w^*)^\top w\geq 0.$
	Together with $(w^*)^\top w\leq0,$ a we have $$\sin\phi_{\tilde{x}}\leq \sin\phi_{{x}}.$$ 
	Then the second integral can be estimated as follows.
	\begin{align*}
		&\int_{\BB_{w}\left(\rho_w\right)\backslash\BB_{-w}\left(\rho_w\right)}\frac{1}{V\left(\rho_w\right)}\frac{\sin(\phi_x)\left(w^*-w^\top w^*w\right)^\top x}{\norm{x}_2}dx\\
		=&\int_{\BB_{w}\left(\rho_w\right)\backslash\BB_{-w}\left(\rho_w\right),\left(w^*-w^\top w^*w\right)^\top x>0} \frac{1}{V\left(\rho_w\right)}\frac{\sin(\phi_x)\left(w^*-w^\top w^*w\right)^\top x}{\norm{x}_2}dx\\
		&+\int_{\BB_{w}\left(\rho_w\right)\backslash\BB_{-w}\left(\rho_w\right),\left(w^*-w^\top w^*w\right)^\top x<0} \frac{1}{V\left(\rho_w\right)}\frac{\sin(\phi_x)\left(w^*-w^\top w^*w\right)^\top x}{\norm{x}_2}dx\\
		%=&\int_{\BB_{w}\left(\rho_w\right)\backslash\BB_{-w}\left(\rho_w\right),\left(w^*-w^\top w^*w\right)^\top x>0}\frac{1}{V\left(\rho_w\right)}\frac{\sin(\phi_x)\left(w^*-w^\top w^*w\right)^\top x}{\norm{x}_2}dx\\&-
		%\int_{\BB_{w}\left(\rho_w\right)\backslash\BB_{-w}\left(\rho_w\right),\left(w^*-w^\top w^*w\right)^\top \tilde{x}>0}\frac{1}{V\left(\rho_w\right)}\frac{\sin(\phi_{\tilde{x}})\left(w^*-w^\top w^*w\right)^\top {\tilde{x}}}{\norm{x}_2}dx\\
		=&\int_{\BB_{w}\left(\rho_w\right)\backslash\BB_{-w}\left(\rho_w\right),\left(w^*-w^\top w^*w\right)^\top x>0} \frac{1}{V\left(\rho_w\right)}\frac{\left(w^*-w^\top w^*w\right)^\top x}{\norm{x}_2}(\sin(\phi_x)-\sin(\phi_{\tilde{x}}))dx\\
		\geq& 0
	\end{align*}
	Thus, combining the above calculations for two integrals, \eqref{ineq2} is proved. Then with the fact that $a_t^\top a^*\geq m_a>0$ for all $t$, we have
	\begin{align*}
		\EE\left(1+{w}_{t+1}^\top w^*|\mathcal{F}_t\right)
		&\geq(1+w_{t}^\top w^*)\left(1+\frac{\eta}{2\pi}(1-w_t^\top w^*)a_t^\top a^* \left(\pi-\EE_\xi\phi_\xi\right)\right)\\
		&\geq(1+C_1\frac{\eta}{p})(1+w_{t}^\top w^*),
	\end{align*}
	for some positive constant $C_1$.
	
	Thus,
	\begin{align*}
		\EE\left(1+{w}_t^\top w^*\right)\geq(1+C_1\frac{\eta}{p})^t(1+w_{0}^\top w^*).
	\end{align*}
	When $t=\tilde{O}(\frac{p}{\eta}\log\frac{1}{\eta}),$ we have $\EE\left(1+{w}_t^\top w^*\right)\geq C_2$ for some constant $C_2\in(-1,0)$. Thus, with constant probability we have $1+{w}_t^\top w^*\geq C_2$. And We could have with at least probability $1-\delta,$ $1+w_{\tau_{12}}^\top w^*\geq C_2$ for some $\tau_{12}=\tilde{O}\left(\frac{p}{\eta}\log\frac{1}{\eta}\log \frac{1}{\delta}\right).$
\end{proof}

\subsubsection{Proof of Lemma \ref{stage-nice}}
\begin{proof}	
	Recall that we have $-1<C_4\leq w_0^\top w^*\leq 0$ and $m_a\leq a_t^\top a^*\leq M_a$ for all $t$. 
	
	Our proof has two steps.
	
	\textbf{Step 1:} We show that $w_t^\top w^*$ have a lower bound $\frac{C_4-1}{2}$with probability $1-\delta$ for $\forall t\leq \tilde{O}\left(\frac{1}{\eta^2}\right)$.
	
	Denote $\mathscr{E}_t=\{\forall\tau \leq t, w_\tau^\top w^*\geq \frac{C_4-1}{2}\}\subset\mathcal{F}_t$. Then if $\mathscr{E}_t$ holds, we have $(w_\tau,a_\tau)\in\mathcal{K}_{\left(C_4-1\right)/2,m_a,M_a}$ for $\forall\tau \leq t$. Recall that $\tilde{w}_{t+1}$ is defined as $$\tilde{w}_{t+1}=w_t-\left(I-w_tw_t^\top\right)\nabla_wL\left(w_t+\xi_t,a_t+\epsilon_t\right),$$ and $$w_{t+1}=\text{Proj}_{\SSS_0}\left(\tilde{w}_{t+1}\right).$$
	
	By Theorem \ref{thm_SC}, when $\left(w,a\right)\in\mathcal{K}_{\left(C_4-1\right)/2,m_a,M_a}$, we have
	\begin{align*}
		\langle-\mathbb{E}_{\xi,\epsilon}\left(I-ww^\top\right)\nabla_wL\left(w+\xi,a+\epsilon\right),w^*-w\rangle\geq \frac{m_a\left(1+C_4\right)}{32}\norm{w-w^*}_2^2-\gamma,
	\end{align*}
	where $\gamma=O\left(k/\rho_w\right)$.\\
	Moreover, we have the bound on expectation of the norm of the perturbed (manifold) gradient.
	\begin{align*}
		\mathbb{E}_{\xi,\epsilon}\norm{\left(I-ww^\top\right)\nabla_wL\left(w+\xi,a+\epsilon\right)}_2^2
		=&\frac{\mathbb{E}_\epsilon\left(\left(a+\epsilon\right)^\top a^*\right)^2 \mathbb{E}_{\xi}\left(\pi -\phi\right)^2}{2\pi^2} w^{*\top}\left(I-ww^\top\right)w^*\\
		+2&\EE_{\xi,\epsilon}\Big(\frac{\norm{a+\epsilon}_2^2}{2}+\frac{\sum_{i\neq j}\left(a_i+\epsilon_i\right)\left(a_j+\epsilon_j\right)}{2\pi}
		-\frac{\left(a+\epsilon\right)^\top a^* \sin\phi_\xi}{2\pi}\frac{1}{\norm{w+\xi}_2}\\
		&-\frac{\sum_{i\neq j}\left(a_i+\epsilon_i\right) a_j^*}{2\pi}\frac{1}{\norm{w+\xi}_2}\Big)^2
		\xi^\top\left(I-ww^T\right)\xi\\
		&\leq \frac{M_a^2 + \rho_a^2\norm{a^*}_2^2}{2}+C_1k^2\rho_w^2\rho_a^2,
	\end{align*}
	where $C_1$ is a constant.\\
	Combine the above two inequalities, we get 
	\begin{align*}
		\mathbb{E}[\norm{\tilde{w}_{t+1}-w^*}_2^2\mathds{1}_{\mathscr{E}_t}|\mathcal{F}_t]
		=&\norm{w_t-w^*}_2^2\mathds{1}_{\mathscr{E}_t} - 2\langle-\eta\mathbb{E}_{\xi_t,\epsilon_t}\nabla_wL\left(w_t+\xi_t,a_t+\epsilon_t\right),w^*-w_t\rangle \mathds{1}_{\mathscr{E}_t}\\
		&+\mathbb{E}_{\xi_t,\epsilon_t}\norm{\eta\left(I-w_tw_t^\top\right)\nabla_wL\left(w_t+\xi_t,a_t+\epsilon_t\right)}_2^2 \mathds{1}_{\mathscr{E}_t}\\
		\leq& \left(1-\frac{\eta m_a\left(1+C_4\right)}{16}\right)\norm{w_t-w^*}_2^2 \mathds{1}_{\mathscr{E}_t} 
		+\left(\eta\gamma+\frac{\eta^2\left(M_a^2 + \rho_a^2\norm{a^*}_2^2\right)}{2} +\eta^2C_1k^2\rho_w^2\rho_a^2\right)\mathds{1}_{\mathscr{E}_{t}}\\
		=&\left(1-\lambda_2\right)\norm{w_t-w^*}_2^2\mathds{1}_{\mathscr{E}_{t}} + b_2\mathds{1}_{\mathscr{E}_{t}},
	\end{align*}
	where $\lambda_2=O\left(\eta/p\right)$, $b_2=O\left(\eta k/\rho_w\right)$ and $\frac{b_2}{\lambda_2}\leq\min\{\frac{1+C_4}{2},\frac{1}{4}\}$ by proper choice of small $\eta$ and large $\rho_w$.
	
	We next show that 
	$\norm{w_{t+1}-w^*}_2^2 \leq \norm{\tilde{w}_{t+1}-w^*}_2^2$. We first have the following inequality.
	\begin{align*}
		\norm{\tilde{w}_{t+1}}_2^2=\norm{w_t}_2^2+\norm{\eta\left(I-w_tw_t^\top\right)\nabla_wL\left(w_t+\xi,a_t+\epsilon\right)}_2^2
		\geq 1.
	\end{align*}
	Since we have $w_{t+1}^\top w^*\leq 1$, we obtain
	\begin{align*}
		\norm{\tilde{w}_{t+1}-w^*}_2^2
		&= 1 + \norm{\tilde{w}_{t+1}}_2^2-2\norm{\tilde{w}_{t+1}}_2w_{t+1}^\top w^*\\
		&\geq 1 + 1-2w_{t+1}^\top w^*=\norm{w_{t+1}-w^*}_2^2.
	\end{align*}
	The above inequality comes from $a^2-2ab+1\geq2-2b\Leftrightarrow a+1\geq 2b$ for $a\geq 1$. Therefore, we have$$\mathbb{E}[\norm{w_{t+1}-w^*}_2^2\mathds{1}_{\mathscr{E}_t}|\mathcal{F}_t]
	\leq\left(1-\lambda_2\right)\norm{w_t-w^*}_2^2\mathds{1}_{\mathscr{E}_{t}} + b_2\mathds{1}_{\mathscr{E}_{t}}.$$
	
	Denote $G_t=\left(1-\lambda_2\right)^{-t}\left(\norm{w_t-w^*}_2^2-\frac{b_2}{\lambda_2}\right)$, the above recursive relation becomes	$$\mathbb{E}[G_{t+1}\mathds{1}_{\mathscr{E}_t}|\mathcal{F}_t]
	\leq G_t\mathds{1}_{\mathscr{E}_{t}}
	\leq G_t\mathds{1}_{\mathscr{E}_{t-1}}.$$
	Thus, $\{G_t\mathds{1}_{\mathscr{E}_{t}}\}$ is a supermartingale.
	
	We then have the following bound.
	\begin{align*}
		d_{t}
		&\triangleq	
		|G_{t}\mathds{1}_{\mathscr{E}_{t-1}}-\mathbb{E}[G_{t}\mathds{1}_{\mathscr{E}_{t-1}}|\mathcal{F}_{t-1}]|=\left(1-\lambda_2\right)^{-t}|\norm{w_t-w^*}_2^2-\mathbb{E}[\norm{w_t-w^*}_2^2|\mathcal{F}_{t-1}]|\\
		&\leq \left(1-\lambda_2\right)^{-t}\left(\eta\left(M_a+\rho_a\norm{a^*}_2\right) 
		%+\frac{\eta^2}{4}\left(M_a^2 +2\rho_a^2\norm{a^*}_2^2 + 2M_a\rho_a\norm{a^*}_2\right)
		+C_2\eta k\rho_w\right) \\
		&= \left(1-\lambda_2\right)^{-t}M_2,
	\end{align*}
	where $M_2=O\left(\eta k \rho_w\right)$.
	
	Denote $r_t=\sqrt{\sum_{i=0}^{t}d_i^2}$. By Azuma's Inequality,
	$$\mathbb{P}\left(G_t\mathds{1}_{\mathscr{E}_{t-1}}-G_0 \geq\tilde{O}\left(1\right)r_t\log^{\frac{1}{2}}\left(\frac{1}{\eta^2\delta}\right)\right) 
	\leq \exp\left(-\frac{\tilde{O}\left(1\right)r_t^2\log\left(\frac{1}{\eta^2\delta}\right)}{2\sum_{i=0}^{t}d_i^2}\right)
	=\tilde{O}\left(\eta^2\delta\right)$$
	
	Therefore, with at least probability $1-\tilde{O}\left(\eta^2\delta\right)$, we have
	\begin{align*}
		\norm{w_t-w^*}_2^2
		&\leq \left(1-\lambda_2\right)^t \left(\norm{w_0-w^*}_2^2-\frac{b_2}{\lambda_2}\right) 
		+ \tilde{O}\left(1\right)\left(1-\lambda_2\right)^t r_t\log^{\frac{1}{2}}\left(\frac{1}{\eta^2\delta}\right)
		+ \frac{b_2}{\lambda_2}\\
		&\leq \norm{w_0-w^*}_2^2  
		+ \tilde{O}\left(1\right)\frac{M}{\sqrt{\lambda_2}}\log^{\frac{1}{2}}\left(\frac{1}{\eta^2\delta}\right)
		+ \frac{b_2}{\lambda_2}\\
		&\leq 3-C_4,
	\end{align*}
	where the last line is true by our choice of small $\eta$.
	
	The above inequality shows that $w_{t}^\top w^*\geq\frac{C_4-1}{2}$ holds with at least probability $1-O\left(\eta^{2}\delta\right)$, which implies that $\mathscr{E}_{t}$ holds with at least probability $1-O\left(\eta^{2}\delta\right)$ when $\mathscr{E}_{t-1}$ holds. Hence, with at least probability $1-\delta$, we have $w_t^\top w^*\geq\frac{C_4-1}{2}$ for all $t\leq\tilde{O}\left(\frac{1}{\eta^2}\right)$.\\

	\textbf{Step 2:} We  show that if the result in Step 1 holds, there exists $\tau_{13}=\tilde{O}\left(\frac{p}{\eta}\log\frac{1}{\delta}\right)$ such that $\phi_{\tau_{13}}\leq \frac{\pi}{3}$ and $\phi_t$ stays in the region $\left\{\phi\Big|\phi\leq\frac{5\pi}{12}\right\}$ with probability $1-\delta$ during the later $\tilde{O}\left(\frac{1}{\eta^2}\right)$ steps.
	
	%In the following proof, we assume $w_t^\top w^*\geq c$ for all $t\leq\tilde{O}\left(\frac{1}{\eta^2}\right)$. As shown in Step 1, this is true with high probability.
	
	Following similar lines to Step 1, we have
	$$\mathbb{E}[G_{t+1}|\mathcal{F}_t]\leq G_t,$$
	where 
	$G_t=\left(1-\lambda_2\right)^{-t}\left(\norm{w_t-w^*}_2^2-\frac{b_2}{\lambda_2}\right)$. % and $\mathscr{E}_t=\{\forall\tau \leq t, w_\tau^\top w^*\geq \frac{C_4-1}{2}\}\subset\mathcal{F}_t$.
	
	Thus, recall that $\lambda_2=O(\eta/p)$, $\frac{b_2}{\lambda_2}\leq\frac{1}{4}$, and let $t=\tilde{O}\left(\frac{p}{\eta}\right)$, and we know that
	
	\begin{align*}
		\mathbb{E}[\norm{w_{t}-w^*}_2^2|\mathcal{F}_{t-1}]
		&\leq \left(1-\lambda_2\right)^{t}\norm{w_0-w^*}_2^2+\frac{b_2}{\lambda_2}
		\leq \frac{1}{2}.
	\end{align*}
	
	By Markov Inequality, we know 
	$$\mathbb{P}\left(\norm{w_{t}-w^*}_2^2\geq 1\right)\leq\frac{1}{2}.$$ We recursively apply the above inequality with $\log\frac{1}{\delta}$ times. Then, with at most $\tau_{13}=\tilde{O}\left(\frac{1}{\eta}\log\frac{1}{\delta}\right)$
	,we have 
	$$\mathbb{P}\left(\phi_{\tau_{13}}\leq\frac{\pi}{3}\right)
	=\mathbb{P}\left(\norm{w_{\tau_{13}}-w^*}_2^2\leq 1\right)\geq1-\delta.$$
	
	For notational simplify, we assume $\phi_0\leq\frac{\pi}{3}$ in the later proof. We will show that $\phi_t$ stays in the region $\left\{\phi\Big|\phi\leq\frac{5\pi}{12}\right\}$ with high probability during the later $\tilde{O}\left(\frac{1}{\eta^2}\right)$ steps.
	
	Denote $\mathscr{H}_t=\{\forall\tau\leq t,\phi_\tau\leq\frac{5\pi}{12}\}$. With the similar argument in Step 1, when $\mathscr{H}_{t-1}$ holds, with at least probability $1-\tilde{O}\left(\eta^2\delta\right)$, we have
	\begin{align*}
		\norm{w_t-w^*}_2^2
		&\leq \left(1-\lambda_2\right)^t\left(\norm{w_0-w^*}_2^2-\frac{b_2}{\lambda_2}\right) +\frac{b_2}{\lambda_2}
		+\tilde{O}\left(1\right)\left(1-\lambda_2\right)^t r_t\log^{\frac{1}{2}}\left(\frac{1}{\eta\delta}\right)\\
		&\leq \norm{w_0-w^*}_2^2 + \frac{b_2}{\lambda_2} + \tilde{O}\left(1\right)\frac{M}{\sqrt{\lambda_2}}\log^{\frac{1}{2}}\left(\frac{1}{\eta\delta}\right)\leq 1.4,
	\end{align*}
	which implies that $\phi_t\leq\frac{5\pi}{12}$, i.e., $\mathscr{H}_t$ holds.
	
	Hence, for all $t\leq T=\tilde{O}\left(\frac{1}{\eta^2}\right)$, we have $\phi_t\leq\frac{5\pi}{12}$ with at least probability $1-\delta$.
	
	Combining the above two steps, with probability $1-\delta$, we have $\phi_t\leq 5\pi/12$ for all $t$'s such that $\tau_{13}\leq t\leq T=\tilde{O}(\eta^{-2})$, where
	$\tau_{13}=\tilde{O}\Big(\frac{p}{\eta}\log\frac{1}{\delta}\Big).$ 
	
\end{proof}

\section{Proof for Phase II}\label{pf_phase2}
\subsection{Technical Lemma}
The next lemma shows that perturbed GD imitates the behavior of GD, when the noise is small enough. Thus, it can finally converge to the global optimum.
\begin{lemma} \label{lem_gphi_2}
	Denote $g\left(\phi\right)=\left(\pi-\phi\right)\cos\phi+\sin\phi$, $\xi\sim\mathrm{unif}\left(\mathbb{B}_0\left(1\right)\right)\in\mathbb{R}^d$ and $w\in\mathbb{R}^d$, $\norm{w}_2=1$.
	Define
	$$\phi_\xi=\arccos\dfrac{v^\top\left(w+\rho\xi\right)}{\norm{v}_2 \ \norm{w+\rho\xi}_2},\quad\phi=\arccos\dfrac{v^\top w}{\norm{v}_2 \ \norm{w}_2}.$$
	Suppose $\phi\leq\frac{\pi}{2}$ and $\rho< 1$.
	Then we have
	\begin{align*}
		\mathbb{E}_\xi\phi_\xi\leq U_\phi^{(1)}\left(\rho\right),
		\quad\mathbb{E}_\xi\left(\pi-g\left(\phi_\xi\right)\right)^2 \leq U_\phi^{(2)}\left(\rho\right),
		\quad\mathbb{E}_\xi g\left(\phi_\xi\right) \geq U_\phi^{(3)}\left(\rho\right).
	\end{align*}
	where $\lim\limits_{\rho\to0} U_\phi^{(1)}\left(\rho\right)=\phi$, 
	$\lim\limits_{\rho\to0} U_\phi^{(2)}\left(\rho\right)=\left(\pi-g\left(\phi\right)\right)^2$, 
	$\lim\limits_{\rho\to0} U_\phi^{(3)}\left(\rho\right)=g\left(\phi\right)$ and 
	$U_\phi^{(1)}\left(\rho\right)$, $U_\phi^{(2)}\left(\rho\right)$ is non-decreasing,
	$U_\phi^{(3)}\left(\rho\right)$ is non-increasing.
	Moreover, we have
	$$|\mathbb{E}_\xi\phi_\xi-\phi|= O\left(\rho\right),\quad |\mathbb{E}_\xi\left(\pi-g\left(\phi_\xi\right)\right)^2-\left(\pi-g\left(\phi\right)\right)^2|=O\left(\rho\right),\quad
	|\mathbb{E}_\xi g\left(\phi_\xi\right)-g\left(\phi\right)|=O\left(\rho\right).$$
\end{lemma}
\begin{proof}
	Without loss of generality, let $v=\left(1,0,...,0\right)^\top$. Then since $\phi\leq\frac{\pi}{2}$, we have $w_1\geq 0$.
	
	We  find the upper bound of $\phi_\xi=\arccos\frac{w_1+\rho\xi_1}{\norm{w+\rho\xi}_2}$, when $\rho$ and $w$ are fixed. $\phi_\xi$ could be explained as the angle between $X$ and $v$, where $X=w+\rho\xi \in \mathbb{B}_w\left(\rho\right)$. Thus, $\phi_\xi$ achieves the maximum when $X$ is tangent to $\mathbb{B}_w\left(\rho\right)$. This means that $\left(w+\rho\xi\right)^\top\rho\xi=0$ and $\norm{\xi}_2=1$, which is equivalent to $w^\top\xi+\rho = 0$ and $\norm{\xi}_2=1$. This leads to $\norm{w+\rho\xi}_2=\sqrt{1-\rho^2}$. Therefore, to get the upper bound of $\phi_\xi$, we need the lower bound of $\xi_1$. This is formulated as following,
	\begin{align*}
		\min {\xi_1} ~\textrm{s.t.}~\sum_i w_i\xi_i+\rho=0,\ \sum_i w_i^2=1,\ \sum_i \xi^2_i=1.
	\end{align*}
	
	By the Lagrange multiplier method, we have $\xi_1^*=-\sqrt{\left(1-\rho^2\right)\left(1-w_1^2\right)}-\rho w_1$. Thus, $$\phi_\xi\leq \arccos\left(w_1\sqrt{1-\rho^2}-\rho\sqrt{1-w_1^2}\right)\triangleq U_\phi^{(1)}\left(\rho\right).$$
	
	Moreover, with the same argument above, we have $$\phi_\xi\geq \arccos\left(w_1\sqrt{1-\rho^2}+\rho\sqrt{1-w_1^2}\right).$$
	Therefore, we have $|\mathbb{E}_\xi\phi_\xi-\phi|=|\mathbb{E}_\xi\phi_\xi-\arccos w_1|\leq C_1\rho$.
	
	Since $g\left(\phi\right)$ is decreasing, $\left(\pi-g\left(\phi\right)\right)^2$ is increasing and both of them are Lipschitz continuous, we have
	\begin{align*}
		&\mathbb{E}_\xi\left(\pi-g\left(\phi_\xi\right)\right)^2 
		\leq\left(\pi-g\left(U_\phi^{(1)}\left(\rho\right)\right)\right)^2\triangleq U_\phi^{(2)}\left(\rho\right),\quad
		\mathbb{E}_\xi g\left(\phi_\xi\right)
		\geq g\left(U_\phi^{(1)}\left(\rho\right)\right)\triangleq U_\phi^{(3)}\left(\rho\right),\\
		&|\mathbb{E}_\xi\left(\pi-g\left(\phi_\xi\right)\right)^2-\left(\pi-g\left(\phi\right)\right)^2|\leq C_2\rho,\quad
		|\mathbb{E}_\xi g\left(\phi_\xi\right) - g\left(\phi\right)| \leq C_3\rho.
	\end{align*}
	By simple manipulation, one can easily verify that 
	$\lim\limits_{\rho\rightarrow0} U_\phi^1\left(\rho\right)=\phi$, $\lim\limits_{\rho\rightarrow0} U_\phi^2\left(\rho\right)=\left(\pi-g\left(\phi\right)\right)^2$, $\lim\limits_{\rho\rightarrow0} U_\phi^3\left(\rho\right)=g\left(\phi\right)$ 
	and ${U^{(1)}_\phi}^\prime\left(\rho\right)\geq 0$, ${U^{(2)}_\phi}^\prime\left(\rho\right)\geq 0$, ${U^{(3)}_\phi}^\prime\left(\rho\right)\leq 0$.
\end{proof}

\subsection{Proof for Theorem \ref{thm_ASC}}\label{pf_ASC}
\begin{proof}
	When $\rho_w < 1$, we have
	\begin{align*}
		\EE_{\xi}\frac{\sin\phi_\xi\left(w^*-w^\top w^*w\right)^\top \xi}{\norm{w+\xi}_2}
		%=&\int_{\BB_w\left(\rho_w\right)} \frac{\sin\phi_{x}}{V\left(\rho_w\right)} \frac{\left(w^*-w^\top w^*w\right)^\top x}{\norm{x}_2}dx\\
		%=&\int_{\BB_w\left(\rho_w\right),\left(w^*-w^\top w^*w\right)^\top x>0} \frac{\sin\phi_x}{V\left(\rho_w\right)}  \frac{\left(w^*-w^\top w^*w\right)^\top x}{\norm{x}_2}\\
		%&+\int_{\BB_w\left(\rho_w\right),\left(w^*-w^\top w^*w\right)^\top x<0} \frac{\sin\phi_x}{V\left(\rho_w\right)} \frac{\left(w^*-w^\top w^*w\right)^\top x}{\norm{x}_2}\\
		%\leq& \int_{\BB_0\left(\rho_w+1\right),\left(w^*-w^\top w^*w\right)^\top \xi>0}\frac{1}{V\left(\rho_w\right)}\frac{\left(w^*-w^\top w^*w\right)^\top \xi}{1-\rho_w}\\
		%&-\int_{\BB_0\left(\rho_w\right),\left(w^*-w^\top w^*w\right)^\top \xi>0}\frac{1}{V\left(\rho_w\right)}\frac{\left(w^*-w^\top w^*w\right)^\top \xi}{1+\rho_2}\\
		%=&O\left(\rho_w^2\right).
		\leq C_1\sqrt{1-(w^\top w^*)^2}\rho_w=O(\rho_w)
	\end{align*}
	Taking $\rho_w=O\left(\frac{\gamma}{kp}\right)$ to be small enough and combining \eqref{eq1}, we have
	\begin{align*}
		\langle-\mathbb{E}_{\xi,\epsilon}\left(I-ww^\top\right)\nabla_wL\left(w+\xi,a+\epsilon\right),w^*-w\rangle
		&\geq\frac{a^\top a^* \left(\pi -\mathbb{E}_\xi\phi_\xi\right)}{2\pi}\left(1-(w^\top w^*)^2\right)-\gamma\\
		&=\frac{a^\top a^* \left(\pi -\mathbb{E}_\xi\phi_\xi\right)}{4\pi}\left(1+w^\top w^*\right)\norm{w-w^*}_2^2-\gamma\\
		&\geq \frac{m\left(1+C_9\right)}{16}\norm{w-w^*}_2^2-\gamma.\\
	\end{align*}
	Given small enough $\rho_w$, by Lemma \ref{lem_gphi_2} and $\norm{w-w^*}_2^2\leq C_{10}\gamma$, we have
	\begin{align*}
		\langle-\mathbb{E}_{\xi,\epsilon}\nabla_aL\left(w+\xi,a+\epsilon\right),a^*-a\rangle
		&= \frac{1}{2\pi}\left(\one^\top a- \one^\top a^*\right)^2+\frac{1}{2\pi}\left(\left(\pi-1\right)a-\left(\mathbb{E}_\xi g\left(\phi_\xi\right)-1\right)a^*\right)^\top \left(a-a^*\right)\\
		&=\frac{1}{2\pi}\left(\one^\top a- \one^\top a^*\right)^2
		+\frac{1}{2\pi}\left(\pi- \mathbb{E}_\xi g\left(\phi_\xi\right)\right)a^{*\top} \left(a-a^*\right)+\frac{\pi-1}{2\pi}\norm{a-a^*}_2^2\\
		&\geq \frac{\pi-1}{2\pi}\norm{a-a^*}_2^2 -\gamma.
	\end{align*}
\end{proof}

\subsection{Detailed Proof of Theorem \ref{stage2}}

\subsubsection{Proof of Lemma \ref{lem_keep}}
\begin{proof}
	Since $\phi_0\leq\frac{5\pi}{12}$, we have $g\left(\phi_0\right)>1.4$. By Lemma \ref{lem_gphi_2}, conditions $\EE_\xi\phi_\xi\leq\frac{3\pi}{4}$ and $1+O\left(\frac{1}{p}\right)\leq \EE_\xi g\left(\phi_\xi\right)\leq\pi$ are satisfied with our choice of small noise level $\rho_w$. Recall that $\one^\top a_0 \one^\top a^*- \left(\one^\top a^*\right)^2$ is bounded. Then, using Lemma \ref{lem_sum_a}, we have $\one^\top a_t \one^\top a^*- \left(\one^\top a^*\right)^2$ is still bounded in the same order for $\forall t\leq \tilde{O}\left(\eta^{-2}\right)$ with probability $1-\delta/3$. Combined with Lemma \ref{lem_inner_a}, with probability $1-\delta/3$ we have $m_a^\prime\leq a_t^\top a^*\leq M_a^\prime$ in the following $\tilde{O}\left(\eta^{-2}\right)$ steps, where $m_a^\prime=m_a/2$, $M_a^\prime=3M_a$. Then, following the same arguments in Step 2 of the proof of Lemma \ref{stage1_w}, we have $\phi_t\leq \frac{1}{2}\left(\frac{\pi}{2}+\frac{5\pi}{12}\right)=\frac{11\pi}{24}$ in the following $\tilde{O}\left(\eta^{-2}\right)$ steps with probability $1-\delta/3$. Therefore, combining the above results together, we have the desired results. 
\end{proof}

\subsubsection{Proof of Lemma \ref{stage2_w}}
\begin{proof}
	Recall that we have $\phi_t\leq\frac{11\pi}{24}$ and $m_a^\prime \leq a_t^\top a^*\leq M_a^\prime$ for all $t$. This implies that $w_t^\top w^*\geq0.1$. Thus, we have $\left(w_t,a_t\right)\in \cK_{0.1,m_a^\prime,M_a^\prime}$ for all $t$.
	
	The following two steps proof is similar with Lemma \ref{stage-nice}.
	
	\textbf{Step 1:} We  show that there exists $\tau_{21}=\tilde{O}\left(\frac{1}{\eta}\log\frac{1}{\gamma}\log\frac{1}{\delta}\right)$ such that $\norm{w_{\tau_{21}}-w^*}_2^2\leq\gamma/2$ holds with at least probability $1-\delta$.
	
	By Theorem \ref{thm_ASC}, for $(w,a)\in\cK_{0.1,m_a^\prime,M_a^\prime}$, we have
	\begin{align*}
		\langle-\mathbb{E}_{\xi,\epsilon}\left(I-ww^\top\right)\nabla_wL\left(w+\xi,a+\epsilon\right),w^*-w\rangle
		\geq \frac{11m_a}{160}\norm{w-w^*}_2^2-M_a\rho_w,
	\end{align*}
	where $M_a\rho_w=O(\gamma/p)$ is a small constant by our choice of $\rho_w$.
	
	Also, we have the bound on the expectation of the norm of the perturbed (manifold) gradient.
	\begin{align*}
		\mathbb{E}_{\xi,\epsilon}\norm{\left(I-ww^\top\right)\nabla_wL\left(w+\xi,a+\epsilon\right)}_2^2
		&\leq \frac{M_a^2 + \rho_a^2\norm{a^*}_2^2}{2} + C_1 k^2\rho_w^2,
	\end{align*}
	for some constant $C_1$.
	
	Thus, denote $\lambda_3=\frac{11m_a\eta}{160}$ and $b_3=\frac{11m_a}{1280}\eta\gamma$. We have
	\begin{align*}
		\mathbb{E}[\norm{\tilde{w}_{t+1}-w^*}_2^2|\mathcal{F}_t]
		&=\norm{w_t-w^*}_2^2 -2\langle-\eta\mathbb{E}_{\xi_t,\epsilon_t}\left(I-w_tw_t^\top\right)\nabla_wL\left(w_t+\xi_t,a_t+\epsilon_t\right),w^*-w_t\rangle\\
		& +\mathbb{E}_{\xi_t,\epsilon_t}\norm{\eta\left(I-w_tw_t^\top\right)\nabla_wL\left(w_t+\xi_t,a_t+\epsilon_t\right)}_2^2\\
		&\leq \left(1-\frac{11m_a\eta}{160}\right)\norm{w_t-w^*}_2^2 + \eta M_a\rho_w
		+\frac{\eta^2\left(M_a^2 + \rho_a^2\norm{a^*}_2^2\right)}{2} + C_1\eta^2k^2\rho_w^2 \\
		&\leq\left(1-\lambda_3\right)\norm{w_t-w^*}_2^2 + b_3,
	\end{align*}
	where the last line is due to our choice of small parameters $\rho_a$, $\rho_w$ and $\eta$.
	
	With same argument in the proof of Lemma \ref{stage-nice}, and we have
	\begin{align*}
		\norm{\tilde{w}_{t+1}-w^*}_2^2\geq\norm{w_{t+1}-w^*}_2^2.
	\end{align*}
	
	Hence, denote $\mathscr{E}_t=\{\forall\tau\leq t,\norm{w_\tau-w^*}_2^2\geq\frac{\gamma}{2}\}$, with our choice of $\eta$ and $t=\tilde{O}\left(\frac{p}{\eta}\log\frac{1}{\gamma}\right)$, we have
	\begin{align*}
		\frac{\gamma}{2}\mathbb{P}\left(\mathscr{E}_t\right)
		\leq\mathbb{E}\norm{w_t-w^*}_2^2
		&=\left(1-\lambda_3\right) \mathbb{E}\norm{w_{t-1}-w^*}_2^2 + b_3\\
		&\leq \left(1-\lambda_3\right)^t \norm{w_0-w^*}_2^2 + \frac{b_3}{\lambda_3}
		\leq \frac{1}{4}\gamma.
	\end{align*}
	Thus, we have $\mathbb{P}\left(\mathscr{E}_t\right)\leq 0.5$ and recursively apply the above lines for $\log\frac{1}{\delta}$ times, we know there exists $\tau_{21}=\tilde{O}\left(\frac{p}{\eta}\log\frac{1}{\gamma}\log\frac{1}{\delta}\right)$ such that $\norm{w_{\tau_{21}}-w^*}_2^2\leq\frac{\gamma}{2}$ with at least probability $1-\delta$.

	\textbf{Step 2:} We show that if $\norm{w_0-w^*}_2^2\leq\frac{\gamma}{2}$, $w_t$ stays in the region $\left\{w\Big|\norm{w-w^*}_2^2 \leq \gamma\right\}$ in the following $\tilde{O}\left(\frac{1}{\eta^2}\right)$ steps with at least probability $1-\delta$. 
	
	Denote $G_t=\left(1-\lambda_3\right)^{-t}\left(\norm{w_t-w^*}_2^2-\frac{b_3}{\lambda_3}\right)$ and $\mathscr{H}_t=\{\forall\tau\leq t,\norm{w_\tau-w^*}_2^2\leq\gamma\}\subset\mathcal{F}_t$. From Step 1, we have
	$$\mathbb{E}[G_{t+1}\mathds{1}_{\mathscr{H}_t}|\mathcal{F}_t]
	\leq G_t\mathds{1}_{\mathscr{H}_{t}}
	\leq G_t\mathds{1}_{\mathscr{H}_{t-1}}.$$

	Thus, $\{G_t\mathds{1}_{\mathscr{H}_{t}}\}$ is a supermartingale.
	
	To apply Azuma's Inequality, we first have to bound the difference between $G_{t+1}\mathds{1}_{\mathscr{H}_{t}}$ and $\mathbb{E}[G_{t+1}\mathds{1}_{\mathscr{H}_{t}}|\mathcal{F}_t]$.
	\begin{align*}
		d_{t+1}
		&\triangleq	
		|G_{t+1}\mathds{1}_{\mathscr{H}_{t}}-\mathbb{E}[G_{t+1}\mathds{1}_{\mathscr{H}_{t}}|\mathcal{F}_t]|
		\leq\left(1-\lambda_3\right)^{-t-1}|\norm{w_{t+1}-w^*}_2^2-\mathbb{E}[\norm{w_{t+1}-w^*}_2^2|\mathcal{F}_t]|\\
		&\leq \left(1-\lambda_3\right)^{-t-1}C_2\eta\gamma^{\frac{1}{2}}k = \left(1-\lambda_3\right)^{-t-1}M_3,
	\end{align*}
	where $\lambda_3=\tilde{O}\left(\eta/p\right)$, $M_3=\tilde{O}\left(\eta\gamma^{\frac{1}{2}}k\right)$.
	
	Denote $r_t=\sqrt{\sum_{i=0}^{t}d_i^2}$. By Azuma's Inequality, we get
	$$\mathbb{P}\left(G_t\mathds{1}_{\mathscr{H}_{t-1}}-G_0 \geq\tilde{O}\left(1\right)r_t\log^{\frac{1}{2}}\left(\frac{1}{\eta^2\delta}\right)\right)
	\leq \exp\left(-\frac{\tilde{O}\left(1\right)r_t^2\log\left(\frac{1}{\eta^2\delta}\right)}{2\sum_{i=0}^{t}d_i^2}\right)
	=\tilde{O}\left(\eta^2\delta\right).$$
	
	Therefore, with at least probability $1-\tilde{O}\left(\eta^2\delta\right)$, we have
	\begin{align*}
		\norm{w_t-w^*}_2^2
		&\leq \left(1-\lambda_3\right)^t\left(\norm{w_0-w^*}_2^2-\frac{b_3}{\lambda_3}\right) 
		+ \tilde{O}\left(1\right)\left(1-\lambda_3\right)^t r_t\log^{\frac{1}{2}}\left(\frac{1}{\eta^2\delta}\right)
		+ \frac{b_3}{\lambda_3}\\
		&\leq \norm{w_0-w^*}_2^2  
		+ \tilde{O}\left(1\right)\frac{M_3}{\sqrt{\lambda_3}}\log^{\frac{1}{2}}\left(\frac{1}{\eta^2\delta}\right)
		+ \frac{b_3}{\lambda_3}\leq \gamma,
	\end{align*}
	where the last line holds, since we can always find $\eta\leq\eta_{\max}=\tilde{O}\left(\frac{\gamma}{k^2p}\right)$ to satisfy the condition.
	
	The above inequality shows that if $\mathscr{E}_t$ holds, then $\mathscr{E}_{t+1}$ holds with at least probability $1-\tilde{O}\left(\eta^2\delta\right)$. Hence, with at least probability $1-\delta$, we have $\norm{w_t-w^*}_2^2\leq\gamma$ for all $t\leq T=\tilde{O}\left(\frac{1}{\eta^2}\right)$.
	
	Combining the above two steps, with probability $1-\delta$, we have $\norm{w_t-w^*}_2^2\leq C_{12}\gamma$ for all t's such that $\tau_{21}\leq t\leq\tilde{O}(\eta^{-2})$, where
	$C_{12}$ is a constant and $\tau_{21}=\tilde{O}\Big(\frac{p}{\eta}\log\frac{1}{\gamma}\log
	\frac{1}{\delta}\Big).$
\end{proof}

\subsubsection{Proof of Lemma \ref{stage2_a}}
\begin{proof}
	Our proof has two steps.\\
	\textbf{Step 1:} We  show that with probability at least $1-\delta$, there exists $\tau_{22}=\tilde{O}\left(\frac{1}{\eta}\log\frac{1}{\gamma}\log\frac{1}{\delta}\right)$ such that $\norm{a_{\tau_{22}}-a^*}_2^2\leq\gamma/2$.

	Recall that $(w_t,a_t)\in\cR_{m_a^\prime,M_a^\prime,C_{12}}$ holds for all t. Then by Theorem \ref{thm_ASC}, we have
	\begin{align*}
		\langle-\mathbb{E}_{\xi,\epsilon}\nabla_aL\left(w+\xi,a+\epsilon\right),a^*-a\rangle
		\geq \frac{\pi-1}{2\pi}\norm{a-a^*}_2^2 -\gamma.
	\end{align*}
	Following similar lines to the proof of Lemma \ref{stage1_a}, we have the bound on the expectation of the norm of the perturbed gradient.
	\begin{align*}
		\mathbb{E}_{\xi,\epsilon}\norm{\nabla_aL\left(w+\xi,a+\epsilon\right)}_2^2
		&=\mathbb{E}_{\xi,\epsilon} 
		\norm{\nabla_aL\left(w+\xi,a+\epsilon\right) - \nabla_aL\left(w^*,a^*\right)}_2^2\\
		&=\mathbb{E}_{\xi,\epsilon} \norm{\frac{1}{2\pi}\left(\mathds{1}\one^\top+\left(\pi-1\right)I\right)\left(a+\epsilon-a^*\right)
			-\frac{g\left(\phi_\xi\right)-\pi}{2\pi}a^*}_2^2\\
		&\leq \frac{1}{2\pi^2} \mathbb{E}_{\xi,\epsilon} \norm{\left(\mathds{1}\one^\top+\left(\pi-1\right)I\right)\left(a+\epsilon-a^*\right)}_2^2 
		+\gamma\leq \frac{\left(k+\pi-1\right)^2}{\pi^2}\left(\norm{a-a^*}_2^2+\rho_a^2\right) + \gamma.
	\end{align*}
	
	Combined the above two, with $\eta_{\max}=\tilde{O}\left(\gamma/k^2p\right)$ and $\eta\leq \eta_{\max}$, we have
	\begin{align*}
		\mathbb{E}[\norm{a_{t+1}-a^*}_2^2|\mathcal{F}_t]
		&=\norm{a_t-a^*}_2^2
		-2\langle-\eta\mathbb{E}_{\xi_t,\epsilon_t}\nabla_aL\left(w_t+\xi_t,a_t+\epsilon_t\right),a^*-a_t\rangle
		+\mathbb{E}_{\xi_t,\epsilon_t}\norm{\eta \nabla_aL\left(w_t+\xi_t,a_t+\epsilon_t\right)}_2^2\\
		&\leq \left(1-\frac{\left(\pi-1\right)\eta}{\pi}+\eta^2\frac{\left(k+\pi-1\right)^2}{\pi^2}\right)\norm{a_t-a^*}_2^2
		+2\eta\gamma+\eta^2\gamma + \frac{\left(k+\pi-1\right)^2}{\pi^2}\eta^2\rho_a^2\\
		&\leq \left(1-\frac{\left(\pi-1\right)\eta}{\pi}+\eta^2\frac{\left(k+\pi-1\right)^2}{\pi^2}\right)\norm{a_t-a^*}_2^2
		+3\eta\gamma
	\end{align*}
	
	Thus, when $\eta\leq\frac{\pi^2}{25\left(k+\pi-1\right)^2}$, we have
	\begin{align*}
		\mathbb{E}[\norm{a_{t+1}-a^*}_2^2-5\gamma|\mathcal{F}_t] 
		&\leq  \left(1-\frac{\left(\pi-1\right)\eta}{\pi}+\eta^2\frac{\left(k+\pi-1\right)^2}{\pi^2}\right)\left(\norm{a_t-a^*}_2^2-5\gamma\right)
		-0.2\eta\gamma+5\frac{\left(k+\pi-1\right)^2}{\pi^2}\eta^2\gamma\\
		&\leq  \left(1-\frac{\left(\pi-1\right)\eta}{\pi}+\eta^2\frac{\left(k+\pi-1\right)^2}{\pi^2}\right)\left(\norm{a_t-a^*}_2^2-5\gamma\right)\\
		&=  \left(1-\lambda_4\right)\left(\norm{a_t-a^*}_2^2-5\gamma\right).
	\end{align*}
	
	Denote
	$\mathscr{E}_t=\{\forall\tau\leq t, \norm{a_\tau-a^*}_2^2\geq12\gamma\}$. When $t\geq\frac{\log\left(\frac{\norm{a_0-a^*}_2^2}{\gamma}\right)}{\lambda_4}
	=\tilde{O}\left(\frac{1}{\eta}\log\frac{1}{\gamma}\right)$, we have
	\begin{align*}
		12\gamma\mathbb{P}\left(\mathscr{E}_t\right)	
		&\leq \mathbb{E} \norm{a_t-a^*}_2^2
		=\left(1-\lambda_4\right)^t(\norm{a_0-a^*}_2^2-5\gamma) + 5\gamma
		\leq \left(1-\lambda_4\right)^t\norm{a_0-a^*}_2^2+5\gamma\leq 6\gamma.
	\end{align*}
	
	Therefore, $\mathbb{P}\left(\mathscr{E}_t\right)\leq 0.5$. Recursively applying the above lines with $\log\frac{1}{\delta}$ times, we know that with at least probability $1-\delta$, there exists $\tau_{22}=\tilde{O}\left(\frac{1}{\eta}\log\frac{1}{\gamma}\log\frac{1}{\delta}\right)$ such that $\norm{a_{\tau_{22}}-a^*}_2^2\leq12\gamma$. Rescaling $\gamma$, we get the desired result.

	\textbf{Step 2:} We show that, if $\norm{a_0-a^*}_2^2\leq\gamma/2$, then $a_t$ stays in the region $\left\{a\Big|\norm{a-a^*}_2^2\leq\gamma\right\}$ in the next $\tilde{O}\left(\frac{1}{\eta^2}\right)$ steps with probability at least $1-\delta$.
	
	%For notational simplicity, assume $\norm{a_0-a^*}_2^2\leq\gamma/2$. 
	Denote $G_t=\left(1-\lambda_4\right)^{-t}\left(\norm{a_t-a^*}_2^2-5\gamma\right)$ and $\mathscr{H}_t=\{\forall\tau\leq t, \norm{a_t-a^*}_2^2\leq6\gamma\}$. With the same argument in Step 1, we have
	$$\mathbb{E}[G_{t+1}\mathds{1}_{\mathscr{H}_t}|\mathcal{F}_t]\leq G_t\mathds{1}_{\mathscr{H}_t}\leq G_t\mathds{1}_{\mathscr{H}_{t-1}}.$$
	Thus, $\{G_t\mathds{1}_{\mathscr{H}_{t}}\}$ is a supermartingale.
	
	To use Azuma's Inequality, we first have to bound the difference between $G_{t+1}\mathds{1}_{\mathscr{H}_{t}}$ and $\mathbb{E}[G_{t+1}\mathds{1}_{\mathscr{H}_{t}}|\mathcal{F}_t]$.
	\begin{align*}
		d_{t+1}
		&\triangleq	
		|G_{t+1}\mathds{1}_{\mathscr{H}_{t}}-\mathbb{E}[G_{t+1}\mathds{1}_{\mathscr{H}_{t}}|\mathcal{F}_t]|\\
		&= \left(1-\lambda_4\right)^{-t-1} \left|\norm{a_{t+1}-a^*}_2^2-\mathbb{E}[\norm{a_{t+1}-a^*}_2^2|\mathcal{F}_t]\right| \mathds{1}_{\mathscr{H}_{t}}\\
		%&\leq \left(1-\lambda_4\right)^{-t-1} \{\frac{\eta\gamma^{\frac{1}{2}}}{\pi}\left((k+\pi-1)\rho_a+2\pi\norm{a^*}_2\right)\\
		%&+\frac{\eta^2}{4\pi^2}\Big(2\left(k+\pi-1\right)^2\rho_a^2
		%+2\left(\pi-1\right)^2\norm{a^*}_2
		%+2\left(k+\pi-1\right)^2\rho_a\left(\norm{a^*}_2+\gamma\right)
		%+2\pi\left(k+\pi-1\right)M_a\\
		%&+2\left(k+\pi-1\right)^2\rho_a\norm{a^*}_2
		%+2k\left(\pi-1\right)\norm{a^*}_2^2\Big)\}
		&\leq \left(1-\lambda_4\right)^{-t-1} (C_1\eta\gamma^{\frac{1}{2}}k + C_2\eta^2k^2)
		= \left(1-\lambda_4\right)^{-t-1}M_4,
	\end{align*}
	for some positive constant $C_1$ and $C_2$, where $\lambda_4=\tilde{O}\left(\eta\right)$, $M_4=\tilde{O}\left(\eta\gamma^{\frac{1}{2}}k\right)$.
	
	Denote $r_t=\sqrt{\sum_{i=0}^{t}d_i^2}$. By Azuma's Inequality,
	$$\mathbb{P}\left(G_t\mathds{1}_{\mathscr{H}_{t-1}}-G_0 \geq\tilde{O}\left(1\right)r_t\log^{\frac{1}{2}}\left(\frac{1}{\eta^2\delta}\right)\right) 
	\leq \exp\left(-\frac{\tilde{O}\left(1\right)r_t^2\log\left(\frac{1}{\eta^2\delta}\right)}{2\sum_{i=0}^{t}d_i^2}\right)
	=\tilde{O}\left(\eta^2\delta\right).$$
	
	Therefore, with at least probability $1-\tilde{O}\left(\eta^2\delta\right)$, we have
	\begin{align*}
		\norm{a_t-a^*}_2^2
		&\leq \left(1-\lambda_4\right)^t\left(\norm{a_0-a^*}_2^2-5\gamma\right) 
		+ \tilde{O}\left(1\right)\left(1-\lambda_4\right)^t r_t\log^{\frac{1}{2}}\left(\frac{1}{\eta^2\delta}\right)
		+ 5\gamma\\
		&\leq \norm{a_0-a^*}_2^2  
		+ \tilde{O}\left(1\right)\frac{M}{\sqrt{\lambda_4}}\log^{\frac{1}{2}}\left(\frac{1}{\eta^2\delta}\right)
		+ 5\gamma\leq 6\gamma,
	\end{align*}
	where the last line holds, since we can always find $\eta\leq\eta_{\max}=\tilde{O}\left(\frac{\gamma}{k^2p}\right)$ to satisfy the condition. 
	The above inequality shows that if $\mathscr{H}_t$ holds, then $\mathscr{H}_{t+1}$ holds with at least probability $1-\tilde{O}\left(\eta^2\delta\right)$. Hence, with at least probability $1-\delta$, we have $\norm{a_t-a^*}_2^2\leq6\gamma$ for all $t\leq T=\tilde{O}\left(\frac{1}{\eta^2}\right)$. Rescaling $\gamma$, we have the desired results for Step 2.
	
	Combining the above two steps, with at least probability $1-\delta$, we have $\norm{a_t-{a}^*}_2^2\leq \gamma$ for all $t$'s such that $\tau_{22}\leq t\leq\tilde{O}(\eta^{-2})$, where $\tau_{22}=\tilde{O}\Big(\frac{p}{\eta}\log\frac{1}{\gamma}\log
	\frac{1}{\delta}\Big).$
\end{proof}

\section{An Additional Experiment for Training Overparameterized Neural Network}\label{add_exp}

Our additional experiment still considers the regression problem under the realizable setting, where the response is generated by a noiseless {\bf teacher} network
\begin{align*}
	y=f(Z, w^*,a^*)=(a^*)^\top\sigma(Z^\top w^*).
\end{align*}
The {\bf student} network $h$, however, adopts a different architecture and contains two convolutional filters, i.e.,
\begin{align*}
	h(Z,w,u,a,b) = a^\top\sigma(Z^\top w) + b^\top\sigma(Z^\top v),
\end{align*}
where $v\in\RR^p$ and $b\in\RR^k$.  Compared with the teach network, the student network is overparameterized. We then learn the overparameterized student network by solving the following optimization problem:
\begin{align}\label{opt-exp}
	\min_{w,v,a,b} F(w,v,a,b)\quad\text{subject to}\quad w^\top w=1~\textrm{and}~v^\top v=1,
\end{align}
where $F(w,v,a,b)=\frac{1}{2}\mathbb{E}_Z(h(Z,w,v,a,b) - f(Z,w^*,a^*))^2$.

Unfortunately, $F(w,v,a,b)$ and $\nabla F(w,v,a,b)$ do not admit analytical forms. Therefore, we randomly sample $n$ realizations of $Z$ (denoted by $Z_i,~i=1,...n$), and solve a finite sample approximation of \eqref{opt-exp},
\begin{align}\label{fin-opt-exp}
	\min_{w,v,a,b} F_n(w,v,a,b)\quad\text{subject to}\quad w^\top w=1~\textrm{and}~v^\top v=1,
\end{align}
where $F_n(w,v,a,b)=\frac{1}{2n}\sum_{i=1}^n(h(Z_i,w,v,a,b) - f(Z_i,w^*,a^*))^2$.

For our experiment, we choose $k=10$ and $p=15$. The first $5$ entries of $a^*$ all equal to $1/\sqrt{10}$ and the remaining entries of $a^*$ all equal to $-1/\sqrt{10}$. $w^*$ is randomly generated over the unit sphere. We choose $n=10,000$, and expect \eqref{fin-opt-exp} to have an optimization landscape to \eqref{opt-exp}.

We run the gradient descent algorithm to solve \eqref{fin-opt-exp}. The initialization is chosen at
\begin{align*}
	w=-w^*,~v = -w^*,~a_0=(\mathds{1}\mathds{1}^\top+(\pi-1)I)^{-1}(\mathds{1}\mathds{1}^\top-I)a^*~\textrm{and}~b_0=0.
\end{align*}
We choose the step size $\eta=10^{-5}$ and run for $10^{8}$ iterations. We eventually observe $\norm{\nabla F_n(w,v,a,b)}_2 < 10^{-4}$ and $F_n(w,v,a,b)> 0.15$. We suspect that the gradient descent algorithm approaches some spurious local optimum.

\end{document}